%% file: main.tex
\documentclass{article}
\usepackage[nonatbib, final, main]{neurips_2026}

\usepackage[utf8]{inputenc}
\usepackage[toc,page,header]{appendix}
\usepackage{minitoc}
\usepackage{amsmath,amssymb,amsthm,mathtools}
\usepackage{hyperref}
\usepackage{cleveref}
\usepackage{thmtools}
\usepackage{xcolor}
\usepackage{enumitem}
\usepackage{graphicx}
\usepackage{booktabs}
\usepackage{threeparttable}
\usepackage{multirow}
\usepackage{graphicx}
\usepackage{subcaption}
\graphicspath{{figures/}}
\usepackage[most]{tcolorbox}
\tcbset{
  emphblock/.style={
    colback=blue!3, colframe=blue!55!black, boxrule=0.6pt,
    arc=2pt, left=8pt, right=8pt, top=6pt, bottom=6pt,
    breakable, enhanced
  }
}
\newcounter{condition}[section]

\PassOptionsToPackage{numbers, compress}{natbib}

\usepackage{mysymbol}

\newtheorem{theorem}{Theorem}[section]
\newtheorem{lemma}[theorem]{Lemma}
\newtheorem{proposition}[theorem]{Proposition}

\newtheorem{assumption}{Assumption}
\newtheorem{remark}[theorem]{Remark}
\newcommand{\R}{\mathbb{R}}
\newcommand{\E}{\mathbb{E}}
\newcommand{\Prob}{\mathbb{P}}
\newcommand{\cL}{\mathcal{L}}

\newcommand{\cD}{\mathcal{D}}
\newcommand{\cN}{\mathcal{N}}
\newcommand{\cW}{\mathcal{W}}
\newcommand{\cR}{\mathcal{R}}
\newcommand{\supp}{\mathrm{supp}}
\newcommand{\bW}{\mathbf{W}}
\newcommand{\bx}{\boldsymbol{x}}
\newcommand{\bh}{\boldsymbol{h}}

\newcommand{\bv}{\boldsymbol{v}}
\newcommand{\br}{\boldsymbol{r}}
\newcommand{\be}{\boldsymbol{e}}
\newcommand{\bz}{\boldsymbol{z}}

\newcommand{\bu}{\boldsymbol{u}}

\newcommand{\norm}[1]{\left\| #1 \right\|}
\newcommand{\goodness}{G}

\DeclareMathOperator{\ELC}{ELC}
\DeclareMathOperator{\erank}{erank}

\title{The Price of Locality: Why Forward-Forward \\
Underperforms Backpropagation?}

\author{
    Zhaoxian Wu, Haichuan Liu,
    {Tianyi Chen}
     \\
    Cornell Tech and Cornell University \\
    \texttt{\{zw868, hl2622, tianyi.chen\}@cornell.edu} \\
  }

\begin{document}
\maketitle
\begingroup
\renewcommand{\thefootnote}{}
\footnotetext{The work was supported by the National Science Foundation Projects 2532349 and 2532653, the NVIDIA Academic Grant, and the Cisco Research Award.}
\endgroup

\doparttoc %
\faketableofcontents %
\input{body}
\newpage
\appendix

\begin{center}
  {\LARGE\bfseries Supplementary Material}
\end{center}
\vspace{-4em}
{
\renewcommand{\thepart}{}
\renewcommand{\partname}{}
\part{}
}
\parttoc 

\newpage

\input{appendix/A_proofs}
\input{appendix/C_assumption_compat}
\input{appendix/D_extended_experiments}
\input{appendix/F_taxonomy}
\input{appendix/G_acknowledgement}

\clearpage

\end{document}

%% file: body.tex
\vspace{-1em}
\begin{abstract}
The Forward-Forward Algorithm (FFA) replaces backpropagation (BP) with layer-wise local contrastive objectives, eliminating the backward pass and the need to retain intermediate activations, yet suffers a persistent performance gap with BP that worsens with depth. This paper diagnoses two structural failure modes: an optimization floor arising from concurrent local updates; and a geometric collapse of layer representations driven by the local update mechanism. On the optimization side, we prove that the FFA loss satisfies the Polyak--{\L}ojasiewicz inequality at each layer; however, simultaneous layer updates induce inter-layer representation-distribution drift, so each layer optimizes against a moving input distribution and incurs an error floor. On the representational side, the pairwise similarity kernel of layer representations contracts exponentially toward rank one as depth increases, collapsing the diversity of per-layer error signals. This collapse bounds FFA's \emph{effective learning capacity}, which measures the diversity of gradient information across layers, independently of depth, whereas BP's chain-rule signal preserves per-layer diversity, yielding a capacity that scales with depth.
\end{abstract}
\section{Introduction}
\label{sec:intro}

Backpropagation (BP) is the de facto standard for training deep neural networks, but it relies on globally propagated error signals: intermediate activations must be stored until the backward pass, which prevents asynchronous updates and incurs substantial memory overhead. The Forward-Forward Algorithm (FFA)~\cite{hinton2022forward} sidesteps these costs by training each layer with a \emph{local} contrastive objective---increasing a ``goodness'' measure for positive samples and decreasing it for negative ones---with no gradient flowing across layers. This locality is attractive for noisy analog hardware~\cite{wu2024distance,oguz2023optical}, pipeline-parallel training~\cite{aktemur2024distributed}, and continual learning~\cite{terres2024continual}. Intuitively, one might expect such locality to incur only a negligible accuracy cost: although FFA does not access the global task loss, it injects supervisory information directly into every layer, avoiding information loss during backpropagation~\cite{cover2006elements}. However, across the original algorithm and more recent FFA variants, empirical studies repeatedly report a persistent gap to BP on nontrivial benchmarks and deeper architectures~\cite{hinton2022forward,torres2025advancements,scodellaro2025cnn,krutsylo2025scalable}. What remains missing is a systematic understanding of how large the gap is and, more fundamentally, a mechanism-level explanation of why FFA fails when scale increases. Understanding why FFA underperforms at scale is therefore not merely an empirical curiosity but also a practical necessity that informs the design of better local learning algorithms and clarifies the price of eliminating the backward pass.

This paper studies FFA in the \emph{scaled ResNet} setting.
Consider an $L$-layer scaled residual network. Layer $\ell$ has weight $W^{(\ell)} \in \R^{d^{(\ell)} \times d^{(\ell-1)}}$ and activation is computed by the residual update
\begin{align}
    \label{eq:resnet_arch}
    \bh^{(\ell)}(\bx)
    =
    \bh^{(\ell-1)}(\bx)
    +
    \frac1L\phi(\text{LN}(W^{(\ell)} \bh^{(\ell-1)}(\bx))),
\end{align}
where $\phi$ is ReLU/GELU and 
$\text{LN}(\bz) := \frac{\bz - \mu(\bz)\mathbf{1}}{\sqrt{\frac{1}{d}\|\bz - \mu(\bz)\mathbf{1}\|_2^2 + \varepsilon}}$ with $\mu(\bz) := \frac{1}{d}\mathbf{1}^{\top}\bz$
is layer normalization \cite{ba2016layer} applied coordinate-wise to $\bz \in \R^d$ (without learnable affine parameters). Unlike backpropagation, FFA trains each layer using only \emph{local} forward-pass information, via a contrastive objective.
A \emph{goodness function} $\goodness: \R^{d} \to \R$ maps an activation vector to a scalar measuring its quality. 
The goodness of layer $\ell$ is $\goodness^{(\ell)}(\bx; W^{(\ell)}) := \goodness(\bh^{(\ell)}(\bx))$. The \emph{goodness threshold} $\theta > 0$ is a per-layer decision boundary: layer~$\ell$ classifies input~$\bx$ as positive when $\goodness^{(\ell)}(\bx) > \theta$ and negative when $\goodness^{(\ell)}(\bx) < \theta$. Given positive distribution $\cD^+$ and negative distribution $\cD^-$, FFA optimizes that layer-wise loss:
\begin{align}\label{eq:ffa_loss}
    \small
    \cL^{(\ell)}(W^{(\ell)}) := 
    \E_{\bx^+, \bx^-}\!\left[-\log \sigma\!\left(\goodness^{(\ell)}(\bx^+; W^{(\ell)}) - \theta\right)-\log \sigma\!\left(\theta - \goodness^{(\ell)}(\bx^-; W^{(\ell)})\right)\right],
\end{align}
where $\sigma(z) = 1/(1+e^{-z})$ is the sigmoid function.
The two terms penalize positive samples with low goodness and negative samples with high goodness, respectively. 
FFA trains layers concurrently: at each step, the forward pass produces both positive and negative samples $(\bx^{+}, \bx^{-})$ and their corresponding activations $(\bh^{+}, \bh^{-})$. Built on that, the following \emph{FFA iteration} is applied
\begin{align}\label{eq:ffa_update}
    W_{k+1}^{(\ell)} = W_k^{(\ell)} - \eta^{(\ell)} \nabla {\cL}^{(\ell)}(W_k^{(\ell)}), \qquad \ell\in [L].
\end{align}
Unlike BP, where the $\ell$-th layer's gradients depend on the weights from all previous layers, FFA computes gradients only for the current layer, i.e., no gradient of $\cL^{(\ell)}$ with respect to $W^{(\ell')}$ for all $\ell' \neq \ell$ is computed. Consequently, it is unnecessary to store activations $\bh^{(\ell)}(\bx)$ for backward.

\subsection{Main Results.} 
Existing empirical studies report a persistent gap between FFA and BP on nontrivial benchmarks and deeper architectures~\cite{hinton2022forward,torres2025advancements,scodellaro2025cnn,krutsylo2025scalable},
indicating that FFA's limitations become more pronounced as task complexity and scale increase. 
These findings prompt a critical question:
\begin{center}
    \textbf{Q)} {\em What is the fundamental mechanism that causes FFA to degrade compared to backpropagation? }
\end{center}

This paper answers the question by identifying two distinct but coupled failure modes of FFA: \textbf{optimization floor} and \textbf{representational collapse}.

\paragraph{R1) Convergence analysis: PL and the error plateau.}
We prove that the FFA loss $\cL^{(\ell)}$ satisfies the Polyak--{\L}ojasiewicz (PL) inequality~\cite{polyak1963gradient,karimi2016linear} inside a neighborhood of initialization (Theorem~\ref{thm:pl_single}), with the PL constant determined by the minimum eigenvalue of the goodness-gradient Gram and the sigmoid decision margin. Extending to multiple simultaneous layer updates, the aggregate sub-optimality $V_k = \frac{1}{L}\sum_{\ell\in[L]}(\cL_k^{(\ell)} - \cL^{(\ell)*})$ converges at a linear rate to an error plateau, where $k$ is the number of iterations.
(Theorem~\ref{thm:multi_layer}). 
This plateau arises because concurrent layer updates create a moving-target problem: an update in one layer changes the input distribution for the next layer, which is optimizing a different objective than it did at the previous step. Consequently, it creates objective mismatch and layer-wise sub-optimality that degrade empirical performance.
On the contrary, BP updates all layers to consistently optimize a single global objective, and the PL condition naturally guarantees global convergence \cite{bottou2018optimization}.
This plateau is supported by empirical observations that the convergence curves of FFA consist of two phases: a fast-decay phase followed by a plateau. 

\paragraph{R2) Representational analysis: kernel contraction.}
In addition to optimization, we identify another independent failure mode: representational collapse.
To do so, this paper defines the Gram kernel $\Sigma^{(\ell)} \in \R^{N\times N}$ of depth-$\ell$ representations, whose $(i,j)$-entry is the inner product of the $i$-th and $j$-th representation vectors and reflects the correlation between samples. The spectrum of $\Sigma^{(\ell)}$ encodes the pairwise discriminative representation available to any downstream readout. We show (Theorem~\ref{thm:kernel_contraction}) that $\Sigma^{(\ell)}$ contracts exponentially toward the all-ones matrix as depth increases by showing the effective rank of $\Sigma^{(\ell)}$ vanishes at a rate determined by the leading Lyapunov exponent of the layer-product map. 
In contrast, BP's chain rule prevents this collapse. Under end-to-end training, BP delivers to each layer a per-sample error signal whose diversity is preserved across depth (Theorem~\ref{thm:bp_esd}). Consequently, the effective rank of the BP error Gram at any layer is bounded below by a constant fraction of that at the output layer. 
\subsection{Related Work}
\label{sec:related}

\textbf{FFA variants.}
Original FFA literature \cite{hinton2022forward} adopts a simple squared-norm goodness and trains on small vision benchmarks. To improve performance on more complex tasks, subsequent work proposes a series of variants, including symmetric variants~\cite{lee2023symba}, layer-collaboration~\cite{lorberbom2024layer}, CNN extensions~\cite{scodellaro2025cnn}, self-contrastive formulations~\cite{chen2024selfcontrastive}, GNN extensions~\cite{park2024forwardgnn}, and SNN extensions~\cite{ghader2025snn}; Torres et al.~\cite{torres2025advancements} survey the landscape. Spyra \& Dzwinel~\cite{spyra2025energy} and Krutsylo~\cite{krutsylo2025scalable} study FFA efficiency and scalability.
Sun et al.~\cite{sun2025deeperforward} propose DeeperForward, replacing squared-norm goodness with mean goodness and layer normalization to extend FFA to 17-layer CNNs.
Unlike these works, this paper does not propose a new instance of the FFA family but rather provides a systematic analysis of the original FFA and its variants.
To study the training behavior of FFA, Tosato et al.~\cite{tosato2023emergent} study emergent sparse representational structure under local learning constraints, which is consistent with our kernel-contraction analysis (Theorem~\ref{thm:kernel_contraction}). 

\textbf{Local supervised learning (LSL).}
LSL often yields better performance than single-step goodness maximization because LSL injects richer multi-label supervision locally. FFA maximizes a goodness gap between positive and negative samples, which is essentially a binary classification objective; by contrast, LSL provides multi-label supervision at each layer, which can preserve and encourage more discriminative representations. From this angle, LSL can be considered as a generalization of FFA on multi-label tasks. Representative local objectives include per-layer cross-entropy heads~\cite{belilovsky2019greedy,nokland2019local}, Decoupled Greedy Learning \cite{belilovsky2020decoupled}, HSIC-based losses~\cite{ma2020hsic}, and Direct Feedback Alignment~\cite{nokland2016direct}.
To further improve accuracy, we also propose introducing inter-layer communication to mitigate the local/global objective mismatch \cite{ma2024auglocal, gomez2022interlocking}.
A detailed taxonomy of local learning baselines is in Appendix~\ref{app:baseline_taxonomy}.
This paper focuses on FFA rather than LSL; however, LSL, in fact, shares FFA's layer-locality and can therefore also be affected by the local/global objective mismatch. Therefore, it is possible to extend our analysis to LSL, which will be left as future work. However, we still compare LSL against FFA and BP in the empirical studies.

\textbf{Energy-based predictive coding (PC).} 
Unlike FFA and LSL, Predictive coding~\cite{rao1999predictive,whittington2017approximation} proposes a biologically inspired local learning algorithm that optimizes a predictive energy through an iterative inference loop.
In principle, PC is a competitive candidate for local learning, since its iterative loop has a fixed point that recovers the backpropagation gradient while keeping the computation local. However, the variant~\cite{millidge2022predictive} that drops this loop is adopted in the implementation to balance the computational cost. In this case, PC is an instance of the LSL family~\cite{belilovsky2019greedy,nokland2019local} and inherits the pathologies characterized by this paper. The detailed analysis will be left for future work.

\textbf{Contrastive learning.}
CPC~\cite{oord2018representation} and Greedy InfoMax~\cite{lowe2019greedy} demonstrate that layer-locality imposes a fundamental representational cost in self-supervised contrastive learning. Gloeckle et al.~\cite{gloeckle2024multi} and Salvatori et al.~\cite{salvatori2026predictive} document depth-dependent signal decay that can be explained by our kernel-contraction analysis, providing independent empirical corroboration. 

\textbf{Theoretical connections.}
Building on NTK-style convergence analysis \cite{jacot2018neural}, our analysis extends Du et al.~\cite{du2019gradient} and Allen-Zhu et al.~\cite{allenzhu2019convergence} by introducing an \emph{adaptive} Gram structure: the goodness-gradient Gram $\Gamma^{(\ell)}$ is shaped by FFA's local objective rather than the task loss, and its rank evolution, which is not captured by standard NTK theory, is the central object of our analysis. The representation gap further connects to NCE theory~\cite{arora2019theoretical,gutmann2010noise} and spectral contrastive analysis~\cite{haochen2021spectral}.
\section{FFA Convergence Fundamentals}
\label{sec:single_layer}

We now state the regularity conditions on the network, data, and goodness function.

\begin{assumption}[Bounded activations]\label{ass:bounded}
$\|\bh^{(\ell)}(\bx)\| \leq B_h$ for all $\bx$ and all $\ell$.
\end{assumption}

\Cref{ass:bounded} is weak since LN is applied.
The remaining assumptions concern the goodness function $\goodness$ itself.

\begin{assumption}[Goodness smoothness]\label{ass:G}\label{ass:G_smooth}
The goodness gradient $\nabla_{\bh}\goodness$ is globally $L_{\goodness}$-Lipschitz: $\|\nabla_{\bh}\goodness(\bh) - \nabla_{\bh}\goodness(\bh^\prime)\| \leq L_{\goodness}\|\bh - \bh^\prime\|$ for all $\bh,\bh^\prime\in\R^d$.
\end{assumption}

Our next step is to write the FFA update in a more compact form. 
For $i$-th training sample, define the flattened gradient $\boldsymbol{\psi}_i^{(\ell)}
    :=
    \mathrm{vec}\!\left(\nabla_{W^{(\ell)}}
    \goodness\!\left(\bh_i^{(\ell)}\right)\right)\in\reals^{d^{(\ell)}d^{(\ell-1)}}$, 
where $\mathrm{vec}(\cdotc)$ flattens a matrix into a vector.
Let $p^-(\bx) = 1 - \sigma(\goodness^{(\ell)}(\bx; W^{(\ell)}) - \theta)$, $p^+(\bx) = \sigma(\goodness^{(\ell)}(\bx; W^{(\ell)}) - \theta)$ be the probabilities of negative and positive classifications, respectively. 
In addition, define the \emph{goodness residual vector} $\br \in \reals^N$ with its $i$-th entry given by $[\br]_i=-\frac{1}{n^+}p^-(\bx_i)$ if $\bx_i\in \ccalD^+$ and $[\br]_i=\frac{1}{n^-}p^+(\bx_i)$ otherwise, where $n^+$ and $n^-$ are the number of positive and negative samples, respectively. With these notations, the FFA update takes a compact form, as shown in the following lemma.
\begin{lemma}\label{lem:gradient}
The FFA update can be written as $\mathrm{vec}(\nabla \cL^{(\ell)}) = \Phi\br$ for the matrix $\Phi\in\reals^{d^{(\ell)}d^{(\ell-1)}\times N}$ whose columns are the vectors $\boldsymbol{\psi}_i^{(\ell)}$.
Equivalently, $\nabla \cL^{(\ell)} = \sum_{i=1}^{N} [\br]_i \nabla_{W^{(\ell)}}\goodness\!\left(\bh_i^{(\ell)}\right)$.
\end{lemma}
The proof of Lemma~\ref{lem:gradient} is in Appendix~\ref{app:proof-lem-grad}.
Besides that, we also need to provide regularity conditions on the geometric structure of the FFA loss landscape, which is captured by the \emph{FFA Gram matrix}, denoted by $H^{(\ell)}:= \Phi^\top\Phi\in\reals^{N\times N}$, where $N$ is the number of training samples. 
The $(i, j)$-th entry of $H^{(\ell)}$ is given by
\begin{align}
    [H^{(\ell)}(W^{(\ell)})]_{ij}
    =&\ 
    \tr{\left(\nabla_{W^{(\ell)}}\goodness\!\left(\bh_i^{(\ell)}\right) \nabla_{W^{(\ell)}}\goodness\!\left(\bh_j^{(\ell)}\right)^\top\right)}
    =\left\langle
    \boldsymbol{\psi}_i^{(\ell)},\boldsymbol{\psi}_j^{(\ell)}
    \right\rangle,
\end{align}
The goodness-gradient Gram $H^{(\ell)}$ governs the PL constant; the sigmoid residual $\br$ concentrates gradient mass on misclassified samples since $|[\br]_i|\approx 1/n$ near the boundary and $|[\br]_i|\to 0$ when correctly classified. 

\begin{assumption}[Gram well-conditioned at initialization]\label{ass:gram}
    There exists a constant $\mu_H > 0$ such that the minimum eigenvalue of the FFA Gram matrix at initialization is lower bounded by
    $\lambda_{\min}(H^{(\ell)}(W_0^{(\ell)})) \geq \mu_H$.
\end{assumption}
\Cref{ass:gram} is the FFA counterpart of the NTK minimum eigenvalue condition~\cite{du2019gradient}, which holds with sufficiently wide networks with high probability, i.e., $d^{(\ell)} \geq \Omega(n/\delta)$ when $W_0^{(\ell)}$ is randomly initialized with probability $\geq 1-\delta$; see Appendix~\ref{app:assumptions} for further details.

Under the assumptions above, we can study the convergence of FFA. 
The convergence results hold within a local good neighborhood $\cW_R := \{W^{(\ell)} : \|W^{(\ell)} - W_0^{(\ell)}\|_F \leq R\}$, a ball around initialization where the Jacobian, sigmoid margins, and FFA Gram matrix remain controlled.
These are conditional local guarantees: they apply only while the training trajectory remains in $\cW_R$. This condition is most plausible near initialization in sufficiently wide lazy-training/NTK regimes; its persistence during long feature-learning training is not established here.

\begin{theorem}[Local PL for one FFA block]\label{thm:pl_single}

Suppose Assumptions~\ref{ass:bounded}--\ref{ass:gram} hold and the following dynamic threshold schedule is used
\begin{align}\label{eq:dyn_capped_schedule}
    \theta_k = \theta_0 + c_\theta\log(1+\min(k, K_{\max})),
    \qquad
    \theta_{\max}:=\theta_0+c_\theta\log(1+K_{\max}),
\end{align}
for some $K_{\max}\in\mathbb{N}_{\geq 1}$ and $c_\theta>0$. 
Within the local good neighborhood $\cW_R$,
it holds for some constant $\gamma > 0$ that $|[\br]_i| \geq \gamma$ for all $i$, and the FFA loss is $\beta^{(\ell)}$-smooth for some constant $\beta^{(\ell)} > 0$
and satisfies the Polyak--{\L}ojasiewicz (PL) inequality, where $\cL^{(\ell)*} := \min_{W^{(\ell)} \in \cW_R}\cL^{(\ell)}(W^{(\ell)})$:
\begin{align}\label{eq:pl}
    \|\nabla\cL^{(\ell)}(W^{(\ell)})\|_F^2 \geq 2\mu(\cL^{(\ell)}(W^{(\ell)}) - \cL^{(\ell)*}), \quad \mu := \mu_H\gamma^2.
\end{align}
\end{theorem}
The proof of Theorem~\ref{thm:pl_single} is deferred to Appendix~\ref{app:proof_pl_single}. The neighborhood radius $R$ must keep the actual goodness-gradient Gram well-conditioned and the scaled-block Jacobian bounded. 
The theorem suggests that the PL constant $\mu = \mu_H\gamma^2$ decomposes into a \emph{geometric} factor $\mu_H$ and a \emph{non-saturation} margin $\gamma^2$.
$\mu_H$ is the smallest eigenvalue of the goodness-gradient Gram matrix, encoding how well positive and negative features are separated in representation space.
$\gamma^2$ measures the distance from the degenerate all-or-nothing sigmoid regime. 
Together, if only one layer is updated and the others are frozen, they set the linear convergence rate: gradient descent at $\eta = 1/\beta^{(\ell)}$ achieves $\cL^{(\ell)}(W_k^{(\ell)}) - \cL^{(\ell)*} \leq (1 - \mu/\beta^{(\ell)})^k(\cL_0^{(\ell)} - \cL^{(\ell)*})$, with the condition number $\kappa = \beta^{(\ell)}/\mu$ setting the convergence timescale \cite{bottou2018optimization}. 

Single-layer convergence, however, is only the tractable half of the story. When all layers update simultaneously, each gradient step at layer $\ell$ shifts the input distribution seen by layer $\ell{+}1$, and these perturbations can amplify with depth. Controlling this amplification requires a bounded-sensitivity condition on the local loss's response to upstream changes.

\begin{assumption}[Bounded goodness sensitivity]\label{ass:sensitivity}
$\left\|\frac{\partial\cL^{(\ell)}}{\partial\bh^{(\ell-1)}}\right\| \leq G_{\max}$ for all $\ell$ and inputs.
\end{assumption}

\begin{theorem}[Multi-layer FFA convergence]\label{thm:multi_layer}\label{thm:resnet_ffa}
Suppose Assumptions \ref{ass:bounded}--\ref{ass:sensitivity} hold.
With per-block Lipschitz constant $\rho < \infty$, $\mu_{\min} := \min\{\mu^{(\ell)}: \ell\in[L]\}$, $\beta_{\max} := \max\{\beta^{(\ell)}: \ell\in[L]\}$, and diminishing step size $\eta_k \le c_0/(\rho^2 L)$, the trajectory falls into the local good neighborhood, i.e., $W^{(\ell)}\in\ccalW_R$ for all $\ell\in[L]$ and $k$ with radius
$R = \Theta(\sqrt{(\cL_0^{(\ell)} - \cL^{(\ell)*}) \beta^{(\ell)} / \mu^2})$.
In addition, the Lyapunov function $V_k := \frac{1}{L}\sum_{\ell\in[L]}(\cL_k^{(\ell)} - \cL^{(\ell)*})$ satisfies
\begin{align}\label{eq:multi_rate}
    V_k \;\le\; 
    \Bigl(1 - \frac{\eta\mu_{\min}}{2}\Bigr)^k V_0 + 
    O\!\left(\frac{e^{2\rho}\beta_{\max}}{\mu_{\min}^2}\right).
\end{align}
\end{theorem}
The proof is deferred to Appendix~\ref{sec:proof-multi-layer}.
The decaying term $(1-\eta\mu_{\min}/2)^k V_0$ establishes \emph{linear} convergence of the averaged sub-optimality gap toward a residual neighborhood: despite the multi-layer FFA not being the gradient of any global objective, the per-layer PL inequality still drives linear convergence. The residual bias $O(e^{2\rho}\beta_{\max}/\mu_{\min}^2)$ is a \emph{coupling error plateau}: each gradient step at layer $\ell$ shifts the block output by $O(C_W\eta_k/L)$, where $C_W$ is the LayerNorm-stable parameter sensitivity constant, and residual propagation amplifies the downstream drift by at most $e^\rho$. These perturbations accumulate into a steady-state error proportional to the coupling term. The exponential factor $e^{2\rho}$ arises from the Lipschitz composition of residual blocks, $(1+\rho/L)^L \le e^\rho$, and represents the geometric price of depth under purely local training without global error signals. 

Currently, theoretical analysis provides only an upper bound on the error. Whether the upper-bound error plateau is merely an analysis artifact remains an open theoretical question. However, in our experimental section, we present a case demonstrating two-phase convergence: an initial period of linear convergence followed by an error plateau.
This empirical evidence suggests that the error is not an analysis artifact.
We intend to investigate more rigorously in future work.

\noindent This completes our account of \emph{when} FFA converges.
We now ask the harder question: inside this regime, \emph{what} does FFA converge to, and why is that not enough?
\section{Why Local Learning Cannot Scale: Bottleneck and the Price of Locality}
\label{sec:representation}

Having established that FFA converges, we now present the central negative result: under normalization-stabilized architectures, \emph{FFA faces structural barriers to achieving parity with BP}. The negative result does not rely on a separate alignment theorem; instead, it combines representation-collapse evidence, the Price of Locality decomposition, and the residual-gap experiments to isolate what local training fails to control.
We build this conclusion in two steps. First, we show that deeper local representations can contract toward a shared rank-one matrix (Theorem~\ref{thm:kernel_contraction}). Second, the Price of Locality framework condenses this mechanism together with the residual-gap experiments into an objective-gap narrative and contrasts it with BP's positive scaling mechanism.

The residual downstream loss is the quantity that survives every negative-sampling fix; its non-vanishing must therefore come from a \emph{geometric} mismatch that persists even at perfect optimization. Our active mechanism for that mismatch is representation collapse: once deeper layers drift toward the same low-rank kernel matrix, additional local optimization no longer creates complementary task-relevant features. An earlier feature-reuse route based on a spectral-compatibility assumption has now been archived; the active paper keeps only the stronger point that feature \emph{redundancy} remains the dominant failure mode.

\paragraph{Kernel contraction and representation collapse.}
Fix $N$ training inputs $\{\bx_i : i\in[N]\}$. For depth-$\ell$ representation $\bh^{(\ell)}(\bx_i)$, define $\tilde{\bh}^{(\ell)}_i := \mathrm{LN}(\bh^{(\ell)}_i)$ so that $\|\tilde{\bh}^{(\ell)}_i\|^2 = d$. 
Define Gram kernel matrix of the depth-$\ell$ representations $\Sigma^{(\ell)} \in \R^{N \times N}$ by $[\Sigma^{(\ell)}]_{ij} = \tilde{\bh}^{(\ell)}(x_i)^\top \tilde{\bh}^{(\ell)}(x_j)/d$. Each entry $[\Sigma^{(\ell)}]_{ij}$ records how similar the layer's LN-normalized embeddings of two inputs $x_i,x_j$ are, so $\Sigma^{(\ell)}$ summarizes the full pairwise geometry that any quadratic-form readout.
We track this object because both the negative-quality term and the residual FFA task-loss term act through it: contrastive negatives must separate from positives in exactly these inner products, while downstream linear evaluation can only use the geometry retained in the same kernel matrix.

Therefore, the spectrum $\Sigma^{(\ell)}$ is a layer-wise encoding-capacity diagnostic. High rank with well-spread eigenvalues means the layer encodes diverse independent directions and can resolve various sample classes, whereas collapse toward the all-ones matrix $J_N$ means every pair of inputs receives the same similarity score, i.e., all samples project onto a single direction and the layer loses its discriminative power, which is a depth-induced degeneracy previously identified as \emph{rank collapse} in pure-attention transformers~\cite{dong2021attention} and \emph{oversmoothing} in deep graph networks~\cite{oono2020graph}, where it is counteracted by architectural components such as residual connections, MLP blocks, or normalization placement \cite{daneshmand2020batch}.

We show below that FFA's local objective cannot effectively exploit these components: once the kernel matrix drifts toward a low-rank state, later local updates keep refining nearly the same feature directions rather than restoring complementary ones. In this sense $\Sigma^{(\ell)}$ tracks the \emph{effective representational capacity} of layer $\ell$. The theorem below states the precise layer-product framework for diagnosing when this capacity is consumed geometrically.

\begin{theorem}[Kernel product rate and representation collapse]\label{thm:kernel_contraction}
Consider the scaled-ResNet architecture~\eqref{eq:resnet_arch} with an i.i.d. weight sequence $W^{(1)},W^{(2)},\ldots\sim\mathcal{P}_W$ satisfying $\E_{\mathcal{P}_W}[\log^+\|W\|_{\mathrm{op}}]<\infty$.
Let $T_W$ be the per-layer kernel transformation on unit-diagonal PSD kernels such that $T_{W^{(\ell+1)}}(\Sigma^{(\ell)}) = \Sigma^{(\ell+1)}$.
Let $A_W := \mathrm{D}T_W(J_N)$ be the linearization at the rank-one kernel $J_N$. Assume moreover that, in a neighborhood of $J_N$, the map $\Sigma \mapsto T_W(\Sigma)$ is $C^1$ and there exist deterministic constants $C_{\mathrm{lin}}, q_{\mathrm{lin}} < \infty$ such that $\|A_W\| \le C_{\mathrm{lin}}(1+\|W\|_{\mathrm{op}})^{q_{\mathrm{lin}}}$ almost surely. Then:
\begin{enumerate}[label=(\alph*),itemsep=2pt]
    \item \textbf{Stationary product rate:} the limit
    \begin{align}
        \lambda_1 :=
        \lim_{\ell\to\infty}\frac{1}{\ell}\log\bigl\|A_{W^{(\ell)}}A_{W^{(\ell-1)}}\cdots A_{W^{(1)}}\bigr\|
    \end{align}
    exists almost surely and is deterministic. Equivalently, $\bar\rho^*:= e^{\lambda_1}$ is the asymptotic per-layer rate of the common infinite extension. The physical depth $L$ only selects its finite prefix.
    \item \textbf{Finite-prefix collapse envelope:} if $\lambda_1<0$ and the finite prefix remains in the linearization neighborhood of $J_N$, then there exists an almost-sure finite random prefactor $C_{\mathrm{env}}(\omega)\ge 1$ such that
    \begin{align}
        \|\Sigma^{(\ell)}-J_N\|_F
        \le
        C_{\mathrm{env}}(\omega)e^{-|\lambda_1|\ell/2}\|\Sigma^{(0)}-J_N\|_F,
        \qquad 0\leq \ell\leq L.
    \end{align}
\end{enumerate}
\end{theorem}
The proof of Theorem~\ref{thm:kernel_contraction} is deferred to Appendix~\ref{app:kernel_contraction_proof}. 
Kernel contraction ($\Sigma^{(\ell)} \to J_N$) implies that the parameter-space FFA Gram matrix $\widehat H^{(\ell)} := \Phi\Phi^\top = \sum_{i\in[N]} \boldsymbol{\psi}_i^{(\ell)}(\boldsymbol{\psi}_i^{(\ell)})^\top \in \reals^{d^{(\ell)}d^{(\ell-1)}\times d^{(\ell)}d^{(\ell-1)}}$ also degrades in rank, because the gradient directions $\bv_i \propto \nabla_{\bh}\goodness(\bh^{(\ell)}_i) \otimes \bh^{(\ell-1)}_i$ depend on the representations $\bh^{(\ell)}_i$ that are converging to a common direction.
Formally, when $\|\Sigma^{(\ell)} - J_N\|_F \leq \epsilon$, the representations $\tilde{\bh}_\ell(x_i)$ satisfy $\|\tilde{\bh}_\ell(x_i) - \tilde{\bh}_\ell(x_j)\|^2 \leq 2d\epsilon$ for all $i,j$, so $\bv_i \approx \bv_j$ and $\erank(\widehat H^{(\ell)}) \to O(L)$.

The kernel theorem above describes the representation-side loss of capacity.
To compare it with the learning signal available to each algorithm, we now place both BP and FFA in a common error-signal-kernel framework.
Define $\Gamma^{(\ell),\mathcal{A}} := \E_{\bx}[\boldsymbol{\delta}^{(\ell),\mathcal{A}}(\bx)\,(\boldsymbol{\delta}^{(\ell),\mathcal{A}}(\bx))^T]$ for algorithm $\mathcal{A}\in\{\mathrm{BP},\mathrm{FFA}\}$, where $\boldsymbol{\delta}^{(\ell),\mathcal{A}}$ is the per-sample error signal that $\mathcal{A}$ feeds into layer $\ell$.
Under FFA, writing $\goodness_x^{(\ell)}(\bh)$ for the per-sample goodness with any local labels or targets frozen, differentiation of~\eqref{eq:ffa_loss} gives an error signal that is always a scalar multiple of the corresponding goodness gradient:
\begin{align*}
    \boldsymbol{\delta}^{(\ell),\mathrm{FFA}}(\bx)
    =
    \begin{cases}
        \bigl(1-\sigma(\goodness_{\bx}^{(\ell)}(\bh^{(\ell)}(\bx))-\theta)\bigr)\,\nabla_{\bh}\goodness_{\bx}^{(\ell)}(\bh^{(\ell)}(\bx)),
        & \bx \sim \cD^+,\\[4pt]
        -\,\sigma(\goodness_{\bx}^{(\ell)}(\bh^{(\ell)}(\bx))-\theta)\,\nabla_{\bh}\goodness_{\bx}^{(\ell)}(\bh^{(\ell)}(\bx)),
        & \bx \sim \cD^-.
    \end{cases}
\end{align*}
Hence $\Gamma^{(\ell),\mathrm{FFA}}$ is governed by the span of the per-sample goodness gradients. Kernel contraction forces this span to collapse to one direction for the rank-one collapsed-gradient subclass (recovering the norm-goodness argument below), while richer label- or target-dependent goodnesses retain exactly the residual diversity present in their collapsed goodness-gradient family. For any PSD matrix $M\succeq 0$, we define \emph{effective rank} as $\erank(M) := (\tr M)^2/\tr(M^2)\in[1,\rank(M)]$, which is a continuous relaxation of $\rank(\cdotc)$~\cite{royVetterli2007effective,rudelson2007stable}.

\paragraph{The positive theory: why backpropagation scales.} FFA fits the autonomous product picture above: once layer $\ell$ receives $\bh^{(\ell-1)}$, its update is driven by a local scalar goodness and has no access to the downstream task residual. BP breaks exactly this autonomy. Under end-to-end backpropagation, any drift of $\Sigma^{(\ell)}$ toward $J_N$ enlarges the task loss $\cL_{\mathrm{task}}$, and the enlargement is fed back to the layer-$\ell$ parameters via the per-sample chain rule
\begin{align*}
    \nabla_{\bW^{(\ell)}}\cL_{\mathrm{task}} = \frac{1}{N}\sum_{i\in[N]}\boldsymbol{\delta}^{(\ell),\mathrm{BP}}(\bx_i)\,(\bh^{(\ell-1)}(\bx_i))^{\top},\qquad
    \boldsymbol{\delta}^{(\ell),\mathrm{BP}}(\bx_i):= J^{(L:\ell+1)}(\bx_i)^{\top}\,\left.\frac{\partial\cL_{\mathrm{task}}}{\partial\bh^{(L)}}\right|_{\bx_i},
\end{align*}
with $J^{(L:\ell+1)}(\bx_i):=\prod_{k=\ell+1}^{L}\left.\frac{\partial\bh^{(k)}}{\partial\bh^{(k-1)}}\right|_{\bx_i}$. If two inputs have collapsed at $\ell$-th layer (i.e., $\bh^{(\ell)}(\bx_i)\approx\bh^{(\ell)}(\bx_j)$) but still require different output corrections, then the output residuals and hence the transported errors $\boldsymbol{\delta}^{(\ell),\mathrm{BP}}(\bx_i)$ and $\boldsymbol{\delta}^{(\ell),\mathrm{BP}}(\bx_j)$ differ. The resulting gradient contains sample-specific directions rather than a shared local-goodness direction, so the update is sensitive to precisely the pairwise distinctions that $\Sigma^{(\ell)}\to J_N$ would erase.

Sensitivity to the kernel spectrum is therefore produced by the chain rule itself. In the unified error-signal-kernel notation introduced above, define
\begin{align*}
    \Gamma^{(\ell),\mathrm{BP}} := \frac{1}{N}\sum_{i\in[N]}\boldsymbol{\delta}^{(\ell),\mathrm{BP}}(\bx_i)\,\boldsymbol{\delta}^{(\ell),\mathrm{BP}}(\bx_i)^{\top}\;\in\;\R^{d^{(\ell)}\times d^{(\ell)}}.
\end{align*}
We will show that $\Gamma^{(\ell),\mathrm{BP}}$ ``almost always'' preserves the rank during BP across depth. 
We begin from a simplified setting where the Jacobian product $J^{(L:\ell+1)}(\bx)$ is independent of the input $\bx$, so that $J^{(L:\ell+1)}(\bx) \equiv \bar J^{(L:\ell+1)}$ for some constant $\bar J^{(L:\ell+1)}$ for all $\bx$.
In this setting, the chain rule yields the congruence $\Gamma^{(\ell),\mathrm{BP}} = (\bar J^{(L:\ell+1)})^{\top}\,\Gamma^{(L)}\,\bar J^{(L:\ell+1)}$; whenever $\bar J^{(L:\ell+1)}$ is non-singular, the congruence preserves $\rank(\Gamma^{(\ell),\mathrm{BP}})=\rank(\Gamma^{(L)})$ at every depth. 
Consequently, BP avoids representation collapse and can update the layer using rich downstream representations.

In the rest of this section, we will formalize this intuition. The first step is to make an assumption that $J^{(L:\ell+1)}(\bx)$ is an approximation of some constant non-singular matrix $\bar J^{(L:\ell+1)}$.
\begin{assumption}[Jacobian concentration along the data]\label{ass:bp_jac_conc}
There exist a constant reference Jacobian product $\bar J^{(L:\ell+1)}\in\R^{d^{(\ell)}\times d^{(\ell)}}$ and a deviation level $\tau_J\geq 0$ such that, for every input $\bx$,
\begin{align}\label{eq:jac_concentration}
    \bigl\|\,J^{(L:\ell+1)}(\bx)\;-\;\bar J^{(L:\ell+1)}\,\bigr\|_{\mathrm{op}} \;\leq\; \tau_J.
\end{align}
In addition, for Jacobian $\|\widetilde{J}^{(\ell)}\|_{\mathrm{op}} \leq \bar{\rho}$, the reference Jacobian satisfies
\begin{align}
    e^{-\bar\rho}\;\leq\; \sigma_{\min}\bigl(\bar J^{(L:\ell+1)}\bigr)\;\leq\; \sigma_{\max}\bigl(\bar J^{(L:\ell+1)}\bigr) \;\leq\; e^{\bar\rho}.
\end{align}
\end{assumption}

Assumption~\ref{ass:bp_jac_conc} states that the per-sample backward signal stay close to a common deterministic reference, up to a uniform deviation $\tau_J$. 
The residual-Jacobian constant $\bar\rho$ in this assumption and the kernel-product rate $\bar\rho^*$ from Theorem~\ref{thm:kernel_contraction} control different random objects. 
To some extent, it is a strong assumption that is sufficient to guarantee that the backward signal preserves diversity, but it is not necessary: the backward Jacobian can be sample-specific and still preserve rank if the per-sample Jacobians are sufficiently well-aligned.
Relaxing this assumption will be left to future work.
With it, we reach the following theorem.

\begin{theorem}[BP error signal diversity]\label{thm:bp_esd}
Under \Cref{ass:bp_jac_conc}, it holds that:
(a)~$\rank(\Gamma^{(\ell),\mathrm{BP}}) = \rank(\Gamma^{(L)})$;
(b)~$\erank(\Gamma^{(\ell),\mathrm{BP}}) \geq \exp(-8\bar{\rho})\, \erank(\Gamma^{(L)})$.
\end{theorem}
The proof of Theorem~\ref{thm:bp_esd} is deferred to Appendix~\ref{sec:proof_bp_esd}.
Intuitively, Assumption~\ref{ass:bp_jac_conc} says that along the data support, the samplewise backward Jacobians remain close to a common well-conditioned transport map. As a result, BP propagates the output-layer supervision backward mainly by a shared change of coordinates: error directions can rotate and rescale, but they do not disappear. 
Part~(a) formalizes this exact rank preservation, while part~(b) shows that even the softer diversity measure given by effective rank can deteriorate by at most the depth-independent constant factor $\exp(-8\bar{\rho})$. This is the structural advantage that the next theorem aggregates across layers.
Together, Theorems \ref{thm:bp_esd} and \ref{thm:elc_scaling} let us compare the representation capabilities of BP and FFA.

To this end, we define the \emph{effective learning capacity (ELC)} of an algorithm $\mathcal{A}$ as the sum of the effective ranks of the per-layer error-signal Grams above the rank-one floor:
$\ELC(\mathcal{A}):= \sum_{\ell\in [L]} \bigl(\erank(\Gamma^{(\ell),\mathcal{A}})-1\bigr)$, which represents the total number of independent directions of error signal that the algorithm can provide across all layers.
The subtracted $1$ in the definition accounts for the fact that even a rank-one Gram contributes one direction of error signal, so we only count the excess diversity above this trivial contribution.

\begin{theorem}[Scaling via effective learning capacity]\label{thm:elc_scaling}
    It holds that:
    (a)~$\ELC({\mathrm{BP}}) \geq L\max\bigl\{0,\,\exp(-8\bar{\rho})\,\erank(\Gamma^{(L)})-1\bigr\} = \Omega(L)$.
    (b)~$\ELC({\mathrm{FFA}}) \le O\lp\frac{\sqrt{\bar\rho^*}}{1-\sqrt{\bar\rho^*}}\rp = O(1)$.
\end{theorem}
The proof of Theorem~\ref{thm:elc_scaling} is deferred to Appendix~\ref{app:proof-elc_scaling}.
Part~(a) follows by summing the per-layer bound from Theorem~\ref{thm:bp_esd} over all $L$ layers, because BP's chain rule propagates the output-layer error signal backward without rank loss (up to the constant factor $e^{-8\bar\rho}$).
If $\bar\rho$ is sufficiently small, every layer retains $\Omega(\Gamma^{(L)})$ independent gradient directions, and the network as a whole accumulates $\Omega(L)$ independent learning directions.
In contrast, Part~(b) reflects the consequence of kernel contraction: once $\Sigma^{(\ell)}\to J_N$, the FFA error signal at layer~$\ell$ is determined by goodness gradients whose directions collapse toward a single vector, giving $\erank(\Gamma^{(\ell),\mathrm{FFA}})=1+O((\sqrt{\bar\rho^*})^\ell)$ and summing up to a $L$-independent upper bound $O(\sqrt{\bar\rho^*}/(1-\sqrt{\bar\rho^*}))$.
Consequently, BP's learning diversity grows linearly with depth, while FFA's total learning diversity is bounded by a constant independent of depth.
From this perspective, BP's advantage is not just an optimization issue but also a representational one.

\begin{figure}[t]
\centering
\caption{
    \textbf{Left:} Multi-layer FFA convergence in the toy model. The vertical dotted line marks $k=10^3$.
    \textbf{Right:} Test accuracy curve of ResNet-18 @ CIFAR-10 training.}
\includegraphics[width=0.42\textwidth]{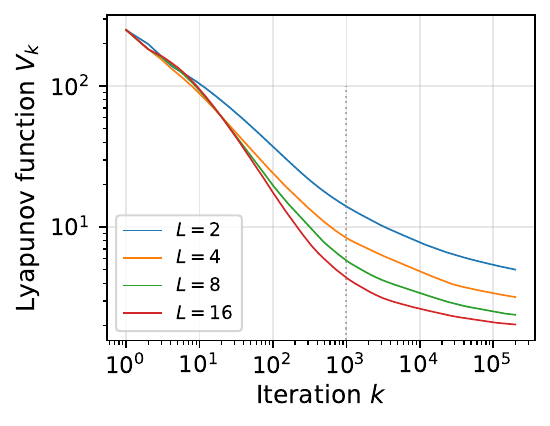}
\includegraphics[width=0.43\textwidth]{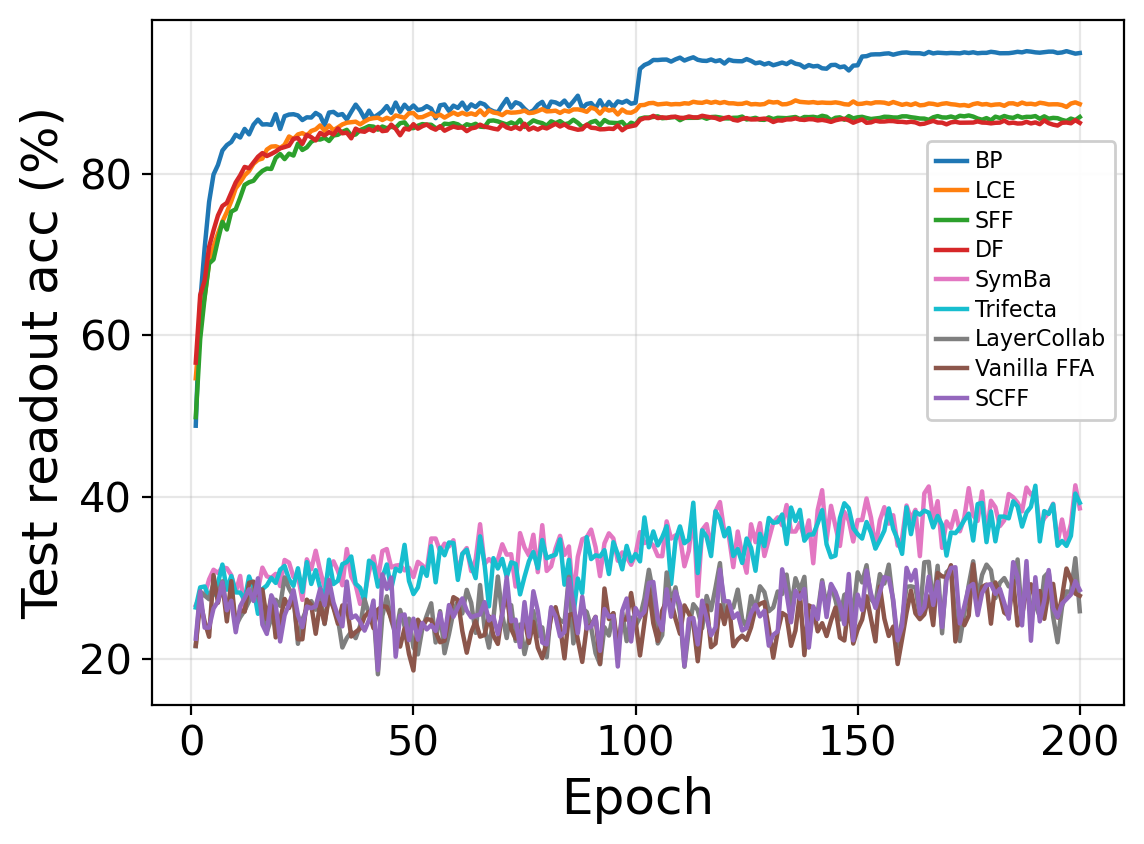}
\vspace{-5pt}
\label{fig:exp2}
\vspace{-1em}
\end{figure}
\section{Empirical Results}
\label{sec:experiments}
This section compares three groups of training algorithms: BP, FFA variants, and other local learning methods. Our BP baseline is standard end-to-end backpropagation. For the FFA family, we present representative state-of-the-art variants including Vanilla FFA~\cite{hinton2022forward}, SymBa~\cite{lee2023symba}, Self-Contrastive Forward-Forward (SCFF)~\cite{chen2024selfcontrastive}, Trifecta~\cite{dooms2023trifecta}, and Layer Collaboration~\cite{lorberbom2024layer}. We also consider other local supervised learning methods, including Local Cross-Entropy (LCE)~\cite{nokland2019local}, Scalable Forward-Forward (SFF)~\cite{krutsylo2025scalable}, and Distance-Forward (DF)~\cite{wu2024distance}. Within each group, we report only a representative subset rather than an exhaustive list of variants; see Appendix~\ref{app:cifar10_cnn_bench} for a broader comparison. Extended results and detailed experimental settings are in Appendix~\ref{app:extended_experiments}.
The code is available at \url{https://github.com/Zhaoxian-Wu/ffa-theory}

\paragraph{Toy multi-layer convergence.}
The left panel of Figure~\ref{fig:exp2} plots the convergence of the Lyapunov function $V_k$ on a log--log scale for $L \in \{2,4,8,16\}$ in the multi-layer toy model. Two phases are visible: an initial linear descent for $k < 10^3$, followed by a slowing down at a depth-dependent error floor. This empirically confirms the claim of Theorem~\ref{thm:multi_layer} that FFA converges only to a non-vanishing error plateau. See Appendix~\ref{app:exp2_toy_setup} for more details on the experimental setup.

\paragraph{Image classification.}
Figure~\ref{fig:cifar10_cnn_benchmark} summarizes CIFAR-10 classification accuracy across ResNet depths (18/24/56) using a detached readout on the learned representations. BP remains the strongest method at every depth ($0.95$--$0.96$). Local-learning-based FFA methods, including LCE, SFF, and DF, are consistently below BP by roughly $5$--$8$ percentage points (pp), while goodness-based FFA variants remain much lower at $0.31$--$0.42$. Thus, even on a moderate vision task, both classes of FFA methods underperform BP, with the gap being especially severe for goodness-based FFA rules. The right panel of Figure~\ref{fig:exp2} presents the convergence curves of various methods. The curves fall into three groups: BP and local-based methods increase steadily, while a persistent gap remains between them. In contrast, the accuracy of goodness-based methods increases slowly, which verifies the gap between FFA and BP.

\begin{figure}[t]
\centering
\caption{CIFAR-10 classification accuracy across ResNet depths (ResNet18/24/56), measured by a detached readout. BP remains best at every depth; supervised local-learning methods trail by a moderate margin, while goodness-based FFA methods trail substantially.}
\includegraphics[width=0.8\textwidth]{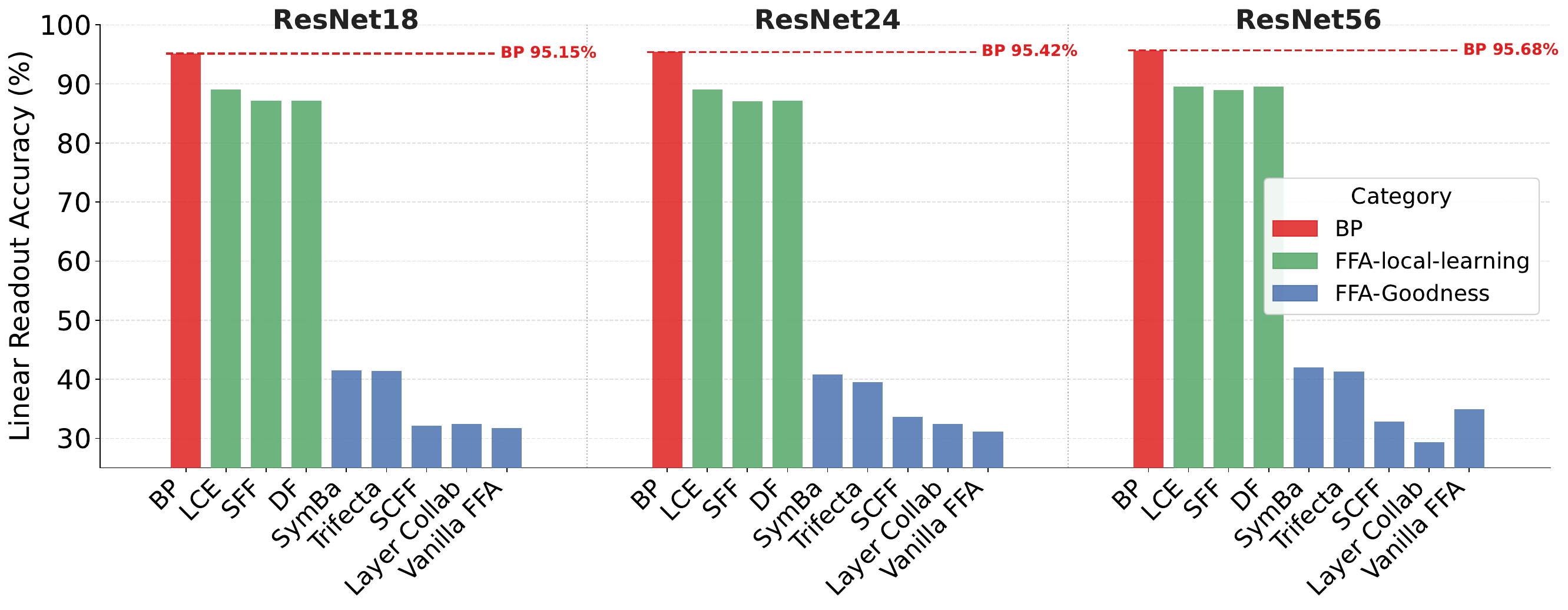}
\vspace{-9pt}
\label{fig:cifar10_cnn_benchmark}
\vspace{-1em}
\end{figure}

\paragraph{Controlled locality intervention.}
We directly study the price of locality by controlling the partitioning of a fixed 12-layer CNN backbone into contiguous blocks.
The gradient signal can flow freely within each block but is blocked across blocks, so the $12\mathbin{\times}1$ partition is strictly local, and the $1\mathbin{\times}12$ partition is fully global.
Table~\ref{tab:locality_intervention} reports the resulting controlled grouped-FFA sweep; see Appendix~\ref{app:cnn12_locality_intervention} for details. Every non-local grouped-FFA treatment improves best accuracy, reduces its final FFA loss, and increases mean layerwise error-signal effective rank relative to the strictly local control. The results therefore support the narrower causal conclusion that relaxing strict layer locality within this fixed grouped-FFA design can mitigate low-diversity learning signals and recover substantial accuracy. 

\begin{table}[t]
\centering
\caption{Controlled locality intervention on CIFAR-10 with CNN12. Entries are mean $\pm$ sample standard deviation over three seeds.}
\label{tab:locality_intervention}
\small
\setlength{\tabcolsep}{3.2pt}
\begin{tabular}{lccc}
\toprule
Blocks $\mathbin{\times}$ layers/block & Best acc. & Final loss & Mean $\erank(\Gamma)$ \\
\midrule
$12\mathbin{\times}1$ & $53.02\pm0.48\%$ & $1.1014\pm0.0034$ & $2.198\pm0.042$ \\
$6\mathbin{\times}2$  & $58.09\pm1.85\%$ & $1.0455\pm0.0068$ & $5.146\pm0.592$ \\
$4\mathbin{\times}3$  & $67.57\pm4.35\%$ & $1.0064\pm0.0158$ & $4.856\pm0.552$ \\
$3\mathbin{\times}4$  & $72.22\pm1.07\%$ & $0.9795\pm0.0018$ & $5.135\pm2.164$ \\
$2\mathbin{\times}6$  & $76.21\pm0.48\%$ & $0.9657\pm0.0027$ & $7.613\pm3.828$ \\
$1\mathbin{\times}12$ & $78.74\pm4.06\%$ & $0.9550\pm0.0199$ & $7.975\pm1.214$ \\
\midrule
BP & $85.67\pm0.09\%$ & $5.87\times10^{-5}\pm7.13\times10^{-5}$ & $31.654\pm3.793$ \\
\bottomrule
\end{tabular}
\vspace{-.8em}
\end{table}

\paragraph{Transformer pre-training.}
To compare BP and FFA on more challenging tasks, we conducted Transformer pre-training experiments under Chinchilla-style compute budgets~\cite{hoffmann2022training} on OpenWebText. 
We select FFA and LCE as the representative goodness-based and local-learning-based methods, respectively. 
Table~\ref{tab:chinchilla} shows that BP scales smoothly from $138$ to $28$ perplexity as the model grows.
While absolute perplexities decrease across all methods, the multiplicative gap between LCE/NCE and BP \emph{widens} with scale (NCE/BP ranges from $3.73\times$ to $9.19\times$, LCE/BP from $1.20\times$ to $1.51\times$), indicating that larger models amplify the relative performance gap. The language-model setting therefore shows the largest and most persistent gap: increasing scale helps BP much more than it helps local objectives.

\begin{table}[ht!]
\centering
\vspace{-1em}
\caption{Perplexity of Transformer pre-training on OpenWebText under Chinchilla-style budgets.
}
\label{tab:chinchilla}
\small
\begin{threeparttable}
\begin{tabular}{lrrrrrrrr}
\toprule
Scale & Params & Budget & BP & FFA & LCE & FFA/BP & LCE/BP \\
\midrule
tiny   ($L{=}2$) &  7M & 0.14B & 138 & 516 & 166 & 3.73$\times$ & 1.20$\times$ \\
small  ($L{=}4$) & 16M & 0.32B &  68 & 313 &  84 & 4.59$\times$ & 1.22$\times$ \\
medium ($L{=}6$) & 30M & 0.60B &  48 & 275 &  64 & 5.74$\times$ & 1.33$\times$ \\
large  ($L{=}8$) & 51M & 1.02B &  39 & 256 &  54 & 6.55$\times$ & 1.38$\times$ \\
xlarge ($L{=}12$) & 124M & 2.48B &  28 & 258 &  42 & 9.19$\times$ & 1.51$\times$ \\
\bottomrule
\end{tabular}
\end{threeparttable}
\vspace{-1em}
\end{table}

\paragraph{Controlled spectral intervention.}
To test whether low-rank update geometry is mechanistically linked to the accuracy gap, we intervened directly on every convolutional update $\Delta W=U\operatorname{diag}(s_1,\ldots,s_q)V^\top$ in a CNN3 CIFAR-10 benchmark, while holding the backbone, data, and 200-epoch budget fixed and retaining each algorithm's original objective and optimizer.
For $p\in[0, 1]$, we replaced its singular values by $\lambda_i(p)=\sqrt{(1-p)s_i^2+p\bar\lambda}$, where $\bar\lambda=q^{-1}\sum_j s_j^2$, and used $U\operatorname{diag}(\lambda_1(p),\ldots,\lambda_q(p))V^\top$ as the update.
This interpolation preserves the Frobenius norm, recovers the original algorithm at $p=0$, and produces a full-rank flat spectrum, a Muon-like update direction, at $p=1$.
Moreover, the participation-ratio effective rank of each transformed update is non-decreasing in $p$.
Table~\ref{tab:spectral_intervention} presents the complete BP/FFA interpolation sweep with the Jacobian-concentration diagnostic.
This result provides intervention evidence that low-rank update geometry is a contributing bottleneck, while the remaining $21.23$-point BP--FFA gap at $p=1$ shows that spectral flattening alone does not address the locality issues.

\begin{table}[t]
\centering
\caption{
    \textbf{(Left)} Controlled spectral interpolation on CNN3/CIFAR-10.
    Entries report best test accuracy and terminal mean layerwise participation-ratio $\erank(\Gamma)$.
    \textbf{(Right)} Finite-sample Jacobian-concentration diagnostics of the residual model \eqref{eq:resnet_arch}.
    Each entry reports mean $\pm$ sample standard deviation over three seeds.}
\begin{subtable}[t]{0.68\textwidth}
\centering
\small
\begin{tabular}{rcccc}
\toprule
& \multicolumn{2}{c}{Best accuracy} & \multicolumn{2}{c}{$\erank(\Gamma)$} \\
\cmidrule(lr){2-3}\cmidrule(lr){4-5}
$p$ & BP & FFA & BP & FFA \\
\midrule
$0.0$  & $77.56{\pm}0.75\%$ & $47.40{\pm}0.98\%$ & $42.78{\pm}2.27$ & $3.43{\pm}0.49$ \\
$0.2$  & $78.80{\pm}0.39\%$ & $53.41{\pm}1.25\%$ & $50.75{\pm}2.04$ & $3.85{\pm}0.24$ \\
$0.4$  & $79.08{\pm}0.17\%$ & $55.30{\pm}1.21\%$ & $53.01{\pm}0.75$ & $4.27{\pm}0.43$ \\
$0.6$  & $79.11{\pm}0.27\%$ & $55.87{\pm}1.65\%$ & $53.85{\pm}3.61$ & $3.95{\pm}0.83$ \\
$0.8$  & $79.13{\pm}0.43\%$ & $57.27{\pm}1.15\%$ & $50.84{\pm}1.38$ & $4.49{\pm}0.64$ \\
$1.0$  & $78.97{\pm}0.29\%$ & $57.74{\pm}0.50\%$ & $44.48{\pm}1.96$ & $4.42{\pm}0.61$ \\
\bottomrule
\end{tabular}
\end{subtable}
\begin{subtable}[t]{0.29\textwidth}
\centering
\begin{tabular}{rc}
\toprule
$L$ & $\tau_J$ \\
\midrule
$4$  & $1.144\pm0.070$ \\
$8$  & $0.742\pm0.080$ \\
$16$ & $0.461\pm0.049$ \\
$32$ & $0.318\pm0.026$ \\
$64$ & $0.206\pm0.014$ \\
\bottomrule
\end{tabular}
\end{subtable}
\label{tab:spectral_intervention}
\label{tab:jacobian_concentration}
\end{table}

\paragraph{Empirical Jacobian concentration.}
To examine Assumption~\ref{ass:bp_jac_conc} directly, we measured Jacobian concentration with standard Kaiming initialization in the residual model \eqref{eq:resnet_arch}.
For each of $N=512$ inputs $\bx_i$ independently sampled from the standard Gaussian distribution and each layer index $\ell$, we set
$\bar J^{(L:\ell+1)} := \frac{1}{N}\sum_{i=1}^{N} J^{(L:\ell+1)}(\bx_i)$, and estimate $\tau_J$ as the maximum deviation of any single Jacobian from the mean: $D_i^{(\ell)} := \left\|J^{(L:\ell+1)}(\bx_i)-\bar J^{(L:\ell+1)}\right\|_{\mathrm{op}}$.
For each depth, we report the conservative maximum over the measured layer indices of the sample maximum; Table~\ref{tab:jacobian_concentration} gives the mean $\pm$ sample standard deviation over three seeds. The results show that the Jacobian products concentrate well, especially in deep models.
\section{Conclusions and Limitations}
\label{sec:conclusion}

This paper provides the first rigorous convergence theory for the Forward-Forward Algorithm (FFA) and identifies two coupled structural mechanisms behind its persistent gap to backpropagation: an \textbf{optimization floor} from concurrent layer updates, and \textbf{geometric representational collapse} from kernel contraction.
We first prove that FFA satisfies the PL inequality at each layer and converges at a linear rate to a depth-dependent error floor (Theorems~\ref{thm:pl_single},~\ref{thm:multi_layer})---a moving-target residual that concurrent layer updates cannot eliminate.
Under the i.i.d. kernel-product hypotheses of Theorem~\ref{thm:kernel_contraction}, we further prove that the representation kernel $\Sigma^{(\ell)}$ contracts exponentially toward rank one at rate $\bar\rho^*$, leaving FFA with only a depth-summable excess learning capacity, while BP preserves capacity proportional to depth (Theorem~\ref{thm:elc_scaling}).
Empirical results verify that the FFA underperforms BP on a wide range of benchmarks, and that the gap widens with depth and scale.

\textbf{Limitation.} 
Our diagnosis isolates two provable, structural mechanisms that explain the observed FFA--BP gap, while it may not be an exhaustive explanation for all failures of FFA.
We leave exploration of these additional factors, as well as translating the diagnosis into a new local learning algorithm with provable improvements, to future work. 
\bibliographystyle{unsrt}
\bibliography{references}

%% file: appendix/A_proofs.tex
\section{Useful Lemmas and their Proof}
\label{app:useful-lemmas}

\subsection{Proof of Lemma~\ref{lem:gradient} (Gradient Expression)}
\label{app:proof-lem-grad}
\begin{proof}
For a positive sample, direct computation using the chain rule and the sigmoid derivative $\sigma'(z) = \sigma(z)(1-\sigma(z))$ yields
\begin{align}
    &\ \nabla_{W^{(\ell)}} \left[-\log\sigma(\goodness^{(\ell)} - \theta)\right] \\
    =&\ -\frac{\sigma'(\goodness^{(\ell)} - \theta)}{\sigma(\goodness^{(\ell)} - \theta)} \nabla_{W^{(\ell)}} \goodness^{(\ell)}
    = -(1-\sigma(\goodness^{(\ell)} - \theta)) \nabla_{W^{(\ell)}} \goodness^{(\ell)}
    = -p^-(\bx) \nabla_{W^{(\ell)}} \goodness^{(\ell)}.
    \nonumber
\end{align}
For a negative sample, the sign from differentiating
$\theta-\goodness^{(\ell)}$ cancels the sign from the logistic loss:
\begin{align}
    \nabla_{W^{(\ell)}}
    \left[-\log\sigma\bigl(\theta-\goodness^{(\ell)}\bigr)\right] 
    =&\ \left(1-\sigma\bigl(\theta-\goodness^{(\ell)}\bigr)\right)
    \nabla_{W^{(\ell)}}\goodness^{(\ell)}
    = p^+(\bx)\nabla_{W^{(\ell)}}\goodness^{(\ell)}.
\end{align}
Denote the positive and negative samples sets as $\cD^+=\{\bx_i^+:i\in[n^+]\}$ and $\cD^-=\{\bx_j^-:i\in[n^-]\}$, respectively. The two empirical expectations therefore give
\begin{align}
    &\ \nabla \cL^{(\ell)}(W^{(\ell)}) \\
    =&\ -\frac{1}{n^+}\sum_{i=1}^{n^+}
    p^-(\bx_i^+)\nabla_{W^{(\ell)}}\goodness^{(\ell)}(\bx_i^+;W^{(\ell)})
    +\frac{1}{n^-}\sum_{i=1}^{n^-}
    p^+(\bx_i^-)\nabla_{W^{(\ell)}}\goodness^{(\ell)}(\bx_i^-;W^{(\ell)}) 
    \nonumber\\
    =&\ \sum_{i=1}^{n^+}[\br]_i
    \nabla_{W^{(\ell)}}\goodness^{(\ell)}(\bx_i^+;W^{(\ell)})
    +\sum_{i=1}^{n^-}[\br]_{i}
    \nabla_{W^{(\ell)}}\goodness^{(\ell)}(\bx_i^-;W^{(\ell)})
    = \sum_{i=1}^{N}[\br]_i
    \nabla_{W^{(\ell)}}\goodness\!\left(\bh_i^{(\ell)}\right).
    \nonumber
\end{align}
Here the second equality is precisely the definition $[\br]_i=-p^-(\bx_i^+)/n^+$ for positive samples and
$[\br]_{n^++i}=p^+(\bx_i^-)/n^-$ for negative samples.
Finally, applying $\mathrm{vec}(\cdot)$ and using that the $s$th column of $\Phi$ is $\boldsymbol{\psi}_s^{(\ell)}$ gives $\mathrm{vec}\!\left(\nabla\cL^{(\ell)}(W^{(\ell)})\right)
    = \sum_{i=1}^{N}[\br]_i\boldsymbol{\psi}_i^{(\ell)}
    = \Phi\br$.
\end{proof}
\subsection{Lemma \ref{thm:dynamic_theta} and its proof (Conditional dynamic threshold tracking)}

\begin{lemma}[Conditional dynamic threshold tracking]\label{thm:dynamic_theta}
Under Assumptions~\ref{ass:bounded}--\ref{ass:gram}, consider gradient descent with step size $\eta=1/\beta^{(\ell)}$ on the dynamic threshold schedule \eqref{eq:dyn_capped_schedule}.
For each positive sample define $\Delta_{i,k}^+=\goodness_i(W_k^{(\ell)})-\theta_k$, and
analogously $\Delta_{j,k}^-=\goodness_j(W_k^{(\ell)})-\theta_k$ for negative
samples. Define the explicit constants
\begin{align}\label{eq:dyn_constants}
    L_0       &:= \cL^{(\ell)}(W_0;\theta_0)+c_\theta\log(1+K_{\max}),\\
    M_g^*     &:= M_\goodness B_J^{(\ell)}\sqrt{\frac{2K_{\max} L_0}{\beta^{(\ell)}}},\\
    M_+       &:= \max_{i\in\cD^+}\Delta_{i,0}^+ + M_g^*,
\end{align}
where $M_\goodness$ is the derived goodness-gradient bound in~\ref{lem:derived_M_goodness} and $B_J^{(\ell)}$ is the parameter-Jacobian
norm of the local block on $\cW_R$. Assume the iterates stay in $\cW_R$,
and the following \emph{single, single-pass verifiable} initial-state
condition holds: there exists a target negative non-degeneracy level
$\gamma_-\in(0,1/2)$ such that for every negative sample $j$,
\begin{align}\label{eq:dyn_init_combined}
    \Delta_{j,0}^-
    \;\geq\;
    \log\frac{\gamma_-}{1-\gamma_-}
    + M_g^* + (\theta_{\max}-\theta_0).
\end{align}
Then it holds that $|[\br]_i| \geq \gamma$and misclassification probabilities satisfies $n^+|[\br]_i| \leq 1-\gamma$ and $n^-|[\br]_i| \leq 1-\gamma$.
\end{lemma}

\begin{proof}[Proof of Lemma~\ref{thm:dynamic_theta}]

\textbf{Step 1: bounded loss and bounded total drift on $[0,K_{\max}]$.}
Direct computation of $\partial_\theta\cL^{(\ell)}$ gives
\begin{align}\label{eq:dyn_dthetaL}
    \partial_\theta\cL^{(\ell)}(W;\theta)
    = \frac{1}{n^+}\sum_i p^-(\bx_i^+;W,\theta)
      -\frac{1}{n^-}\sum_j p^+(\bx_j^-;W,\theta)
    \in[-1,1],
\end{align}
so $|\cL^{(\ell)}(W;\theta')-\cL^{(\ell)}(W;\theta)|\leq|\theta'-\theta|$
uniformly in $W\in\cW_R$. The envelope theorem
($\cL^{(\ell)*}(\theta)=\inf_{W\in\cW_R}\cL^{(\ell)}(W;\theta)$ is a
pointwise infimum of $1$-Lipschitz functions of $\theta$, hence
$1$-Lipschitz) extends the same bound to $\cL^{(\ell)*}(\theta)$.

By $\beta^{(\ell)}$-smoothness of $\cL^{(\ell)}(\cdot;\theta_k)$ at fixed
$\theta_k$ and step size $\eta=1/\beta^{(\ell)}$,
\begin{align}
    \cL^{(\ell)}(W_{k+1};\theta_k)
    \leq
    \cL^{(\ell)}(W_k;\theta_k)
    -\frac{1}{2\beta^{(\ell)}}\|\nabla\cL_k\|_F^2.
\end{align}
Combining with the $1$-Lipschitz $\theta$-dependence,
\begin{align}
    \cL^{(\ell)}(W_{k+1};\theta_{k+1})
    \leq
    \cL^{(\ell)}(W_k;\theta_k)
    -\frac{1}{2\beta^{(\ell)}}\|\nabla\cL_k\|_F^2
    + c_\theta\log\frac{k+2}{k+1}\cdot\mathbb{I}[k<K_{\max}].
\end{align}
Telescoping over $k=0,\dots,K_{\max}-1$ and using $\cL^{(\ell)}\geq 0$,
\begin{align}\label{eq:dyn_grad_sumsq}
    \frac{1}{2\beta^{(\ell)}}\sum_{k=0}^{K_{\max}-1}\|\nabla\cL_k\|_F^2
    \leq
    \cL^{(\ell)}(W_0;\theta_0)+c_\theta\log(1+K_{\max})
    = L_0,
\end{align}
and along the way
\begin{align}\label{eq:dyn_loss_bound}
    \cL^{(\ell)}(W_k;\theta_k)\leq L_0
    \qquad\text{for all } k\leq K_{\max}.
\end{align}
For $k>K_{\max}$, $\theta_k=\theta_{\max}$ is frozen, so the
fixed-threshold descent inequality alone keeps $\cL^{(\ell)}(W_k;\theta_{\max})$
non-increasing, and~\eqref{eq:dyn_loss_bound} continues to hold.

By Cauchy--Schwarz applied to~\eqref{eq:dyn_grad_sumsq},
\begin{align}
    \sum_{k=0}^{K_{\max}-1}\|\nabla\cL_k\|_F
    \leq \sqrt{K_{\max}}\Bigl(\sum_{k=0}^{K_{\max}-1}\|\nabla\cL_k\|_F^2\Bigr)^{1/2}
    \leq \sqrt{2\beta^{(\ell)} K_{\max} L_0}.
\end{align}
On $\cW_R$, the goodness map $W^{(\ell)}\mapsto\goodness_i(W^{(\ell)})$ is
$M_\goodness B_J^{(\ell)}$-Lipschitz (the derived gradient bound~\ref{lem:derived_M_goodness}
composed with a parameter Jacobian of operator-norm
$B_J^{(\ell)}$), so for every sample $i$,
\begin{align}\label{eq:dyn_total_drift}
    \sum_{k=0}^{K_{\max}-1}|\goodness_i(W_{k+1})-\goodness_i(W_k)|
    \leq M_\goodness B_J^{(\ell)}\eta\sum_{k}\|\nabla\cL_k\|_F
    \leq M_\goodness B_J^{(\ell)}\sqrt{\frac{2K_{\max} L_0}{\beta^{(\ell)}}}
    = M_g^*.
\end{align}
For $k>K_{\max}$, descent at frozen $\theta_{\max}$ only contributes further
bounded drift, but the relevant comparison below is to $\theta_{\max}$, so
$M_g^*$ as defined remains the operative bound.

\textbf{Step 2: per-sample margin bounds.}
By~\eqref{eq:dyn_total_drift} and the monotone threshold
$\theta_k\geq\theta_0$,
\begin{align}
    \Delta_{i,k}^+
    = \Delta_{i,0}^+ + (\goodness_i(W_k)-\goodness_i(W_0)) - (\theta_k-\theta_0)
    \leq \Delta_{i,0}^+ + M_g^*
    \leq M_+
    \qquad\forall i\in\cD^+,\;k\geq 0.
\end{align}
Hence $p^-(\bx_i^+;W_k,\theta_k)=\sigma(-\Delta_{i,k}^+)\geq\sigma(-M_+)$.

\emph{Saturation gaps from bounded loss (former bounded-loss part of (D2)).}
Each non-negative per-sample term is at most $\cL^{(\ell)}(W_k;\theta_k)\leq L_0$,
so for positives,
\begin{align}
    \tfrac{1}{n^+}\log(1+e^{-\Delta_{i,k}^+})\leq L_0
    \;\Longrightarrow\;
    \Delta_{i,k}^+\geq -\log(e^{n^+L_0}-1),
\end{align}
hence $1-p^-(\bx_i^+;W_k,\theta_k)=\sigma(\Delta_{i,k}^+)\geq e^{-n^+L_0}$.
Symmetrically $1-p^+(\bx_j^-;W_k,\theta_k)\geq e^{-n^-L_0}$ for negatives.

\emph{Negative lower bound (former~(D2) non-degeneracy).}
By~\eqref{eq:dyn_total_drift}, $\goodness_j(W_k)\geq\goodness_j(W_0)-M_g^*$,
hence
\begin{align}
    \Delta_{j,k}^-
    \geq \Delta_{j,0}^- - M_g^* - (\theta_k-\theta_0)
    \geq \Delta_{j,0}^- - M_g^* - (\theta_{\max}-\theta_0)
    \stackrel{\eqref{eq:dyn_init_combined}}{\geq}
    \log\frac{\gamma_-}{1-\gamma_-},
\end{align}
hence $p^+(\bx_j^-;W_k,\theta_k)=\sigma(\Delta_{j,k}^-)\geq\gamma_-$ for all
$k\geq 0$.
Putting together the positive and negative bounds, it holds that
\begin{align}
    \frac{\sigma(-M_+)}{n^+}\leq |[\br]_i^+|\leq \frac{1-e^{-n^+L_0}}{n^+},
    \qquad
    \frac{\gamma_-}{n^-}\leq |[\br]_j^-|\leq \frac{1-e^{-n^-L_0}}{n^-},
\end{align}
\end{proof}

\subsection{Lemma~\ref{lem:smooth} and its proof (Loss smoothness for a local block)}
\begin{lemma}[Loss smoothness for a local block]\label{lem:smooth}
Under Assumptions~\ref{ass:bounded} and~\ref{ass:G_smooth}, in a local good region where the block map
$W^{(\ell)}\mapsto T^{(\ell)}_{W^{(\ell)}}(\bh^{(\ell-1)})$ has parameter
Jacobian norm at most $B_J^{(\ell)}$ and parameter-Jacobian Lipschitz constant
$L_J^{(\ell)}$, each $\cL^{(\ell)}$ is $\beta^{(\ell)}$-smooth with
\begin{align}\label{eq:smooth_const}
    \beta^{(\ell)}
    \leq
    \tfrac14 M_{\goodness}^2 (B_J^{(\ell)})^2
    + L_{\goodness}(B_J^{(\ell)})^2
    + M_{\goodness}L_J^{(\ell)}.
\end{align}
\end{lemma}
\begin{proof}
For a single sample, the per-sample gradient is
$g(W) = -\alpha(W)\nabla_{W}\goodness$ with
$\alpha = 1 - \sigma(\goodness - \theta)$ (positive) or
$\alpha = \sigma(\goodness - \theta)$ (negative). Splitting the difference:
\begin{align*}
    g(W) - g(W') = \underbrace{-(\alpha(W){-}\alpha(W'))\nabla_W\goodness(W)}_{\text{(I)}} - \underbrace{\alpha(W')(\nabla_W\goodness(W) {-} \nabla_{W'}\goodness(W'))}_{\text{(II)}}.
\end{align*}
\textbf{Term (I):} $|\alpha(W)-\alpha(W')| \leq \frac{1}{4}|\goodness(W)-\goodness(W')| \leq \frac{1}{4}M_{\goodness}B_J^{(\ell)}\|W-W'\|_F$. Combined with $\|\nabla_W\goodness\|_F \leq M_{\goodness}B_J^{(\ell)}$, this gives the first term in~\eqref{eq:smooth_const}.

\textbf{Term (II):} $|\alpha|\leq1$ and the chain rule gives
\begin{align}
    \|\nabla_W\goodness(W)-\nabla_{W'}\goodness(W')\|_F
    \le
    L_{\goodness}(B_J^{(\ell)})^2\|W-W'\|_F
    + M_{\goodness}L_J^{(\ell)}\|W-W'\|_F .
\end{align}
Summing and taking expectations yields~\eqref{eq:smooth_const}.
\end{proof}

\subsection{Lemma~\ref{lem:perturbation} and its proof (Recursive Perturbation Bound)}

\begin{lemma}[Recursive perturbation bound]\label{lem:perturbation}
If Assumption~\ref{ass:bounded} (bounded activations) holds and the LayerNorm
denominator stays uniformly away from zero on the local good region, then for
any layer $\ell$ and gradient step size $\eta$:
\begin{align}
    \|\bh_{k+1}^{(\ell)} - \bh_k^{(\ell)}\|
    \;\leq\; \frac{C_W \eta\,e^{\rho}}{L}\sum_{j=1}^{\ell}\|\nabla_{W^{(j)}}\cL_k^{(j)}\|,
\end{align}
where $C_W = C_{\phi\mathrm{LN}}B_{h^{(\ell-1)}}$, and
$C_{\phi\mathrm{LN}}$ is a uniform local Lipschitz constant of
$z\mapsto\phi(\mathrm{LN}(z))$. 
\end{lemma}

\begin{proof}
Write $\delta\bh^{(\ell)} := \bh_{k+1}^{(\ell)} - \bh_k^{(\ell)}$. Note that the input $\bh_k^{(0)} = \bx$ is fixed across all iteration $k$ so the base case $\delta\bh^{(0)}=0$. We expand the proof in several steps.

\textbf{Step 1 (Per-layer decomposition).} The residual update~\eqref{eq:resnet_arch} gives, for any $\ell\ge 1$,
\begin{align}
    \bh_{k+1}^{(\ell)}
    &= \bh_{k+1}^{(\ell-1)} + \tfrac{1}{L}\,\phi\bigl(\mathrm{LN}(W_{k+1}^{(\ell)}\bh_{k+1}^{(\ell-1)})\bigr),\\
    \bh_{k}^{(\ell)}
    &= \bh_{k}^{(\ell-1)} + \tfrac{1}{L}\,\phi\bigl(\mathrm{LN}(W_{k}^{(\ell)}\bh_{k}^{(\ell-1)})\bigr).
\end{align}
Subtracting and inserting the cross term $\tfrac{1}{L}\phi(\mathrm{LN}(W_{k}^{(\ell)}\bh_{k+1}^{(\ell-1)}))$ yields the additive split
\begin{align}\label{eq:perturb_split}
    \delta\bh^{(\ell)}
    \;=&\; \underbrace{\delta\bh^{(\ell-1)}}_{\text{(skip)}}
    \;+\; \underbrace{\tfrac{1}{L}\bigl[\phi(\mathrm{LN}(W_{k+1}^{(\ell)}\bh_{k+1}^{(\ell-1)})) - \phi(\mathrm{LN}(W_{k}^{(\ell)}\bh_{k+1}^{(\ell-1)}))\bigr]}_{\text{(A) parameter change}}
    \bkeqwn
    \;+\; \underbrace{\tfrac{1}{L}\bigl[\phi(\mathrm{LN}(W_{k}^{(\ell)}\bh_{k+1}^{(\ell-1)})) - \phi(\mathrm{LN}(W_{k}^{(\ell)}\bh_{k}^{(\ell-1)}))\bigr]}_{\text{(B) input change}}.
    \nonumber
\end{align}

\textbf{Step 2 (Bounding the parameter-change term~(A)).}
By Assumption~\ref{ass:bounded}, $\|\bh_{k+1}^{(\ell-1)}\|\le
B_{h^{(\ell-1)}}$. The local LayerNorm regularity gives
\begin{align}
    \|\phi(\mathrm{LN}(W\bh))-\phi(\mathrm{LN}(W'\bh))\|
    \leq C_{\phi\mathrm{LN}}\|(W-W')\bh\|
    \leq C_{\phi\mathrm{LN}}B_{h^{(\ell-1)}}\|W-W'\|_{\mathrm{op}}.
\end{align}
Thus the map $W \mapsto \phi(\mathrm{LN}(W\bh))$ is $C_W$-Lipschitz at
$\bh = \bh_{k+1}^{(\ell-1)}$ in the operator norm, with
$C_W=C_{\phi\mathrm{LN}}B_{h^{(\ell-1)}}$ independent of
$\|W^{(\ell)}\|_{\mathrm{op}}$. Hence,
\begin{align}
    \norm{\text{(A)}} \;\le\; \frac{C_W}{L}\,\norm{W_{k+1}^{(\ell)} - W_{k}^{(\ell)}}_{\mathrm{op}}.
\end{align}
The gradient step $W_{k+1}^{(\ell)} = W_{k}^{(\ell)} - \eta \nabla_{W^{(\ell)}}\cL_k^{(\ell)}$ gives $\|W_{k+1}^{(\ell)} - W_{k}^{(\ell)}\|_{\mathrm{op}} \le \eta\|\nabla_{W^{(\ell)}}\cL_k^{(\ell)}\|$, so
\begin{align}\label{eq:perturb_A}
    \norm{\text{(A)}} \;\le\; \frac{C_W\,\eta}{L}\,\norm{\nabla_{W^{(\ell)}}\cL_k^{(\ell)}}.
\end{align}
The factor $1/L$ comes from the architectural residual scale and is what prevents the per-step weight perturbation from amplifying with depth.

\textbf{Step 3 (Bounding the input-change term~(B)).}
By definition, the un-scaled branch map $\bh \mapsto \phi(\mathrm{LN}(W_k^{(\ell)}\bh))$ is $\rho$-Lipschitz, hence the $1/L$-scaled branch in (B) is $(\rho/L)$-Lipschitz in its input. Therefore,
\begin{align}\label{eq:perturb_B}
    \norm{\text{(B)}} \;\le\; \frac{\rho}{L}\,\norm{\delta\bh^{(\ell-1)}}.
\end{align}

\textbf{Step 4 (Recursive inequality).} Triangle inequality on~\eqref{eq:perturb_split} and substitution of \eqref{eq:perturb_A}--\eqref{eq:perturb_B} give
\begin{align}\label{eq:perturb_recursion}
    \norm{\delta\bh^{(\ell)}}
    \;\le\; \Bigl(1+\tfrac{\rho}{L}\Bigr)\,\norm{\delta\bh^{(\ell-1)}}
    \;+\; \frac{C_W\,\eta}{L}\,\norm{\nabla_{W^{(\ell)}}\cL_k^{(\ell)}}.
\end{align}
The factor $1+\rho/L$ is the Lipschitz constant of one residual block (skip path contributes $1$, branch contributes $\rho/L$). Unrolling \eqref{eq:perturb_recursion} yields that
\begin{align}\label{eq:perturb_unrolled}
    \norm{\delta\bh^{(\ell)}}
    \;\le\; \frac{C_W\,\eta}{L}\,\sum_{j=1}^{\ell}\Bigl(1+\tfrac{\rho}{L}\Bigr)^{\ell-j}\,\norm{\nabla_{W^{(j)}}\cL_k^{(j)}}.
\end{align}

\textbf{Step 5 (Geometric closure).}
Since $0 \le \ell-j \le L$ for all $j\in[\ell]$, the elementary inequality $(1+x/L)^L \le e^x$ for $x\ge 0$ gives
\begin{align}
    \Bigl(1+\tfrac{\rho}{L}\Bigr)^{\ell-j} \;\le\; \Bigl(1+\tfrac{\rho}{L}\Bigr)^{L} \;\le\; e^{\rho}.
\end{align}
Substituting this termwise bound into~\eqref{eq:perturb_unrolled} yields the announced perturbation bound
\begin{align}
    \norm{\bh_{k+1}^{(\ell)} - \bh_k^{(\ell)}} \;\le\; \frac{C_W\,\eta\,e^{\rho}}{L}\sum_{j=1}^{\ell}\norm{\nabla_{W^{(j)}}\cL_k^{(j)}}
\end{align}
which completes the proof.
\end{proof}

\subsection{Lemma~\ref{lem:derived_M_goodness} and its proof (Goodness-Gradient Bound)}
\begin{lemma}[Goodness-gradient bound]\label{lem:derived_M_goodness}
Under Assumptions~\ref{ass:bounded} and~\ref{ass:G_smooth}, define
\begin{align}\label{eq:derived_M_goodness}
M_{\goodness}
:= \|\nabla_{\bh}\goodness(\mathbf{0})\| + L_{\goodness}B_h.
\end{align}
Then the goodness gradient is uniformly bounded over all represented activations:
\begin{align}
\|\nabla_{\bh}\goodness(\bh)\| \leq M_{\goodness}
\qquad\text{for every }\|\bh\|\leq B_h.
\end{align}
\end{lemma}
\begin{proof}
Let $\bh$ be any represented activation. Assumption~\ref{ass:bounded} gives $\|\bh\|\leq B_h$, and the global Lipschitz property in Assumption~\ref{ass:G_smooth} yields
\begin{align*}
\|\nabla_{\bh}\goodness(\bh)\|
&\leq \|\nabla_{\bh}\goodness(\mathbf{0})\|
   + \|\nabla_{\bh}\goodness(\bh)-\nabla_{\bh}\goodness(\mathbf{0})\| \\&\leq \|\nabla_{\bh}\goodness(\mathbf{0})\|
   + L_{\goodness}\|\bh\|
 \leq \|\nabla_{\bh}\goodness(\mathbf{0})\| + L_{\goodness}B_h
 = M_{\goodness}.
\end{align*}
Thus the stated uniform bound holds for every represented activation.
\end{proof}

\section{Full Proofs}
\label{app:proofs}

\subsection{Proof of Theorem~\ref{thm:pl_single} (Local PL for Single-Layer FFA)}
\label{app:proof_pl_single}

\begin{proof}

The $\beta^{(\ell)}$-smoothness of $\cL^{(\ell)}$ is established in Lemma~\ref{lem:smooth}.
The rest of the proof establishes the PL inequality.
For notational simplicity, we omit the variable $W^{(\ell)}$ in the following proof, e.g., writing $\cL^{(\ell)}$ instead of $\cL^{(\ell)}(W^{(\ell)})$ and $H^{(\ell)}$ instead of $H^{(\ell)}(W^{(\ell)})$.
By \Cref{ass:gram}, $\lambda_{\min}(H^{(\ell)}) \geq \mu_H$.
\begin{align}
    \norm{\nabla \cL^{(\ell)}}_F^2
    = \br^\top \Phi^\top \Phi\br
    = \br^\top H^{(\ell)} \br
    \geq \lambda_{\min}(H^{(\ell)})\norm{\br}^2.
\end{align}

We will show $\cL^{(\ell)} - \cL^{(\ell)*} \leq \frac{1}{\gamma^2}\norm{\br}^2$ in two steps. Define a set of notation
\begin{align}
    p_i^- := p^-(\bx_i^+), \quad
    p_i^+ := p^+(\bx_i^+), \quad
    r_i^+ := -p_i^- / n^+, \quad
    r_j^- := p_j^+ / n^-,
\end{align}
With which the loss contribution of positive sample $i$ is $-\frac{1}{n^+}\log(1-p_i^-)$ where $p_i^- = n^+|r_i^+|$ is its misclassification probability; analogously $-\frac{1}{n^-}\log(1-p_j^+)$ with $p_j^+ = n^-|r_j^-|$ for negative sample $j$.

\textbf{Bounded loss on $\cW_R$ implies a uniform saturation gap.}
By Assumption~\ref{ass:bounded} and the derived gradient bound~\ref{lem:derived_M_goodness}, the fundamental theorem of calculus on the segment from $\mathbf{0}$ to $\bh$ gives $|\goodness(\bh)|\leq B_\goodness:=|\goodness(\mathbf{0})|+M_\goodness B_h$ for every $\|\bh\|\leq B_h$. Hence this bound holds for every sample and $W^{(\ell)}\in\cW_R$. Each margin therefore satisfies $|\Delta_i| = |\goodness_i - \theta| \leq B_\goodness + |\theta|$, and the total loss is uniformly bounded:
\begin{align}\label{eq:pl_loss_ub}
    \cL^{(\ell)}(W^{(\ell)})
    \leq L_0
    := \log\!\left(1+e^{B_\goodness+|\theta|}\right)
    \quad\text{for all } W^{(\ell)} \in \cW_R.
\end{align}
Since $\cL^{(\ell)}$ is a sum of non-negative per-sample terms, each individual term is at most $\cL^{(\ell)} \leq L_0$. For a positive sample,
\begin{align}
    \tfrac{1}{n^+}\log(1+e^{-\Delta_i^+})
    \leq L_0
    \;\Longrightarrow\;
    \Delta_i^+ \geq -\log(e^{n^+L_0}-1),
\end{align}
hence $1 - p_i^- = \sigma(\Delta_i^+) \geq e^{-n^+L_0}$, i.e., $n^+|r_i^+| \leq 1 - e^{-n^+L_0}$. Symmetrically, $n^-|r_j^-| \leq 1 - e^{-n^-L_0}$ for every negative sample. Defining the loss-derived saturation gap
\begin{align}\label{eq:pl_gamma_L}
    \gamma_L := e^{-\max(n^+,n^-)L_0} > 0,
\end{align}
we obtain $n^\pm|[\br]_i| \leq 1-\gamma_L$ uniformly on $\cW_R$. After redefining $\gamma \leftarrow \min(\gamma, \gamma_L)$ if necessary, both this upper-saturation bound and the lower bound $|[\br]_i| \geq \gamma$ from condition~(b) hold with the same constant $\gamma$.

\textbf{Composing the two saturation bounds.}
Using the two bounds above, we chain two elementary inequalities:
\begin{enumerate}
    \item \textbf{Upper bound (log step).} For a positive sample $i$, since $p_i^- = n^+|r_i^+| \leq 1-\gamma$,
    \begin{align}
        -\frac{1}{n^+}\log(1-p_i^-)
        \leq \frac{p_i^-}{n^+(1-p_i^-)}
        \leq \frac{p_i^-}{n^+\gamma}
        = \frac{|r_i^+|}{\gamma}.
    \end{align}
    The same argument gives $-\frac{1}{n^-}\log(1-p_j^+) \leq \frac{|r_j^-|}{\gamma}$ for every negative sample $j$.
    \item \textbf{Lower bound (quadratic step).} Since $|r_i^+| \geq \gamma$ and $|r_j^-| \geq \gamma$, we have $\frac{|r_i^+|}{\gamma} \leq \frac{|r_i^+|^2}{\gamma^2}$ and $\frac{|r_j^-|}{\gamma} \leq \frac{|r_j^-|^2}{\gamma^2}$.
\end{enumerate}
Composing the two inequalities,
\begin{align}
    \cL^{(\ell)}
    &= \sum_{i\in[n^+]} -\frac{1}{n^+}\log(1-p_i^-)
      + \sum_{j\in[n^-]} -\frac{1}{n^-}\log(1-p_j^+) \\
    &\leq \frac{1}{\gamma^2}\left(
        \sum_{i\in[n^+]} |r_i^+|^2
        + \sum_{j\in[n^-]} |r_j^-|^2
    \right)
    = \frac{1}{\gamma^2}\norm{\br}^2.
\end{align}
Using $\cL^{(\ell)*} \geq 0$, we conclude
\begin{align}
    \cL^{(\ell)} - \cL^{(\ell)*} \leq \cL^{(\ell)} \leq \frac{1}{\gamma^2}\norm{\br}^2.
\end{align}

Combining, after absorbing the harmless normalization convention of
$H^{(\ell)}$ into $\mu_H$:
\begin{align}
    \norm{\nabla \cL^{(\ell)}}_F^2
    \geq
    \mu_H \gamma^2(\cL^{(\ell)}-\cL^{(\ell)*}),
\end{align}
establishing the PL inequality with $\mu = \mu_H \gamma^2$.

\end{proof}

\subsection{Proof of Theorem~\ref{thm:multi_layer} (Multi-layer FFA Convergence under Scaled Residuals)}
\label{sec:proof-multi-layer}

\begin{proof}
Let the iterate is in the \emph{invariant good region} at time $k$, i.e., $W_k^{(\ell)}\in\cW_R^{(\ell)}, \forall \ell\in[L]$, which means that the following layerwise properties hold for every $\ell\in[L]$:
\begin{enumerate}[label=(GR\arabic*)]
    \item \emph{Local block regularity.} The current block and its current input stay in the local neighborhood used by Theorem~\ref{thm:pl_single}: $W_k^{(\ell)}\in\cW_R^{(\ell)}$, the scaled block has bounded and Lipschitz parameter Jacobian on the segment from $W_k^{(\ell)}$ to $W_{k+1}^{(\ell)}$, and, for ReLU blocks, no training sample crosses a kink, i.e. $D_k^{(\ell)}(\bx)=D_0^{(\ell)}(\bx)$ on the training set.
    \item \emph{Sigmoid margin.} The sigmoid weights in the FFA Gram are uniformly non-degenerate:
    \begin{align}
        \min_i \alpha_{i,k}^{(\ell)} \geq \gamma_{\min},
    \end{align}
    after reducing $\gamma_{\min}$ by a universal constant factor if necessary.
    \item \emph{Gram conditioning.} The actual goodness-gradient Gram matrix remains well-conditioned:
    \begin{align}
        \lambda_{\min}\!\left(H^{(\ell)}(W_k^{(\ell)})\right) \geq \frac{\mu_H}{2}.
    \end{align}
    \item \emph{Uniform smoothness and PL.} For the lower-layer environment $W_{<\ell,k}$, the map
    $W^{(\ell)}\mapsto \cL^{(\ell)}(W^{(\ell)};W_{<\ell,k})$ is $\beta_{\max}$-smooth and satisfies the uniform PL inequality
    \begin{align}
        \left\|\nabla_{W^{(\ell)}}\cL^{(\ell)}(W_k^{(\ell)};W_{<\ell,k})\right\|_F^2
        \geq 2\mu_{\min}\bigl(\cL_k^{(\ell)}-\cL^{(\ell)*}\bigr).
    \end{align}
\end{enumerate}
Property (GR4) is not an additional assumption: by Theorem~\ref{thm:pl_single}, it follows from (GR1)--(GR3) together with the goodness regularity assumptions, with constants absorbed into $\mu_{\min}$ and $\beta_{\max}$. 
The proof is therefore a bootstrap. Conditional on (GR1)--(GR4) holding up to time $k$, Steps~1--3 prove the Lyapunov descent recursion. Step~4 then proves that the update also satisfies (GR1)--(GR4) at time $k+1$, using the scaled-ResNet perturbation bound, dynamic-threshold margin maintenance, and NTK/Gram stability. This closes the induction and avoids any circular use of the PL hypothesis.

\textbf{Step 1: Per-layer one-step decomposition.}
For each layer $\ell$, write
\begin{align}
    \cL_{k+1}^{(\ell)} - \cL_k^{(\ell)} 
    =&\ \underbrace{\cL^{(\ell)}(W_{k+1}^{(\ell)}; W_{<\ell,k}) - \cL^{(\ell)}(W_k^{(\ell)}; W_{<\ell,k})}_{\text{(A): descent}} \\
    &\ + \underbrace{\cL^{(\ell)}(W_{k+1}^{(\ell)}; W_{<\ell,k+1}) - \cL^{(\ell)}(W_{k+1}^{(\ell)}; W_{<\ell,k})}_{\text{(B): perturbation}}.
    \nonumber
\end{align}
The descent lemma plus PL gives $\text{(A)} \le -\eta \mu^{(\ell)}(\cL_k^{(\ell)} - \cL^{(\ell)*})$.

In term (B), the layer-$\ell$ weight is held fixed at $W_{k+1}^{(\ell)}$ and only the input $\bh^{(\ell-1)}$ varies between $W_{<\ell,k}$ and $W_{<\ell,k+1}$. For $\ell=1$ the input is fixed and $\delta^{(0)}=0$, so $(B)=0$. For $\ell\ge 2$, by Assumption~\ref{ass:sensitivity} and Lemma~\ref{lem:perturbation} applied at layer $\ell-1$ (whose drift only depends on lower-layer updates $j\le\ell-1$),
\begin{align}
    &\ |\text{(B)}|
    \;\le\; G_{\max}\,\|\bh_{k+1}^{(\ell-1)} - \bh_k^{(\ell-1)}\|
    \;\le\; \frac{G_{\max}C_W\eta\,e^\rho}{L}\sum_{j=1}^{\ell-1}\|\nabla_{W^{(j)}}\cL_k^{(j)}\| \\
    \le&\ G_{\max}C_W\eta\,e^\rho\sqrt{\frac{1}{L}\sum_{j=1}^{L}\|\nabla_{W^{(j)}}\cL_k^{(j)}\|^2}
    \le G_{\max}C_W\eta\,e^\rho\sqrt{2\beta_{\max}V_k}
    \nonumber\\
    =&\ G_{\max}C_W\eta\,e^\rho\sqrt{2\beta_{\max}}\sqrt{V_k} 
    \nonumber
\end{align}
where the second inequality is Cauchy--Schwarz and the third uses
$\|\nabla\cL_k^{(j)}\|^2\le 2\beta_{\max}
(\cL_k^{(j)}-\cL^{(j)*})$ for lower-bounded smooth losses.
Consequently, 
\begin{align}
    \cL_{k+1}^{(\ell)} - \cL^{(\ell)*} \le (1-\eta \mu^{(\ell)})(\cL_k^{(\ell)} - \cL^{(\ell)*}) + G_{\max}C_W\eta\,e^\rho\sqrt{2\beta_{\max}}\sqrt{V_k} 
\end{align}

\textbf{Step 3: Lyapunov aggregation.}
Write $\alpha := G_{\max}C_W e^\rho$. 
\emph{Young's inequality} ($ab \le a^2/(2\epsilon) + \epsilon b^2/2$ for $\epsilon>0$) with $a = \alpha\sqrt{2\beta_{\max}}$, $b=\sqrt{V_k}$, and $\epsilon = \mu_{\min}$:
\begin{align}
    \alpha\sqrt{2\beta_{\max}V_k}
    \;\le\; \frac{\alpha^2 \beta_{\max}}{\mu_{\min}}
        + \frac{\mu_{\min}}{2}V_k
    \;=\; \frac{\alpha^2\beta_{\max}}{\mu_{\min}} + \frac{\mu_{\min}}{2}V_k.
\end{align}

Substituting back yields the Lyapunov descent recursion
\begin{align}
    V_{k+1} - V_k
    \;\le\; -\frac{\eta\mu_{\min}}{2}V_k + \frac{\eta\alpha^2\beta_{\max}}{\mu_{\min}}.
\end{align}
Substituting $\alpha^2 = G_{\max}^2 C_W^2 e^{2\rho}$ yields
\begin{align}
    V_{k+1} \;\le\; \Bigl(1 - \frac{\eta\mu_{\min}}{2}\Bigr)V_k + \eta \cdot \frac{G_{\max}^2 C_W^2 e^{2\rho} \beta_{\max}}{\mu_{\min}}.
\end{align}
Telescoping the recursion gives
\begin{align}
    V_k \;\le\; 
    \Bigl(1 - \frac{\eta\mu_{\min}}{2}\Bigr)^k V_0 + 
    O\left(\frac{e^{2\rho}\beta_{\max}}{\mu_{\min}^2}\right) 
\end{align}

\textbf{Step 4: Trajectory stays in $\cW_R$.}

With PL constant $\mu$ and $\beta^{(\ell)}$-smoothness, the descent lemma gives:
\begin{align}
    \cL^{(\ell)}(W_{k+1}^{(\ell)}) 
    \leq&\ \cL^{(\ell)}(W_k^{(\ell)}) - \frac{\eta}{2\beta^{(\ell)}}\norm{\nabla \cL^{(\ell)}(W_k^{(\ell)})}_F^2 \\
    \leq&\ \cL^{(\ell)}(W_k^{(\ell)}) - \frac{\eta\mu}{\beta^{(\ell)}}(\cL^{(\ell)}(W_k^{(\ell)}) - \cL^{(\ell)*}).
    \nonumber
\end{align}
Rearranging the inequality yields that
\begin{align}
    \cL^{(\ell)}(W_{k+1}^{(\ell)}) - \cL^{(\ell)*} 
    \leq&\ (1 - \eta\mu/\beta^{(\ell)})(\cL^{(\ell)}(W_k^{(\ell)}) - \cL^{(\ell)*}) \\
    \leq&\ (1 - \eta\mu/\beta^{(\ell)})^{k+1}(\cL^{(\ell)}(W_0^{(\ell)}) - \cL^{(\ell)*})
    \nonumber
\end{align}
By linear convergence, the total parameter movement is bounded:
\begin{align}
    \norm{W_k^{(\ell)} - W_0^{(\ell)}}_F \leq&\ \sum_{s=0}^{k-1} \eta\norm{\nabla \cL^{(\ell)}(W_s^{(\ell)})}_F \leq \frac{\eta}{\beta^{(\ell)}}\sum_{s=0}^{\infty} \sqrt{2\beta^{(\ell)} (\cL^{(\ell)}(W_s^{(\ell)}) - \cL^{(\ell)*})} \\
    \leq&\ \frac{\eta\sqrt{2(\cL_0^{(\ell)} - \cL^{(\ell)*})}}{\sqrt{\beta^{(\ell)}}(1-\sqrt{1-\eta\mu/\beta^{(\ell)}})}
    = \frac{\sqrt{2(\cL_0^{(\ell)} - \cL^{(\ell)*})}}{\sqrt{\beta^{(\ell)}}}\cdot \frac{1+\sqrt{1-\eta\mu/\beta^{(\ell)}}}{\mu/\beta^{(\ell)}} \\
    \leq&\ \frac{2\sqrt{2(\cL_0^{(\ell)} - \cL^{(\ell)*})\beta^{(\ell)}}}{\mu}
    \leq R,
    \nonumber
\end{align}
where we used $1+\sqrt{1-\eta\mu/\beta^{(\ell)}} \leq 2$, and the last inequality holds by taking
\begin{align}
    R = 2\sqrt{2} \sqrt{\frac{(\cL_0^{(\ell)} - \cL^{(\ell)*})\beta^{(\ell)}}{\mu^2}}
    = \Theta(\sqrt{(\cL_0^{(\ell)} - \cL^{(\ell)*}) \beta^{(\ell)} / \mu^2}).
\end{align}

\end{proof}

\subsection{Proof of Theorem~\ref{thm:kernel_contraction} (Kernel Contraction)}
\label{app:kernel_contraction_proof}

Every finite network is the prefix $0\leq \ell\leq L$ of one common infinite i.i.d.\ block sequence. The proof uses Oseledec's multiplicative ergodic theorem~\cite{oseledec1968multiplicative}, applied to the linearized random kernel transformation operator. The argument has three parts: (i) write the kernel transition as a random map, (ii) invoke Oseledec to obtain a deterministic top Lyapunov exponent $\lambda_1$ for the linearized cocycle, and (iii) convert the sign of this exponent into a finite-prefix envelope for $\|\Sigma^{(\ell)}-J_N\|_F$.

\paragraph{Step 1: Random kernel transformation.}
Conditional on $W$, $T_W:\mathcal S\to\mathcal S$ is a measurable random map on the cone of PSD unit-diagonal kernels, where $\mathcal{S} = \{\Sigma \in \R^{N\times N}:\, \Sigma \succeq 0,\, \mathrm{diag}(\Sigma) = \mathbf{1}\}$. Its i.i.d.\ iterates satisfy
\begin{align}\label{eq:T_W_def}
    \Sigma^{(\ell)}
    = T_{W^{(\ell)}}\circ\cdots\circ T_{W^{(1)}}\bigl(\Sigma^{(0)}\bigr).
\end{align}
The all-ones kernel $J_N$ is a fixed point: if all normalized sample representations coincide, applying the same block map to every sample preserves equality. The theorem hypotheses posit this i.i.d. kernel-map recursion; they do not derive a single map shared across different physical depths from the $1/L$ factor in~\eqref{eq:resnet_arch}.

\paragraph{Step 2: Oseledec applied to the linearized cocycle.}
Let $A_W := \mathrm{D}T_W(J_N)$ act on the tangent space of off-diagonal kernel perturbations $\Sigma-J_N$. The random operators $A_{W^{(1)}},A_{W^{(2)}},\ldots$ are i.i.d.
We use the same local bounded-Jacobian regime as in Lemma~\ref{lem:smooth} and Appendix~\ref{app:assumptions}: on a no-kink region for ReLU (or after smoothing ReLU/LN), the map $\Sigma \mapsto T_W(\Sigma)$ is $C^1$ near $J_N$, the LayerNorm denominator stays uniformly away from its singular set, and the linearization norm grows at most polynomially in $\|W\|_{\mathrm{op}}$. Equivalently, there exist deterministic constants $C_{\mathrm{lin}}, q_{\mathrm{lin}} < \infty$ such that
\begin{align}
    \|A_W\|
    \le
    C_{\mathrm{lin}}(1+\|W\|_{\mathrm{op}})^{q_{\mathrm{lin}}}.
\end{align}
Hence
\begin{align}
    \log^+\|A_W\|
    \le
    \log C_{\mathrm{lin}} + q_{\mathrm{lin}}\log(1+\|W\|_{\mathrm{op}})
    \le
    \log C_{\mathrm{lin}} + q_{\mathrm{lin}}\log 2 + q_{\mathrm{lin}}\log^+\|W\|_{\mathrm{op}},
\end{align}
so the theorem assumption $\E_{\mathcal{P}_W}[\log^+\|W\|_{\mathrm{op}}]<\infty$ implies $\E_{\mathcal{P}_W}[\log^+\|A_W\|]<\infty$.

\begin{proposition}[Oseledec's multiplicative ergodic theorem \cite{oseledec1968multiplicative}]
Let $(A_{W^{(m)}})_{m\ge 1}$ be the finite-dimensional i.i.d.\ linear cocycle above, and assume $\E_{\mathcal{P}_W}[\log^+\|A_W\|]<\infty$. Then there exists an almost-sure measurable invariant filtration of the tangent space such that every nonzero vector $v$ in each stratum has a deterministic asymptotic exponential growth rate:
\begin{align}
    \lim_{m\to\infty}\frac{1}{m}\log\bigl\|A_{W^{(m)}}A_{W^{(m-1)}}\cdots A_{W^{(1)}}v\bigr\|
\end{align}
exists almost surely and equals one of the Lyapunov exponents of the cocycle.
\end{proposition}

In particular, the top Lyapunov exponent governs the maximal exponential growth rate of perturbations, so the operator norm satisfies the same top-exponent asymptotics. Thus the following limit exists almost surely:
\begin{align}\label{eq:lyapunov_def}
    \lambda_1 := \lim_{m\to\infty}\frac{1}{m}\log\bigl\|A_{W^{(m)}}A_{W^{(m-1)}}\cdots A_{W^{(1)}}\bigr\|
    \quad\text{a.s.,}
\end{align}
and the limit is deterministic. 
Setting $\bar\rho^*:= e^{\lambda_1}$ proves Theorem~\ref{thm:kernel_contraction}(a). If $\lambda_1<0$, fix $\eta = \frac{|\lambda_1|}{2}$
Then, on the almost-sure event where~\eqref{eq:lyapunov_def} holds, we have
\begin{align}
    &\ \frac{1}{m}\log\!\left(
        e^{|\lambda_1|m/2}
        \bigl\|A_{W^{(m)}}A_{W^{(m-1)}}\cdots A_{W^{(1)}}\bigr\|
    \right) \\
    =&\ \frac{1}{m}\log\bigl\|A_{W^{(m)}}A_{W^{(m-1)}}\cdots A_{W^{(1)}}\bigr\|
    +
    \frac{|\lambda_1|}{2}
    \;\to\;
    -\frac{|\lambda_1|}{2}
    <0.
    \nonumber
\end{align}
Hence
\begin{align}
    e^{|\lambda_1|m/2}
    \bigl\|A_{W^{(m)}}A_{W^{(m-1)}}\cdots A_{W^{(1)}}\bigr\|
    \to 0.
\end{align}
so this sequence is almost surely bounded. Let $\omega$ denote the sample point of the underlying probability space carrying the realized i.i.d.\ block sequence, and define
\begin{align}
    C_{\mathrm{env}}(\omega)
    :=
    1+
    \sup_{m\ge 1}
    e^{|\lambda_1|m/2}
    \bigl\|A_{W^{(m)}}A_{W^{(m-1)}}\cdots A_{W^{(1)}}\bigr\|.
\end{align}
Then $C_{\mathrm{env}}(\omega)<\infty$ almost surely, and therefore
\begin{align}
    \bigl\|A_{W^{(m)}}A_{W^{(m-1)}}\cdots A_{W^{(1)}}\bigr\|
    \le
    C_{\mathrm{env}}(\omega)e^{-|\lambda_1|m/2}
    \qquad\text{for all }m\ge 1.
\end{align}
Therefore the linearized dynamics satisfy
\begin{align}
    \|\Sigma^{(m)}-J_N\|_F
    \le
    C_{\mathrm{env}}(\omega)e^{-|\lambda_1|m/2}\|\Sigma^{(0)}-J_N\|_F.
\end{align}
Since $T_W(J_N)=J_N$ and $T_W$ is twice differentiable in a neighborhood of $J_N$ for smooth LN and fixed activation masks (or by standard ReLU smoothing), the nonlinear remainder is quadratic in $\|\Sigma^{(m)}-J_N\|_F$. Therefore the same envelope holds for the nonlinear kernel dynamics while the finite prefix remains in the linearization neighborhood, yielding Theorem~\ref{thm:kernel_contraction}(b).

\subsection{Proof of Theorem~\ref{thm:bp_esd} (BP Error Signal Diversity)}
\label{sec:proof_bp_esd}

\begin{proof}
Consider the $L$-layer ResNet with per-sample layer Jacobian $J^{(\ell)}(\bx) = I + \frac{1}{L}\widetilde{J}^{(\ell)}(\bx)$, $\|\widetilde{J}^{(\ell)}(\bx)\|_{\mathrm{op}} \leq \bar{\rho}$ uniformly in $\bx$.
Recall from the body that $\boldsymbol{\delta}^{(\ell),\mathrm{BP}}(\bx)=J^{(L:\ell+1)}(\bx)^{\top}\boldsymbol{\delta}^{(L),\mathrm{BP}}(\bx)$ with $\boldsymbol{\delta}^{(L),\mathrm{BP}}(\bx) := \frac{\partial \cL_{\mathrm{task}}}{\partial\bh^{(L)}}|_{\bx}$. 
For simplicity, write $\bbdelta(\bx):= \boldsymbol{\delta}^{(L),\mathrm{BP}}(\bx)$ and
\begin{align}
    s_-^{*} := \sigma_{\min}\bigl(\bar J^{(L:\ell+1)}\bigr), \quad s_+^{*} := \sigma_{\max}\bigl(\bar J^{(L:\ell+1)}\bigr).
\end{align}
With this notation, the BP Gram matrix at layer $\ell$ is
\begin{align}\label{eq:bp_gram_expectation}
    \Gamma^{(\ell),\mathrm{BP}} \;=\; \E_{\bx}\!\bigl[(J^{(L:\ell+1)}(\bx))^{\top}\,\bbdelta(\bx)\bbdelta(\bx)^{\top}\,J^{(L:\ell+1)}(\bx)\bigr].
\end{align}
The Jacobian product $J^{(L:\ell+1)}(\bx)$ depends on $\bx$ through every nonlinearity (ReLU, LayerNorm, attention softmax, \emph{etc.}) along the forward pass.

\textbf{Preliminary: per-sample Jacobian product bounds.}
For any input $\bx$, let $J^{(a:b)}(\bx) = \prod_{j=b}^{a}(I + \frac{1}{L}\widetilde{J}^{(j)}(\bx))$. Its singular values satisfy
\begin{align}
    \left(1 - \frac{\bar{\rho}}{L}\right)^{a-b+1} \leq \sigma_{\min}\bigl(J^{(a:b)}(\bx)\bigr) \leq \sigma_{\max}\bigl(J^{(a:b)}(\bx)\bigr) \leq \left(1 + \frac{\bar{\rho}}{L}\right)^{a-b+1},
\end{align}
uniformly in $\bx$. Therefore, it holds for the full product $J^{(L:\ell+1)}(\bx)$ that $e^{-\bar{\rho}} \leq \sigma_{\min} \leq \sigma_{\max} \leq e^{\bar{\rho}}$. The same singular-value range therefore applies to any deterministic average $\bar J^{(L:\ell+1)}$ taken over the data support, giving $s_-^*, s_+^*\in[e^{-\bar\rho}, e^{\bar\rho}]$ in Assumption~\ref{ass:bp_jac_conc}.

\textbf{Reference / deviation decomposition.}
Under Assumption~\ref{ass:bp_jac_conc}, write
\begin{align}\label{eq:J_decomp}
    J^{(L:\ell+1)}(\bx) \;=\; \bar J^{(L:\ell+1)} \;+\; E(\bx),\qquad \sup_{\bx}\bigl\|E(\bx)\bigr\|_{\mathrm{op}}\;\leq\;\tau_J,
\end{align}
and abbreviate $\bar J:=\bar J^{(L:\ell+1)}$. Substituting~\eqref{eq:J_decomp} into the Gram identity~\eqref{eq:bp_gram_expectation},
\begin{align}\label{eq:gamma_decomp}
    \Gamma^{(\ell),\mathrm{BP}} \;=\; \underbrace{\bar J^{\top}\,\Gamma^{(L)}\,\bar J}_{=:\,\widehat\Gamma} \;+\; \cR,
\end{align}
where the remainder is
\begin{align}
    \cR := \E_{\bx}\!\bigl[\bar J^{\top}\bbdelta(\bx)\bbdelta(\bx)^{\top} E(\bx)\bigr] + \E_{\bx}\!\bigl[E(\bx)^{\top}\bbdelta(\bx)\bbdelta(\bx)^{\top}\bar J\bigr] + \E_{\bx}\!\bigl[E(\bx)^{\top}\bbdelta(\bx)\bbdelta(\bx)^{\top} E(\bx)\bigr].
\end{align}

\textbf{Step 1: bounds on the remainder $\cR$.}
Each summand inside the expectation is a rank-one outer product, so $\|\bu\bv^{\top}\|_F=\|\bu\|\,\|\bv\|$ together with $\|E(\bx)^{\top}\bbdelta(\bx)\|\le\tau_J\|\bbdelta(\bx)\|$ and $\|\bar J^{\top}\bbdelta(\bx)\|\le s_+^{*}\|\bbdelta(\bx)\|$ give
\begin{align}
    \bigl\|\bar J^{\top}\bbdelta(\bx)\bbdelta(\bx)^{\top} E(\bx)\bigr\|_F = \|\bar J^{\top}\bbdelta(\bx)\|\,\|E(\bx)^{\top}\bbdelta(\bx)\| &\leq s_+^{*}\tau_J\,\|\bbdelta(\bx)\|^{2},\\
    \bigl\|E(\bx)^{\top}\bbdelta(\bx)\bbdelta(\bx)^{\top} E(\bx)\bigr\|_F = \|E(\bx)^{\top}\bbdelta(\bx)\|^{2} &\leq \tau_J^{2}\,\|\bbdelta(\bx)\|^{2}.
\end{align}
Jensen's inequality $\|\E[\cdot]\|_F\leq\E\|\cdot\|_F$ and $\E_{\bx}\|\bbdelta(\bx)\|^{2}=\tr(\Gamma^{(L)})$ then yield
\begin{align}\label{eq:R_F_bound}
    \|\cR\|_F \;\leq\; \bigl(2 s_+^{*}\tau_J + \tau_J^{2}\bigr)\,\tr(\Gamma^{(L)}) \;\eqqcolon\; \Delta\,\tr(\Gamma^{(L)}).
\end{align}
For the trace, expand $\tr(\cR)$ summand-wise. The two cross terms are mutual transposes:
\begin{align}
    \bigl(E(\bx)^{\top}\bbdelta(\bx)\bbdelta(\bx)^{\top}\bar J\bigr)^{\top} \;=\; \bar J^{\top}\bbdelta(\bx)\bbdelta(\bx)^{\top}E(\bx),
\end{align}
and trace is invariant under transposition, so $\tr(E^{\top}\bbdelta\bbdelta^{\top}\bar J)=\tr(\bar J^{\top}\bbdelta\bbdelta^{\top}E)$ and the two cross terms add to $2\,\tr\bigl(\bar J^{\top}\bbdelta\bbdelta^{\top}E\bigr)$. By the cyclic property of trace,
\begin{align}
    \tr\bigl(\bar J^{\top}\bbdelta(\bx)\bbdelta(\bx)^{\top}E(\bx)\bigr) \;=\; \tr\bigl(\bbdelta(\bx)^{\top}E(\bx)\bar J^{\top}\bbdelta(\bx)\bigr) \;=\; \bbdelta(\bx)^{\top}E(\bx)\bar J^{\top}\bbdelta(\bx),
\end{align}
where the last equality drops the trace because the argument is a $1\times 1$ scalar. Cauchy--Schwarz then bounds this scalar by
\begin{align}
    \bigl|\bbdelta(\bx)^{\top}E(\bx)\bar J^{\top}\bbdelta(\bx)\bigr|
    \;\leq\; \|E(\bx)^{\top}\bbdelta(\bx)\|\,\|\bar J^{\top}\bbdelta(\bx)\|
    \;\leq\; \tau_J\cdot s_+^{*}\,\|\bbdelta(\bx)\|^{2}.
\end{align}
For the third term, cyclicity gives $\tr(E^{\top}\bbdelta\bbdelta^{\top}E)=\bbdelta^{\top}EE^{\top}\bbdelta=\|E^{\top}\bbdelta\|^{2}\leq\tau_J^{2}\|\bbdelta\|^{2}$. Taking $\E_{\bx}$ and using $\E_{\bx}\|\bbdelta(\bx)\|^{2}=\tr(\Gamma^{(L)})$,
\begin{align}\label{eq:R_tr_bound}
    \bigl|\tr(\cR)\bigr|
    \;\leq\; 2 s_+^{*}\tau_J\,\tr(\Gamma^{(L)}) + \tau_J^{2}\,\tr(\Gamma^{(L)})
    \;=\; \Delta\,\tr(\Gamma^{(L)}).
\end{align}

\textbf{Step 2: bounds on the leading term $\widehat\Gamma=\bar J^{\top}\Gamma^{(L)}\bar J$.}
For PSD $\Gamma^{(L)}$ with eigendecomposition $\Gamma^{(L)}=\sum_i\lambda_i\bv_i\bv_i^{\top}$,
\begin{align}
    \tr(\widehat\Gamma) = \tr(\Gamma^{(L)}\bar J\bar J^{\top}) = \sum_i\lambda_i\|\bar J^{\top}\bv_i\|^{2}\;\in\;\bigl[s_-^{*2}\tr(\Gamma^{(L)}),\,s_+^{*2}\tr(\Gamma^{(L)})\bigr].
\end{align}
For the Frobenius norm, since $\bar J$ is invertible, $\|\bar J^{\top}M\bar J\|_F\leq s_+^{*2}\|M\|_F$ follows from the Cauchy--Schwarz inequality for the Frobenius inner product combined with $\|\bar J Y\bar J^{\top}\|_F\leq s_+^{*2}\|Y\|_F$; the matching lower bound $\|\widehat\Gamma\|_F\geq s_-^{*2}\|\Gamma^{(L)}\|_F$ follows by applying the same inequality to $\Gamma^{(L)}=\bar J^{-\top}\widehat\Gamma\bar J^{-1}$ with $\sigma_{\max}(\bar J^{-1})=1/s_-^{*}$. Hence
\begin{align}
    s_-^{*4}\,\tr({\Gamma^{(L)}}^{2}) \;\leq\; \|\widehat\Gamma\|_F^{2} \;\leq\; s_+^{*4}\,\tr({\Gamma^{(L)}}^{2}).
\end{align}

\textbf{Step 3: Part (a) -- approximate rank preservation.}
Combining~\eqref{eq:gamma_decomp},~\eqref{eq:R_F_bound}, and Weyl's inequality on the spectra of $\widehat\Gamma$ and $\widehat\Gamma+\cR$,
\begin{align}\label{eq:bp_weyl}
    \lambda_i\bigl(\Gamma^{(\ell),\mathrm{BP}}\bigr) \;\geq\; \lambda_i(\widehat\Gamma) - \|\cR\|_{\mathrm{op}} \;\geq\; s_-^{*2}\,\lambda_i(\Gamma^{(L)}) - \Delta\,\tr(\Gamma^{(L)}),
\end{align}
where the second step uses $\lambda_i(\widehat\Gamma)\geq s_-^{*2}\lambda_i(\Gamma^{(L)})$ (which follows component-wise from the SVD of $\bar J$) and $\|\cR\|_{\mathrm{op}}\leq\|\cR\|_F$. Therefore
\begin{align}\label{eq:bp_rank_final}
    \rank\bigl(\Gamma^{(\ell),\mathrm{BP}}\bigr) \;\geq\; \#\!\Bigl\{\,i\;:\;\lambda_i(\Gamma^{(L)})\;>\;\tfrac{\Delta}{s_-^{*2}}\tr(\Gamma^{(L)})\,\Bigr\}.
\end{align}
In the limit $\tau_J\to 0$ the right-hand side equals $\rank(\Gamma^{(L)})$, recovering the data-independent congruence rank-equality.

\textbf{Step 4: Part (b) -- effective rank preservation.}
Combining~\eqref{eq:gamma_decomp} with the trace bound on $\widehat\Gamma$ and~\eqref{eq:R_tr_bound},
\begin{align}
    \tr\bigl(\Gamma^{(\ell),\mathrm{BP}}\bigr) \;\geq\; (s_-^{*2}-\Delta)\,\tr(\Gamma^{(L)}).
\end{align}
For the Frobenius norm, the triangle inequality with~\eqref{eq:R_F_bound} gives
\begin{align}
    \bigl\|\Gamma^{(\ell),\mathrm{BP}}\bigr\|_F \;\leq\; s_+^{*2}\,\|\Gamma^{(L)}\|_F + \Delta\,\tr(\Gamma^{(L)}).
\end{align}
Using the identity $\tr(\Gamma^{(L)})/\|\Gamma^{(L)}\|_F=\sqrt{\erank(\Gamma^{(L)})}$ and assuming $\Delta<s_-^{*2}$,
\begin{align}\label{eq:bp_erank_final}
    \erank\bigl(\Gamma^{(\ell),\mathrm{BP}}\bigr)
        = \frac{\tr(\Gamma^{(\ell),\mathrm{BP}})^{2}}{\|\Gamma^{(\ell),\mathrm{BP}}\|_F^{2}}
        \;\geq\; \frac{(s_-^{*2}-\Delta)^{2}}{\bigl(s_+^{*2}+\Delta\sqrt{\erank(\Gamma^{(L)})}\bigr)^{2}}\,\erank(\Gamma^{(L)}).
\end{align}

\textbf{Step 5: asymptotic form recovering $e^{-8\bar\rho}$.}
By definition $\Delta=2 s_+^{*}\tau_J+\tau_J^{2}=O(\tau_J)$. Whenever $\tau_J\sqrt{\erank(\Gamma^{(L)})}\ll s_+^{*2}$, the prefactor in~\eqref{eq:bp_erank_final} simplifies to
\begin{align}
    \frac{(s_-^{*2}-\Delta)^{2}}{\bigl(s_+^{*2}+\Delta\sqrt{\erank(\Gamma^{(L)})}\bigr)^{2}}
    \;=\; \frac{s_-^{*4}}{s_+^{*4}}\,\Bigl(1-O\bigl(\tau_J\sqrt{\erank(\Gamma^{(L)})}/s_+^{*2}\bigr)\Bigr)
    \;\geq\; e^{-8\bar\rho}\,\bigl(1 - o(1)\bigr).
\end{align}
This recovers the data-independent bound $\erank(\Gamma^{(\ell),\mathrm{BP}})\geq e^{-8\bar\rho}\erank(\Gamma^{(L)})$ stated in Theorem~\ref{thm:bp_esd}(b), valid under Assumption~\ref{ass:bp_jac_conc} with $\tau_J\sqrt{\erank(\Gamma^{(L)})}=o(s_+^{*2})$.
\end{proof}

\subsection{Proof of Theorem~\ref{thm:elc_scaling} (Scaling via Effective Learning Capacity)}
\label{app:proof-elc_scaling}

\begin{proof}
\textbf{Part (a):} Direct summation of Theorem~\ref{thm:bp_esd}(b):
\begin{align}
\ELC^{\mathrm{BP}}
&= \sum_{\ell=1}^L \bigl(\erank(\Gamma^{(\ell),\mathrm{BP}})-1\bigr)
\geq L\max\bigl\{0,\,\exp(-8\bar{\rho})\,\erank(\Gamma^{(L)})-1\bigr\},
\end{align}
where we use both Theorem~\ref{thm:bp_esd}(b) and the universal bound $\erank(\Gamma^{(\ell),\mathrm{BP}})\geq 1$.

For language models: the output error is $\boldsymbol{\delta}^{(L)} = W_{\mathrm{out}}^T(\hat{p} - e_Y) \in \R^d$ where $W_{\mathrm{out}} \in \R^{V \times d}$. The ESK is $\Gamma^{(L)} = W_{\mathrm{out}}^T \E[(\hat{p}-e_Y)(\hat{p}-e_Y)^T] W_{\mathrm{out}}$. Since $V \gg d$ and $W_{\mathrm{out}}$ has rank~$d$, we have $\rank(\Gamma^{(L)}) = d$. Under random or trained initialization, $\Gamma^{(L)}$ is non-degenerate with $\erank(\Gamma^{(L)}) = \Theta(d)$. Hence $\ELC({\mathrm{BP}}) = \Theta(Ld)$.

\textbf{Part (b):} Differentiating~\eqref{eq:ffa_loss} with respect to $\bh^{(\ell)}(\bx)$ gives the per-sample FFA descent direction
\begin{align}
    \boldsymbol{\delta}^{(\ell),\mathrm{FFA}}(\bx)
    =
    \begin{cases}
        \bigl(1-\sigma(\goodness^{(\ell)}(\bh^{(\ell)}(\bx))-\theta)\bigr)\,\nabla_{\bh}\goodness^{(\ell)}(\bh^{(\ell)}(\bx)),
        & \bx \sim \cD^+,\\[4pt]
        -\,\sigma(\goodness^{(\ell)}(\bh^{(\ell)}(\bx))-\theta)\,\nabla_{\bh}\goodness^{(\ell)}(\bh^{(\ell)}(\bx)),
        & \bx \sim \cD^-.
    \end{cases}
\end{align}
Hence $\Gamma^{(\ell),\mathrm{FFA}}$ is the second moment of weighted goodness gradients. By Theorem~\ref{thm:kernel_contraction}, in a collapsing regime with $\lambda_1<0$ and along a prefix that remains in the linearization neighborhood, the normalized representations satisfy $\tilde{\bh}_\ell(\bx) = \sqrt{d}(\bar{\bv} + \boldsymbol{\epsilon}(\bx))$ with constant $\bar{\bv}$ and $\|\boldsymbol{\epsilon}(\bx)\| = O((\sqrt{\bar\rho^*})^\ell)$. Assumption~\ref{ass:G_smooth} then gives
\begin{align}
    \nabla_{\bh}\goodness^{(\ell)}(\bh^{(\ell)}(\bx))
    =
    \nabla_{\bh}\goodness^{(\ell)}(\sqrt{d}\,\bar{\bv}) + O((\sqrt{\bar\rho^*})^\ell),
\end{align}
uniformly over samples. Therefore, the deep-layer ESK is controlled by the span of the collapsed goodness-gradient family $\{\nabla_{\bh}\goodness_x^{(\ell)}(\sqrt{d}\,\bar{\bv})\}_x$. In the rank-one subclass used in the theorem statement---for example $\goodness(\bh)=\|\bh\|^2$, for which $\nabla_{\bh}\goodness=2\bh$, or any fixed quadratic goodness with $\nabla_{\bh}\goodness(\bh)=B\bh$ for a fixed matrix $B$---these collapsed gradients are all collinear. Hence, there exists a unit vector $\bar{\bu}_\ell$, a scalar coefficient $c_\ell$ bounded above and below by positive constants, and a symmetric remainder $R_\ell$ with $\|R_\ell\|_F = O((\sqrt{\bar\rho^*})^\ell)$ such that
\begin{align}
    \Gamma^{(\ell),\mathrm{FFA}}
    =
    c_\ell\,\bar{\bu}_\ell\bar{\bu}_\ell^\top + R_\ell.
\end{align}
Using $\tr(\bar{\bu}_\ell\bar{\bu}_\ell^\top)=\tr((\bar{\bu}_\ell\bar{\bu}_\ell^\top)^2)=1$ and the bounds
\begin{align}
    |\tr(R_\ell)| \leq \sqrt{d}\,\|R_\ell\|_F = O((\sqrt{\bar\rho^*})^\ell),
    \qquad
    |\tr(\bar{\bu}_\ell\bar{\bu}_\ell^\top R_\ell)| \leq \|R_\ell\|_{\mathrm{op}} \leq \|R_\ell\|_F = O((\sqrt{\bar\rho^*})^\ell),
\end{align}
we obtain
\begin{align}
    \tr(\Gamma^{(\ell),\mathrm{FFA}})
    &= c_\ell + \tr(R_\ell)
    = c_\ell + O((\sqrt{\bar\rho^*})^\ell),\\
    \tr((\Gamma^{(\ell),\mathrm{FFA}})^2)
    &= c_\ell^2 + 2c_\ell\,\tr(\bar{\bu}_\ell\bar{\bu}_\ell^\top R_\ell) + \tr(R_\ell^2)
    = c_\ell^2 + O((\sqrt{\bar\rho^*})^\ell),
\end{align}
where $\tr(R_\ell^2) \leq \|R_\ell\|_F^2 = O((\sqrt{\bar\rho^*})^{2\ell}) = O((\bar\rho^*)^\ell)$. Therefore, by the definition of effective rank,
\begin{align}
    \erank(\Gamma^{(\ell),\mathrm{FFA}})
    &=
    \frac{(\tr(\Gamma^{(\ell),\mathrm{FFA}}))^2}{\tr((\Gamma^{(\ell),\mathrm{FFA}})^2)} 
    =
    \frac{(c_\ell + O((\sqrt{\bar\rho^*})^\ell))^2}{c_\ell^2 + O((\sqrt{\bar\rho^*})^\ell)}
    =
    1 + O((\sqrt{\bar\rho^*})^\ell),
    \nonumber
\end{align}
where the last step uses the uniform positive lower bound on $c_\ell$. Subtracting the rank-one floor and summing gives
\begin{align}
    \ELC({\mathrm{FFA}})
    &= \sum_{\ell=1}^L \bigl(\erank(\Gamma^{(\ell),\mathrm{FFA}})-1\bigr)
    = O\lp\sum_{\ell=1}^L (\sqrt{\bar\rho^*})^\ell\rp
    \leq O\lp\frac{\sqrt{\bar\rho^*}}{1-\sqrt{\bar\rho^*}}\rp,
\end{align}
which completes the proof of Theorem~\ref{thm:elc_scaling}(b).
\end{proof}

%% file: appendix/C_assumption_compat.tex
\section{Verification of Assumptions}
\label{app:assumptions}

\subsection{Verification of Assumption~\ref{ass:gram}}
\label{app:verification-gram}
The following proposition provides an example that the FFA Gram matrix satisfies Assumption~\ref{ass:gram} with high probability. 
\begin{proposition}[Gram Conditioning at Initialization]
Suppose \Cref{ass:bounded} holds. Let $K_{\mathrm{in}}^{(\ell-1)}\in\R^{N\times N}$ be the previous-layer representation Gram matrix,
\begin{align}
    [K_{\mathrm{in}}^{(\ell-1)}]_{ij}
    :=
    \left\langle \bh_i^{(\ell-1)},\bh_j^{(\ell-1)}\right\rangle .
\end{align}
Assume the previous-layer representations satisfy the quadratic-feature diversity condition
\begin{align}
    \lambda_{\min}\!\left(
        \frac{8}{d^{(\ell-1)}}
        \bigl(K_{\mathrm{in}}^{(\ell-1)}\odot K_{\mathrm{in}}^{(\ell-1)}\bigr)
    \right)
    \geq 2\mu_H
    \quad\text{for some }\mu_H>0,
    \label{eq:quadratic_feature_diversity}
\end{align}
where $\odot$ denotes the entrywise product. This condition holds with high probability, for example, when the previous-layer representations $\bh_i^{(\ell-1)}$ are independent isotropic sub-Gaussian vectors with non-degenerate norms and $d^{(\ell-1)}\gtrsim N+\log(1/\delta)$. Then, under Gaussian initialization $W_0^{(\ell)}\sim\cN(0,2/d^{(\ell-1)})$ and sufficiently large output width $d^{(\ell)}$, 
\begin{align}
    \lambda_{\min}\!\left(\frac{1}{d^{(\ell)}}H^{(\ell)}(W_0^{(\ell)})\right)\geq \mu_H
\end{align}
with high probability. 
\end{proposition}

\begin{proof}
Let $H_0^{(\ell)}:= H^{(\ell)}(W_0^{(\ell)})$ and let $\Phi_0$ be the matrix whose $i$-th column is the flattened goodness-gradient direction
\begin{align}
    \boldsymbol{\psi}_{i,0}^{(\ell)}
    :=
    \mathrm{vec}\!\left(
        \nabla_{W^{(\ell)}}\goodness\!\left(\bh_i^{(\ell)}\right)
    \right).
\end{align}
By the definition of the FFA Gram matrix in Section~\ref{sec:single_layer},
\begin{align}
    H_{0,ij}^{(\ell)}
    =
    \left\langle
        \boldsymbol{\psi}_{i,0}^{(\ell)},
        \boldsymbol{\psi}_{j,0}^{(\ell)}
    \right\rangle,
    \qquad
    H_0^{(\ell)}=\Phi_0^\top\Phi_0 .
    \label{eq:init_gram_column_form}
\end{align}
Thus the relevant matrix is an $N\times N$ column Gram matrix, not the parameter-space covariance $\sum_i \boldsymbol{\psi}_{i,0}^{(\ell)}(\boldsymbol{\psi}_{i,0}^{(\ell)})^\top$.

For the quadratic goodness $\goodness(\bh)=\|\bh\|^2$ and a linear or fixed-mask local block, the chain rule gives
\begin{align}
    \nabla_{W^{(\ell)}}\goodness(\bh_i^{(\ell)})
    =
    2\,\bh_i^{(\ell)}\bigl(\bh_i^{(\ell-1)}\bigr)^\top,
    \qquad
    \boldsymbol{\psi}_{i,0}^{(\ell)}
    =
    2\,\mathrm{vec}\!\left(
        \bh_i^{(\ell)}\bigl(\bh_i^{(\ell-1)}\bigr)^\top
    \right),
    \label{eq:quadratic_goodness_vec_gradient}
\end{align}
which is the Kronecker product $2\,\bh_i^{(\ell-1)}\otimes\bh_i^{(\ell)}$ under the column-major vectorization convention.
Now we show that the projection of random vector $\boldsymbol{\psi}_{i,0}^{(\ell)}$ on any unit vector $u=\mathrm{vec}(U)$ is sub-Gaussian. By the definition of the Kronecker product and the inner product,
\begin{align}
    \left\langle u,\boldsymbol{\psi}_{i,0}^{(\ell)}\right\rangle
    =
    2\,\left\langle U,\bh_i^{(\ell)}\bigl(\bh_i^{(\ell-1)}\bigr)^\top\right\rangle
    =
    2\,\left\langle U\bh_i^{(\ell-1)},\bh_i^{(\ell)}\right\rangle .
    \label{eq:psi_subgaussian_scalar}
\end{align}
The bounded-activation assumption controls the deterministic vector multiplying the Gaussian output. Indeed, since $u=\mathrm{vec}(U)$, $\|U\|_F=\|u\|$, and hence
\begin{align}
    \|U\bh_i^{(\ell-1)}\|
    \leq
    \|U\|_{\mathrm{op}}\|\bh_i^{(\ell-1)}\|
    \leq
    \|U\|_F B_h
    =
    B_h\|u\|.
    \label{eq:bounded_activation_projection}
\end{align}
By Cauchy--Schwarz's inequality, it holds that
\begin{align}
    \left\|\left\langle u,\boldsymbol{\psi}_{i,0}^{(\ell)}\right\rangle\right\|_{\psi_2}
    \leq
    C B_h^2\|u\|.
    \label{eq:psi_subgaussian_bound}
\end{align}
Implying $\left\langle u,\boldsymbol{\psi}_{i,0}^{(\ell)}\right\rangle$ is sub-Gaussian with parameter $O(B_h^2)$.

It remains to explain the matrix concentration step. Expose the Gaussian block row by row. In the quadratic fixed-mask specialization, \eqref{eq:quadratic_goodness_vec_gradient} implies
\begin{align}
    H_0^{(\ell)}
    =
    \sum_{r=1}^{d^{(\ell)}} Y_r,
    \qquad
    (Y_r)_{ij}
    =
    4\,
    \left\langle \bh_i^{(\ell-1)},\bh_j^{(\ell-1)}\right\rangle
    h_{i,r}^{(\ell)}h_{j,r}^{(\ell)},
    \label{eq:row_feature_decomposition}
\end{align}
where the matrices $Y_r$ are independent, self-adjoint, and positive semidefinite conditional on the previous-layer representations. 
Moreover, since $W_{0,r\cdot}^{(\ell)}\sim \cN(0,2/d^{(\ell-1)})$,
\begin{align}
    \E_{W_0}\!\left[h_{i,r}^{(\ell)}h_{j,r}^{(\ell)}\mid \{x_i\}_{i=1}^N\right]
    =
    \frac{2}{d^{(\ell-1)}}\left\langle x_i,x_j\right\rangle.
\end{align}
Therefore
\begin{align}
    \frac{1}{d^{(\ell)}}\E_{W_0}H_0^{(\ell)}
    =
    \frac{8}{d^{(\ell-1)}}
    \bigl(K_{\mathrm{in}}^{(\ell-1)}\odot K_{\mathrm{in}}^{(\ell-1)}\bigr),
    \label{eq:expected_gram_from_input_diversity}
\end{align}
and \eqref{eq:quadratic_feature_diversity} is exactly a previous-layer representation condition implying
\begin{align}
    \lambda_{\min}\!\left(
        \frac{1}{d^{(\ell)}}\E_{W_0}H_0^{(\ell)}
    \right)
    \geq 2\mu_H .
\end{align}

Apply the self-adjoint matrix Bernstein inequality~\cite[Theorem~1.6]{tropp2012user} to
\begin{align}
    X_r
    :=
    \frac{1}{d^{(\ell)}}\left(Y_r-\E Y_r\right),
    \qquad
    \frac{1}{d^{(\ell)}}\left(H_0^{(\ell)}-\E H_0^{(\ell)}\right)
    =
    \sum_{r=1}^{d^{(\ell)}}X_r .
\end{align}
The definition of $X_r$ is chosen so that it is exactly the $r$-th centered contribution to the normalized Gram matrix. If
\begin{align}
    A:= \frac{1}{d^{(\ell)}}\E H_0^{(\ell)},
    \qquad
    S:= \sum_{r=1}^{d^{(\ell)}}X_r
    =
    \frac{1}{d^{(\ell)}}\left(H_0^{(\ell)}-\E H_0^{(\ell)}\right),
\end{align}
then
\begin{align}
    \frac{1}{d^{(\ell)}}H_0^{(\ell)}=A+S .
\end{align}
Weyl's inequality gives
\begin{align}
    \lambda_{\min}(A+S)
    \geq
    \lambda_{\min}(A)-\|S\|_{\mathrm{op}}.
\end{align}
Therefore, the event that the normalized empirical Gram loses at least $t$ in its minimum eigenvalue is contained in the event that the centered fluctuation $S=\sum_r X_r$ has operator norm at least $t$.
On the bounded-activation/local-region event, there are constants $R_\ell$ and $\sigma_\ell^2$, controlled by $B_h$ and the operator norm of the previous-layer sample Gram, such that
\begin{align}
    \|X_r\|_{\mathrm{op}}\leq \frac{R_\ell}{d^{(\ell)}},
    \qquad
    \left\|\sum_{r=1}^{d^{(\ell)}}\E X_r^2\right\|_{\mathrm{op}}
    \leq
    \frac{\sigma_\ell^2}{d^{(\ell)}}.
\end{align}
For normalized sample Gram matrices these constants are $R_\ell=O(B_h^4)$ and $\sigma_\ell^2=O(B_h^8)$. Bernstein's inequality therefore gives, for every $t>0$,
\begin{align}
    \label{eq:gram_bernstein_tail}
    \Prob\!\left[
        \lambda_{\min}\!\left(\frac{1}{d^{(\ell)}}H_0^{(\ell)}\right)
        \leq
        \lambda_{\min}\!\left(\frac{1}{d^{(\ell)}}\E H_0^{(\ell)}\right)-t
    \right]
    &\leq
    \Prob\!\left[
        \left\|\frac{1}{d^{(\ell)}}\left(H_0^{(\ell)}-\E H_0^{(\ell)}\right)\right\|_{\mathrm{op}}
        \geq t
    \right] \\
    &\qquad\leq
    2N\exp\!\left(
        -\frac{d^{(\ell)} t^2/2}{\sigma_\ell^2+R_\ell t/3}
    \right). \nonumber
\end{align}
The factor $N$ is the dimension of the Gram matrix in \eqref{eq:init_gram_column_form}. Taking $t=\mu_H$ and using the input diversity condition through \eqref{eq:expected_gram_from_input_diversity} yields
\begin{align}
    \Prob\!\left[
        \lambda_{\min}\!\left(\frac{1}{d^{(\ell)}}H_0^{(\ell)}\right)
        \leq
        \mu_H
    \right]
    \leq
    2N\exp\!\left(
        -\frac{d^{(\ell)}\mu_H^2/2}{\sigma_\ell^2+R_\ell\mu_H/3}
    \right).
\end{align}
Thus $d^{(\ell)}\gtrsim (\sigma_\ell^2/\mu_H^2+R_\ell/\mu_H)\log(2N/\delta)$ gives failure probability at most $\delta$. The numerical choice of $\mu_H$ is therefore controlled by the quadratic-feature diversity margin in \eqref{eq:quadratic_feature_diversity}, rather than being an independent population-Gram assumption.
\end{proof}

\subsection{Verification of \Cref{ass:bp_jac_conc}}
\label{app:verification-bp_jac_conc}

\begin{proposition}[NTK realization of Assumption~\ref{ass:bp_jac_conc}]\label{prop:bp_jac_conc_ntk}
Consider the equal-width scaled ResNet of~\eqref{eq:resnet_arch} (so $d^{(k)}\equiv d^{(\ell-1)}$ for every $k\in[\ell+1,L]$), under the additional regularity conditions
\begin{enumerate}[label=\textup{(N\arabic*)}, leftmargin=3em, itemsep=-0.3em, topsep=0em]
    \item\label{ass:smooth_phi} \textbf{Smooth activation.} $\phi$ is twice differentiable with $\|\phi'\|_\infty\le C_\phi$ and $\|\phi''\|_\infty\le L_\phi$.
    \item\label{ass:ntk_init} \textbf{NTK initialization and bounded inputs.} $[W_0^{(k)}]_{ij}\stackrel{\mathrm{iid}}{\sim}\cN\!\bigl(0,\,2/d^{(k-1)}\bigr)$, and $\|\bx\|\le B_x$ for all input $\bx$.
    \item\label{ass:ntk_radius} \textbf{NTK neighborhood.} $\bW\in\cW_R$, i.e.\ $\sum_{k}\|W^{(k)}-W_0^{(k)}\|_F\le R$.
\end{enumerate}
Pick any anchor $\bx_0\in\supp\cD$ and set $\bar J^{(L:\ell+1)}:=J^{(L:\ell+1)}|_{\bx_0}(\bW_0)$. Then there exists $C=C(C_\phi,L_\phi,\bar\rho,B_x)$ such that, for every $\delta\in(0,1)$, with probability $\ge 1-\delta$ over the random initialization $\bW_0$,
\begin{align}\label{eq:tau_J_NTK}
    \sup_{\bx\in\supp\cD,\;\bW\in\cW_R}\bigl\|J^{(L:\ell+1)}|_{\bx}(\bW)-\bar J^{(L:\ell+1)}\bigr\|_{\mathrm{op}}\;\le\;C\,\frac{R+\sqrt{\log(N/\delta)}}{\sqrt{d^{(\ell-1)}}}.
\end{align}
Equivalently, Assumption~\ref{ass:bp_jac_conc} holds with $\tau_J=O\bigl(R/\sqrt{d^{(\ell-1)}}\bigr)$ for any fixed $R$ as $d^{(\ell-1)}\to\infty$.
\end{proposition}
\begin{proof}[Proof]
Write $J:=J^{(L:\ell+1)}$. Triangle-decompose the deviation:
\begin{align}\label{eq:tau_J_decomp}
    \bigl\|J|_{\bx}(\bW)-J|_{\bx_0}(\bW_0)\bigr\|_{\mathrm{op}}
    \;\le\;\underbrace{\bigl\|J|_{\bx}(\bW)-J|_{\bx}(\bW_0)\bigr\|_{\mathrm{op}}}_{\text{(I): parameter perturbation}}
    \;+\;\underbrace{\bigl\|J|_{\bx}(\bW_0)-J|_{\bx_0}(\bW_0)\bigr\|_{\mathrm{op}}}_{\text{(II): sample variation at init}}.
\end{align}

\textbf{Term (I).} The standard per-layer ResNet-NTK linearization estimate~\cite{allenzhu2019convergence,du2019gradient}, under~\ref{ass:smooth_phi}--\ref{ass:ntk_init} and an $\epsilon$-net of $\supp\cD$, yields
\begin{align}
    \sup_{\bx\in\supp\cD}\bigl\|\widetilde J^{(k)}|_{\bx}(\bW)-\widetilde J^{(k)}|_{\bx}(\bW_0)\bigr\|_{\mathrm{op}}\;\le\;C_1 L_\phi B_x\,\frac{R}{\sqrt{d^{(k-1)}}}
\end{align}
uniformly in $k\in[\ell+1,L]$ with probability $\ge 1-\delta/2$. The residual scaling $\frac{1}{L}\widetilde J^{(k)}$ in~\eqref{eq:resnet_arch} cancels the layer count when propagating through the product:
\begin{align}
    \mathrm{(I)} \;\le\;e^{2\bar\rho}\cdot\frac{1}{L}\sum_{k=\ell+1}^{L} \bigl\|\widetilde J^{(k)}|_{\bx}(\bW)-\widetilde J^{(k)}|_{\bx}(\bW_0)\bigr\|_{\mathrm{op}}\;\le\;\frac{C_2 R}{\sqrt{d^{(\ell-1)}}},
\end{align}
where $C_2$ absorbs $C_1, L_\phi, B_x$ and the operator-norm factor $e^{2\bar\rho}$.

\textbf{Term (II).} At fixed $\bW_0$, the per-layer residual Jacobian admits the chain-rule form $\widetilde J^{(k)}|_{\bx}(\bW_0)=\diag\!\bigl(\phi'(W_0^{(k)}\,\mathrm{LN}(\bh^{(k-1)}|_{\bx}))\bigr)\,W_0^{(k)}\,\partial\,\mathrm{LN}\bigl(\bh^{(k-1)}|_{\bx}\bigr)$. The LN normalization places $\mathrm{LN}(\bh^{(k-1)}|_{\bx})$ on a fixed-radius sphere, so by the rotational symmetry of $W_0^{(k)}$ the centred difference
\begin{align}
    \widetilde J^{(k)}|_{\bx}(\bW_0)-\widetilde J^{(k)}|_{\bx_0}(\bW_0)
\end{align}
has zero $\bW_0$-mean, and each entry is a finite-degree polynomial in the Gaussian $W_0^{(k)}$ with bounded coefficients (by~\ref{ass:smooth_phi}), giving sub-exponential entry variance $O(1/d^{(k-1)})$. Matrix Bernstein and a union bound over $\supp\cD\times[L]$ yield
\begin{align}
    \sup_{\bx\in\supp\cD,\,k}\bigl\|\widetilde J^{(k)}|_{\bx}(\bW_0)-\widetilde J^{(k)}|_{\bx_0}(\bW_0)\bigr\|_{\mathrm{op}}\;\le\;C_3\,\frac{\sqrt{\log(N L/\delta)}}{\sqrt{d^{(k-1)}}}
\end{align}
with probability $\ge 1-\delta/2$. Propagation through the product as in (I) gives $\mathrm{(II)}\le C_4\sqrt{\log(N/\delta)}/\sqrt{d^{(\ell-1)}}$.

Combining (I) and (II) by a union bound gives~\eqref{eq:tau_J_NTK}.
\end{proof}

\begin{remark}[ReLU activation]\label{rem:relu_jac_conc}
For ReLU activation $\phi(z)=\max(z,0)$, $\phi''$ is supported on a Lebesgue-null set and~\ref{ass:smooth_phi} fails, so~\eqref{eq:tau_J_NTK} does not follow directly from the linearization argument above. Two routes are available: (i) replace ReLU by a smooth surrogate (e.g.\ a softplus $\phi_\beta(z)=\beta^{-1}\log(1+e^{\beta z})$) in the analysis and pass the bound to $\beta\to\infty$; (ii) work directly with the random ReLU activation pattern matrix and apply Du et al.~\cite{du2019gradient}, yielding $\tau_J=O\bigl(R/\sqrt{d^{(\ell-1)}}\bigr)$ at the cost of an additional union bound over the activation pattern. Either route preserves the $\sqrt{d^{(\ell-1)}}$ rate.
\end{remark}

%% file: appendix/D_extended_experiments.tex
\section{Extended Experimental Details and Results}
\label{app:extended_experiments}

This appendix presents additional reproduction details and experimental results that support the theoretical analysis in the main text.

\subsection{Compute resources and execution time}
\label{app:compute_resources}

\paragraph{Hardware.}
All experiments are conducted on NVIDIA RTX PRO 6000 GPUs with 96~GB RAM. 

\paragraph{Memory requirements.}
For the ResNet experiments on CIFAR-10/100, memory footprint ranges from 8--16~GB per GPU depending on depth and batch size. The Tiny-ImageNet runs at $64 \times 64$ resolution require 16--24~GB. Language model experiments scale as follows: tiny and small scales require 6--8~GB per GPU; medium and large scales require 12--18~GB; xlarge scales with per-layer untied heads (LCE mode) peak at 40--44~GB.

\paragraph{Training time.}
Toy problem (Figure~\ref{fig:exp2}, vanilla FFA, $T=200{,}000$ iterations): $\sim$5 minutes per depth setting.
ResNet experiments (CIFAR-10, 50 epochs, single GPU): $\sim$4--8 hours per method per depth, depending on ResNet size (R18: $\sim$4 hours; R108: $\sim$20 hours).
CIFAR-100 experiments (50 epochs per method per depth): $\sim$6--12 hours per configuration.
Tiny-ImageNet experiments at $64 \times 64$ (25 epochs): $\sim$8--15 hours per method per depth.
Language model pre-training (Table~\ref{tab:chinchilla}, Chinchilla-style budgets): tiny scale ($0.14$B tokens) requires $\sim$1 hour on 1 GPU; small scale ($0.32$B tokens) requires $\sim$2.5 hours on 1 GPU; medium and large scales ($0.60$B and $1.02$B tokens) require $\sim$4--6 hours on 2 GPUs; xlarge scale ($2.48$B tokens) with LCE requires $\sim$12--16 hours on 4 GPUs. BP and FFA/NCE runs on the same scales are 2.5--3.0$\times$s.

\subsection{Signal Effective Rank Trajectories}
\label{app:matched_rank_trajectories}

This subsection provides an empirical verification of representation collapse, to be specific, the complete finite-depth rank diagnostic on CIFAR-10 trajectories.
The aim is to distinguish forward representation diversity, measured by $\erank(\Sigma)$, from learning-signal diversity, measured by $\erank(\Gamma)$, while avoiding comparisons across independently trained checkpoints.
The error-signal diagnostic empirically tests the finite-depth signature underlying Theorem~\ref{thm:elc_scaling}: BP retains richer error directions than FFA, whose local signals can approach the rank-one floor after representation collapse.

\begin{figure*}[t]
\centering
\includegraphics[width=0.95\textwidth]{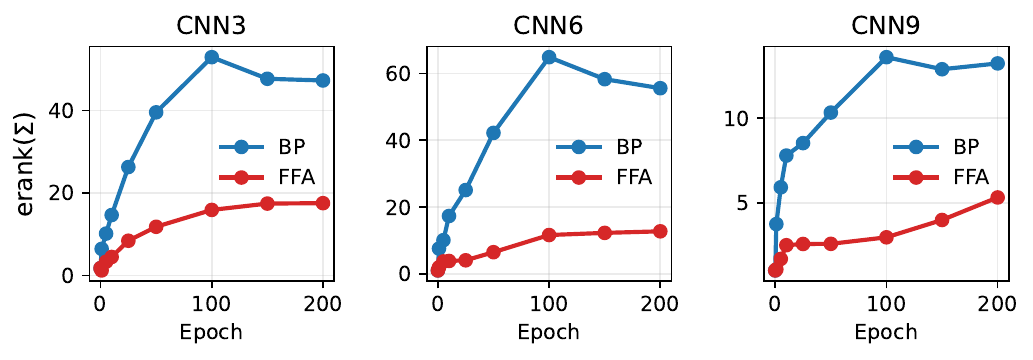}\\[-4pt]
\includegraphics[width=0.95\textwidth]{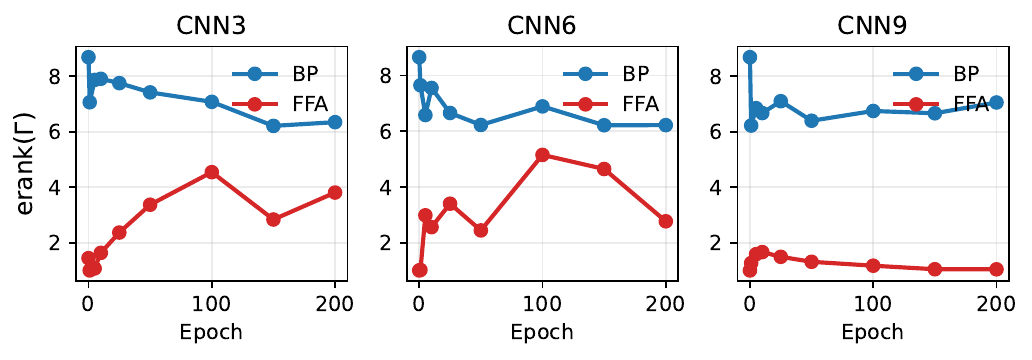}
\caption{Training-time last-block rank diagnostics on CNN@CIFAR-10. 
Top: representation effective rank $\erank(\Sigma)$. 
Bottom: error-signal effective rank $\erank(\Gamma)$.
The $\Sigma$ probe contains one row-normalized representation per image.}
\label{fig:app_matched_rank_trajectories}
\end{figure*}

\paragraph{{Protocol.}}
We train CNN3, CNN6, and CNN9 for 200 epochs with seed 0, Adam step size $10^{-3}$, batch size 128, no data augmentation, and a ten-fold learning-rate reduction after epoch 100. BP uses the task cross-entropy and BatchNorm backbone; FFA uses independent per-block optimizers, detached inter-block activations, and its local positive/negative goodness loss.

At each probe, the final convolutional block is evaluated in evaluation mode on the same class-balanced test subset of 100 images from each of the ten CIFAR-10 classes, hence 1000 raw input images.
For the forward diagnostic, let $d$ denote the dimension of the flattened final-block representation and let $Z\in\mathbb{R}^{1000\times d}$ collect the corresponding image representations row-wise. Each row of $Z$ is normalized to norm $\sqrt{d}$ before forming $\Sigma=ZZ^{\top}/d$.

\paragraph{{Results.}}
At every measured epoch and for all three depths, FFA has a lower last-block $\erank(\Gamma)$ than BP.
Its error-signal rank remains near the one-direction floor after early training, most sharply for CNN9, whereas BP remains between approximately six and nine effective directions. 
Thus the lower FFA error-signal diversity occurs alongside lower non-centered representation diversity.
The paired curves are empirical evidence compatible with the finite-depth mechanism behind Theorem~\ref{thm:elc_scaling}: FFA sustains fewer distinct local error directions, and its forward representations are less diverse on this probe.

\subsection{FFA for classification}
\label{app:ffa_classification}

\paragraph{Why FFA classification is non-trivial.}
The Forward--Forward Algorithm~\cite{hinton2022forward} and most of its descendants do not produce class logits directly: each layer optimizes a local goodness against a positive/negative contrast, and no end-to-end gradient reaches a final softmax classifier head.
This raises a question for benchmarking: how do we report classification accuracy in a form comparable to BP, which terminates in a softmax cross-entropy classifier?
Two protocols dominate the FFA literature.

\textbf{(1) Goodness-based label enumeration}~\cite{hinton2022forward}.
At inference one constructs $K$ copies of the input, one per candidate label, runs each through the network with that label injected, and predicts the label whose summed per-block goodness is largest.
This is the algorithm-internal classifier for label-injection FFA variants (Vanilla FFA, SCFF, SymBa, LayerCollab, Trifecta), and we report its accuracy as \emph{native} classification accuracy where applicable.
It has two drawbacks for cross-method comparison: it is $K\times$ more expensive than a single forward pass, and it is undefined for non-label-injection local-learning methods such as LCE, SFF, and Distance-Forward, which instead ship their own per-block supervised heads.

\textbf{(2) Linear probe / readout head}~\cite{hinton2022forward}.
A linear classifier is trained on the network's hidden features---either on a frozen trunk after training (a \emph{post-hoc linear probe}) or jointly with the native objective using stop-gradient (a \emph{jointly-trained detached readout}).
This protocol applies uniformly to BP, label-injection FFA, and per-block supervised local methods, providing a single comparable axis on which to read out trunk representation quality.

\paragraph{Detached-readout protocol used throughout this appendix.}
Where stated, our \emph{detached readout} attaches a linear classifier head jointly to the network's stage outputs and trains it alongside the native objective at every epoch.
The stage outputs are reduced by global average pooling, concatenated into a single feature vector, passed through optional dropout, and projected to class logits by one linear layer.
The readout's gradients are stopped before reaching the backbone (\verb|h.detach()| on each stage output), so the native trunk sees only its own algorithm's learning rule; the readout simply reports how class-discriminative that trunk is at each epoch.
This differs from a post-hoc linear probe in two practical ways: the readout is co-trained throughout the run rather than fitted on a frozen network, and it sees a multi-stage concatenation rather than a single feature.
Both choices give the readout the strongest reasonable chance of recovering class-discriminative information from a layer-locally trained trunk, putting BP and FFA-family methods on a single comparable axis: how class-readable is the representation the trunk produces.
We refer to this protocol as \emph{detached readout} in subsequent subsections, and to the algorithm-internal classifier as \emph{native} accuracy.

\subsection{Controlled Locality Intervention on CNN12}
\label{app:cnn12_locality_intervention}

This subsection gives the exact protocol underlying Table~\ref{tab:locality_intervention}. It is a grouped-goodness-FFA intervention on CIFAR-10: all grouped-FFA conditions use the same CNN12 convolutional backbone, true-label positive overlay, random-wrong-label negative overlay, local softplus goodness loss, Adam optimizer, learning-rate schedule, batch size, and epoch budget. The parameter varied by design is the partition through which the local goodness gradient is permitted to travel.

\paragraph{Fixed CNN12 backbone and inference rule.}
Write the $\ell$-th backbone map as $f^{(\ell)}$ and its activation as $\bh^{(\ell)}$, so that
\begin{equation}
  \bh^{(\ell)} = f^{(\ell)}\bigl(\bh^{(\ell-1)}\bigr),
  \qquad \ell\in[12].
  \label{eq:cnn12_layer_map}
\end{equation}
Each map is a $3\mathbin{\times}3$ convolution followed by ReLU and, where prescribed, pooling; the channel sequence is
$3\to32\to32\to64\to64\to128\to128\to256\to256\to256\to256\to256\to256$.
Max pooling follows layers $2$, $4$, and $6$, and adaptive average pooling to $4\mathbin{\times}4$ follows layer $12$. The grouped-goodness runs use neither BatchNorm nor channel LayerNorm in this trunk.

Let $m_{\mathrm{blk}}\in\{1,2,3,4,6,12\}$ denote the number of convolutional layers per block and $B_{\mathrm{blk}}=12/m_{\mathrm{blk}}$ the number of blocks. For $b\in[B_{\mathrm{blk}}]$, define the contiguous backbone submap
\begin{equation}
  F^{(b)}
  := f^{(b m_{\mathrm{blk}})}\circ f^{(b m_{\mathrm{blk}}-1)}\circ\cdots\circ
  f^{((b-1)m_{\mathrm{blk}}+1)}.
  \label{eq:cnn12_block_map}
\end{equation}
At inference, the block-boundary activations satisfy $\widehat{\bh}^{[0]}=\bh^{(0)}$ and
\begin{equation}
  \widehat{\bh}^{[b]} = F^{(b)}\bigl(\widehat{\bh}^{[b-1]}\bigr),
  \qquad b\in[B_{\mathrm{blk}}],
  \label{eq:cnn12_inference_blocks}
\end{equation}
so $\widehat{\bh}^{[B_{\mathrm{blk}}]}=f^{(12)}\circ\cdots\circ f^{(1)}(\bh^{(0)})=\bh^{(12)}$ for every partition. Classification is the native label-overlay rule: all ten candidate labels are overlaid on the input, each candidate is passed through the same 12-layer backbone, the goodness scores of the block heads are summed, and the label with the largest sum is returned. Thus, the deployed backbone has identical depth, widths, and forward convolutional maps in every grouped-FFA condition.

\paragraph{Blockwise FFA update.}
For a positive/negative overlay pair, block $b$ receives the stopped boundary activation and computes
\begin{equation}
  \widetilde{\bh}^{[b]}
  = F^{(b)}\bigl(\operatorname{sg}\left[\widehat{\bh}^{[b-1]}\right]\bigr),
  \label{eq:cnn12_stop_gradient}
\end{equation}
where $\operatorname{sg}$ is the stop-gradient operator. Its local head evaluates the positive and negative goodnesses on $\widetilde{\bh}^{[b]}$, and the parameters of $F^{(b)}$ and of that head are updated using the resulting block loss $\cL_{\mathrm{FFA}}^{(b)}$. Hence gradients traverse the $m_{\mathrm{blk}}$ layers within $F^{(b)}$, but do not cross a block boundary. The $12\mathbin{\times}1$ condition has horizon one, whereas $1\mathbin{\times}12$ permits one goodness signal to assign credit through all 12 convolutional layers.

\paragraph{Training, evaluation, and interpretation.}
Each grouped-FFA condition is trained for 200 epochs with batch size 128 and one Adam optimizer per block at learning rate $10^{-3}$; gradients are clipped to norm one and the learning rate is multiplied by $0.1$ after epoch 100.
The common CIFAR-10 preprocessing uses the standard dataset normalization and no data augmentation.
At the final epoch, mean layerwise $\erank(\Gamma)$ is computed from 1,000 class-balanced test images and positive/negative goodness signals.
The BP reference instead uses end-to-end task cross-entropy and is reported only as a distinct reference.

The intervention fixes the backbone shape and the local-goodness recipe while changing the training-time grouping of credit assignment, i.e., locality.
The results therefore support the causal conclusion that relaxing strict layer locality within this fixed grouped-FFA design can mitigate low-diversity learning signals and recover substantial accuracy.

Finally, we should note that changing the block partition lengthens the gradient horizon, but it also changes the number of block-specific goodness heads and optimizers. We therefore interpret this experiment as evidence that relaxing strict layerwise locality through grouped FFA can improve performance, rather than as an estimate of the isolated effect of gradient horizon alone.

\subsection{Toy Multi-Layer Convergence Setup}
\label{app:exp2_toy_setup}

Figure~\ref{fig:exp2} is generated by running vanilla FFA $T=200{,}000$ iterations.
The toy problem uses two Gaussian classes in input dimension $d_0=50$ with sample size $N=100$: positive samples are drawn as $x^+ = \mu + \xi$ and negative samples as $x^- = -\mu + \xi$, where $\mu = (1,0,\ldots,0)$ and $\xi \sim \cN(0, I_{d_0})$.
For each depth $L \in \{2,4,8,16\}$, the network width is fixed at $d=500$, each block is a bias-free linear layer followed by LayerNorm (without affine parameters) and ReLU, and weights are initialized with Kaiming normal initialization.
Training uses the FFA loss with goodness $G(\bh)=\|\bh\|_2^2$, per-layer threshold $\theta^{(\ell)} = d$, and decreasing step size $\eta_t = 5/(50+t)$.

\subsection{Architecture generalization: FFA vs.\ BP on CNN and Vision Transformer.}
To assess whether FFA's theoretical properties extend beyond fully-connected networks, we compare BP and Vanilla FFA on two modern architectures: a 3-layer CNN (32/64/128 filters) and a 2-layer Vision Transformer (4-head attention, 64-dim embeddings), both on full CIFAR-10 (50,000 samples, 10 classes, 50 epochs). FFA uses the label-overlay protocol~\cite{hinton2022forward} with standard goodness $G(\bh) = \|\bh\|^2$, dynamic threshold $\theta$, and gradient clipping.

\begin{table}[ht!]
\centering
\caption{Test accuracy on modern architectures.}
\label{tab:cnn_vit}
\small
\begin{tabular}{l ccc}
\toprule
Architecture & BP & Vanilla FFA & Gap \\
\midrule
CNN (3 conv layers) & 76.03\% & 73.73\% & \textbf{2.30\%} \\
ViT (2 transformer layers) & 72.71\% & 64.77\% & 7.94\% \\
\bottomrule
\end{tabular}
\end{table}

Table~\ref{tab:cnn_vit} shows that on CNN, FFA closes the gap to within 2.3\% of BP. On the ViT training, which is more complex, the gap between BP and FFA achieves increases significantly to 7.94pp. 
The larger ViT gap may reflect the difficulty of propagating label-overlay signal through attention mechanisms that globally mix spatial positions.

\subsection{Unified CIFAR-10 Benchmark Across Local Learning Algorithms}
\label{app:cifar10_cnn_bench}

We conduct a horizontal CIFAR-10 benchmark on three shared CNN backbones (Table~\ref{tab:cifar10_archs}).  All benchmark entries share an \emph{outer training budget}: Adam with learning rate $10^{-3}$, batch size 128, 200 epochs, no augmentation, and a $\times0.1$ learning-rate decay at epoch 100.  This is not a claim of identical computation or a common normalization intervention: algorithms retain their native trunk/head normalization and readout protocols.  In particular, BP and several native CNN methods use BatchNorm, FFA-style heads use LayerNorm, and fixed-feature or probe-based methods have method-specific evaluation stages.

\begin{table}[ht!]
\centering
\caption{CNN backbone architectures used in the unified benchmark.}
\label{tab:cifar10_archs}
\small
\begin{tabular}{lcccc}
\toprule
Arch & Conv blocks & Channel growth & Spatial scales & Params (BP) \\
\midrule
CNN3 & 3 & $3{\to}32{\to}64{\to}128$ & $32^2{\to}16^2{\to}8^2$ & 0.62M \\
CNN6 & 6 & $3{\to}32{\to}32{\to}64{\to}64{\to}128{\to}128$ & $32^2{\to}16^2{\to}8^2$ & 0.81M \\
CNN9 & 9 & $32{\to}64{\to}128{\to}256$ (VGG-style) & $32^2{\to}16^2{\to}8^2{\to}4^2$ & 2.81M \\
\bottomrule
\end{tabular}
\end{table}

Table~\ref{tab:cifar10_cnn_bench} reports one row per algorithm and one column per backbone depth, so that depth trajectories are read along horizontal lines. 
All benchmark rows share the outer budget: Adam with learning rate $10^{-3}$, batch size 128, 200 epochs, no augmentation, and a $0.1\times$ learning-rate decay at epoch 100. Algorithms retain their native normalization and readout protocols. 
Detailed algorithmic definitions and locality discussion appear in Section~\ref{app:baseline_taxonomy}.
The \textbf{Category} column groups algorithms by their local-learning mechanism:
\begin{itemize}[leftmargin=1.8em,itemsep=2pt,topsep=4pt]
  \item \textbf{Baseline.} Standard end-to-end backpropagation.
\end{itemize}

\begin{table*}[t]
  \centering
  \small
  \caption{CIFAR-10 CNN benchmark across local learning algorithms. Entries are best test accuracy (\%, mean $\pm$ std over seeds 0/1/2).}
  \label{tab:cifar10_cnn_bench}
  \begin{threeparttable}
    \input{appendix/CIFAR10-CNN_bench}
    \begin{tablenotes}[flushleft]
      \footnotesize
      \item N/A denotes an algorithm whose native design degenerates. PC $K$ is the number of hidden-state prediction/inference updates per outer parameter update; $K=0$ is not reported because it removes the PC inference loop.
    \end{tablenotes}
  \end{threeparttable}
\end{table*}

\smallskip\noindent\textit{Gradient-isolated local updates (per-block or per-module):} In FFA and its variants, no loss from a deeper local unit differentiates an earlier unit.
\begin{itemize}[leftmargin=1.8em,itemsep=2pt,topsep=2pt]
  \item \textbf{FFA.} Hinton's original Forward-Forward Algorithm~\cite{hinton2022forward}: two-pass (positive/negative) sigmoid contrastive goodness objective applied independently at each layer; no gradient crosses block boundaries.
  \item \textbf{FFA-contrastive.} FFA variants such as SymBa~\cite{lee2023symba} and SCFF~\cite{chen2024selfcontrastive} that alter the contrastive loss surface (symmetric margin loss, self-contrastive formulation) while retaining the positive/negative two-pass structure; all members still rely on sampled negatives.
  \item \textbf{FFA-extended.} FFA variants such as Trifecta~\cite{dooms2023trifecta}, Scodellaro CNN-FFA~\cite{scodellaro2025cnn}, and Distance-Forward~\cite{wu2024distance} that augment the base algorithm with richer per-layer auxiliary machinery---threshold schedules, label-embedding heads, distance-to-prototype objectives---whose native recipes are tied to specific backbone designs and are stripped here by the fair-comparison contract.
  \item \textbf{Local-CE.} Methods such as SFF~\cite{krutsylo2025scalable}, Greedy Layerwise~\cite{belilovsky2019greedy}, and N{\o}kland $L_{\text{pred}}$~\cite{nokland2019local} that replace the contrastive FFA objective with a per-block full-softmax cross-entropy classification head; no random negatives are drawn.
  \item \textbf{Local-sim.} Per-block similarity regression as in N{\o}kland $L_{\text{sim}}$~\cite{nokland2019local} toward the class-match indicator $\mathbf{1}[y_i{=}y_j]$, with no supervised classification head; isolates whether similarity matching alone suffices, independent of depth.
  \item \textbf{Local-combined.} Combines a per-block CE head ($L_{\text{pred}}$) and similarity regression ($L_{\text{sim}}$) within each layer (N{\o}kland $L_{\text{pred}+\text{sim}}$~\cite{nokland2019local}).
  \item \textbf{InfoNCE-local.} Greedy InfoMax~\cite{lowe2019greedy} trains gradient-isolated module groups with a local InfoNCE objective that predicts spatially shifted patch representations. Its locality unit is a module, rather than an individual convolution.
  \item \textbf{Dendritic-local.} Dendritic Localized Learning~\cite{lv2025dendritic} uses detached block activations and separate block/head optimizers; its next-state prediction target is detached.
\end{itemize}

\smallskip\noindent\textit{Auxiliary-network supervision:} 
Each local learning signal is constructed through an auxiliary network, either private to a block or used to propagate targets through learned inverses.
\begin{itemize}[leftmargin=1.8em,itemsep=2pt,topsep=2pt]
  \item \textbf{Aux-net.} Each layer is supervised via an auxiliary network whose depth follows a schedule $\{L{-}\ell{-}1\}_\ell$ that grows with backbone depth, explicitly designed to scale its capacity with the number of layers, as in AugLocal~\cite{ma2024auglocal}.
  \item \textbf{Target-Prop.} Difference Target Propagation~\cite{lee2015difference} trains each layer to invert the next-layer mapping via a learned reconstruction network. Its parameter update is local, but its target is recursively derived from the top classifier error through these inverse maps.
\end{itemize}

\smallskip\noindent\textit{Cross-layer coupling:} A local parameter update may consume detached signals produced at another layer, but no end-to-end gradient is propagated.
\begin{itemize}[leftmargin=1.8em,itemsep=2pt,topsep=2pt]
  \item \textbf{Dual-path feedback.} Counter-Current Learning~\cite{kao2024countercurrent} uses a dual path that carries label-driven feedback across layers while retaining parameter-local updates. 
  It is parameter-local but label-feedback-coupled. Its row uses only the repaired dual-path runs; the historic single-path fallback runs remain excluded.
  \item \textbf{Multi-layer-coord.} Methods such as Layer Collaboration~\cite{lorberbom2024layer} explicitly couple adjacent layers through an accumulated collaboration offset $\gamma_{<t}$, providing inter-layer coordination without end-to-end gradients.
  \item \textbf{Random-features.} Closed-form ridge regression on fixed random convolutional projections as in Forward Projection~\cite{oshea2025forwardprojection}; no iterative gradient descent is performed at any layer.
\end{itemize}

\smallskip\noindent\textit{Energy-based model training:} The inference stage can be formulated as an optimization problem that minimizes an energy function of the hidden states.
An inner loop of $K$ iterations is adopted to iteratively refine the hidden states before the outer parameter update.
\begin{itemize}[leftmargin=1.8em,itemsep=2pt,topsep=2pt]
  \item \textbf{Predictive Coding.} Fixed-prediction PC~\cite{rao1999predictive,whittington2017approximation,millidge2022predictive} alternates local state inference and parameter updates; $K\in\{1,5,20\}$ is the number of synchronous hidden-state inference steps performed before each outer update.
\end{itemize}

\paragraph{Scope and N/A entries.}
\emph{Greedy InfoMax} \cite{lowe2019greedy} is originally deployed with $M{=}3$ \emph{gradient-isolated modules}, each of which is itself a \emph{ResNet-50 residual block group of tens of convolutional layers} \cite[Sec.~4.1, App.~A.1]{lowe2019greedy}, and its InfoNCE objective predicts $k$ future rows of a $7{\times}7$ patch grid on $64{\times}64$ inputs. Under the 3-layer contract the combined receptive field saturates the $32{\times}32$ image after the second maxpool, collapsing the spatial patch hierarchy that InfoNCE relies on; we therefore mark Greedy InfoMax N/A on CNN3. Under the 6-layer contract the network preserves three spatial scales ($32^2 \to 16^2 \to 8^2$), which suffices to support a two-module InfoNCE variant that predicts shifted $4{\times}4$ patches at the $h_3$ ($16{\times}16$) and $h_5$ ($8{\times}8$) feature maps; this is the CNN6 entry for GIM. 

\emph{Layer Collaboration} \cite{lorberbom2024layer} is ported from its original MLP design by computing each block's goodness on the globally-average-pooled activation vector $v_\ell = \mathrm{GAP}(h_\ell)$ and accumulating the collaboration offset $\gamma_{<t} = \sum_{t'<t} G_{t'}$ as a detached scalar; each (block, FFA-head) pair has its own Adam optimizer, so no gradient crosses block boundaries.
\emph{AugLocal}~\cite{ma2024auglocal} is tagged ``degenerate'' on CNN3 because its linearly decreasing auxiliary-network depth schedule collapses to $\{2,1,0\}$ at $L=3$; on CNN6 the schedule is $\{5,4,3,2,1,0\}$, the paper's intended regime, so the tag is cleared.
\emph{Counter-Current Learning}~\cite{kao2024countercurrent} is now reported only from the repaired dual-path implementation: all three seeds completed for CNN3, CNN6, and CNN9 under the shared outer budget. The historic \texttt{single\_path} fallback remains excluded because it is end-to-end CE training rather than Counter-Current Learning.
\emph{Forward Projection}~\cite{oshea2025forwardprojection} is tagged ``closed\_form'' because the algorithm has no iterative training stage; the reported number is the one-shot ridge regression solution on random features.

\subsection{Task-Difficulty Scaling on Deep ResNets: BP vs.\ Local Learning}
\label{app:resnet_readout_difficulty}

The preceding paragraphs and Section~\ref{app:cifar10_cnn_bench} probe the BP--local-learning gap along two axes: \emph{algorithm family} at fixed task (CIFAR-10, CNN3/6/9), and \emph{architecture} at fixed dataset (MLP, CNN, ViT on CIFAR-10).
This subsection adds a third axis---\emph{task difficulty}---by sweeping four datasets of increasing complexity (MNIST, CIFAR-10, CIFAR-100, Tiny-ImageNet) on a deep ResNet ladder (ResNet-\{18,24,56,108\}).
The goal is to test whether the BP advantage observed at moderate depth is task-dependent and how depth and the training recipe redistribute the gap as task complexity grows.

\subsubsection{Setup and protocol}
\label{ssec:resnet_setup_protocol}

\paragraph{Detached-readout protocol.}
We use the detached-readout protocol described in Section~\ref{app:ffa_classification}, configured for the ResNet backbone as follows.
The readout concatenates all five stage outputs (the stem and the four residual stages: $h^{(0)},\ldots,h^{(4)}$) after global average pooling into a $1024$-dimensional feature vector ($64{+}64{+}128{+}256{+}512$), applies optional dropout, and projects to class logits by a single linear layer.
The readout's gradients are stopped before reaching the backbone, so the native trunk sees only its own algorithm's learning rule.
We report both the detached-readout test accuracy and the algorithm-native classification accuracy, where the latter is well-defined; readout-versus-native differences are analyzed in Section~\ref{ssec:resnet_native_vs_readout}.
The reported metric is the best detached-readout test accuracy across the $200$ training epochs; native (algorithm-internal) test accuracy is reported alongside it where relevant.

\paragraph{Per-method training recipe.}
We adopt a \emph{per-method} training protocol rather than imposing a uniform optimizer.
For BP, we train the ResNet backbone introduced by He et al.~\cite{he2016deep} using a conventional CIFAR-style recipe: SGD with Nesterov momentum~$0.9$, weight decay~$5{\times}10^{-4}$, initial learning rate~$0.1$, and a multistep schedule decaying by $\times 0.1$ at epochs $\{100,150\}$.
For each layer-local method (LCE, SFF, Distance-Forward (DF), and the label-injection FFA family) we instantiate one independent Adam optimizer per residual block at $\text{lr}{=}10^{-3}$ and apply a single $\times 0.1$ decay at epoch $100$, following the optimizer choice in the original FFA, LCE, SFF, and DF works.
We use ``LCE'' for the per-block local cross-entropy method of N{\o}kland \& Eidnes~\cite{nokland2019local}; this is the same family as the ``Local-CE'' methods of Section~\ref{app:cifar10_cnn_bench} and the per-layer cross-entropy heads in the language-modeling experiments of Section~\ref{sec:experiments}, the unifying idea being that the FFA contrastive loss is replaced by a per-block full-softmax cross-entropy supervision rather than tied to negative sampling.
We do \emph{not} claim these recipes are globally optimal; they are widely reported defaults under which each method has been evaluated, and our goal is to compare each method against a sensible default from its own literature.
Imposing a single optimizer across paradigms would systematically handicap whichever paradigm differs from the imposed default---SGD penalizes layer-local Adam users, Adam penalizes end-to-end ResNet---which we view as a less informative comparison than per-method canonical defaults.
The MNIST results (Section~\ref{ssec:resnet_main_results}), where the four BP-competitive methods cluster within $0.06$~pp at every depth, provide a calibration that the protocol is not so biased as to produce systematic gaps at task saturation; this rules out catastrophic recipe bias but does not exclude smaller hard-task-only biases at the magnitudes observed on CIFAR-100 / Tiny-ImageNet.
The detached readout itself uses AdamW($\text{lr}{=}10^{-3}$) with weight decay $10^{-4}$ in the baseline recipe and $10^{-3}$ in the hardened recipe (with readout dropout $0.2$ and label smoothing $0.1$ added in the latter).

\paragraph{Architectures.}
The four ResNets share a single CIFAR-style topology: a $3{\times}3$ stride-$1$ stem (no max-pool), four residual stages with channel widths $\{64,128,256,512\}$ and downsampling strides $\{1,2,2,2\}$, and BasicBlock units (two $3{\times}3$ convolutions with batch normalization and ReLU, projection shortcut when channels or stride change).
The variants differ only in the number of BasicBlocks per stage: ResNet-18 uses $[2,2,2,2]$ (the standard CIFAR adaptation of \texttt{torchvision.models.resnet18}), while ResNet-$\{24,56,108\}$ use $[3,3,3,2]$, $[7,7,7,6]$, $[13,13,13,14]$ respectively. Writing $n_s$ for the number of BasicBlocks in stage $s\in\{1,2,3,4\}$, the total convolutional + linear depth is $L=2\sum_s n_s+2 \in \{18,24,56,108\}$ ($+1$ for the stem convolution and $+1$ for the final linear classifier).
Label-injection FFA variants additionally take a $4$-channel input by concatenating a learned $1$-channel label embedding spatially aligned with the input (Trifecta-style label injection~\cite{hinton2022forward}); the concatenated input passes through the same backbone topology.

\paragraph{Datasets and recipes.}
The four datasets cover a coarse difficulty spectrum: MNIST ($10$ classes, grayscale upsampled to $32{\times}32$ RGB), CIFAR-10 ($10$ classes, $32{\times}32$), CIFAR-100 ($100$ classes, $32{\times}32$), and Tiny-ImageNet ($200$ classes, native $64{\times}64$).
The \emph{baseline} training recipe uses each dataset's standard mean/std normalization, $\text{RandomCrop}(32, \text{padding}{=}4) + \text{HorizontalFlip}$ on CIFAR variants and a $\text{RandomResizedCrop}$+flip on Tiny-ImageNet (downsampled to $32{\times}32$ in the baseline pipeline so that all four datasets share a single ResNet input size).
On the two harder datasets we additionally report a \emph{hardened} recipe that simultaneously (i) applies \texttt{RandAugment}($n{=}2,m{=}9$) and \texttt{RandomErasing}($p{=}0.25$), (ii) raises BP weight decay to $10^{-3}$, label smoothing to $0.1$, readout dropout to $0.2$, readout weight decay to $10^{-3}$, (iii) swaps BP's multistep schedule for cosine annealing, and---on Tiny-ImageNet only---(iv) restores the native $64{\times}64$ input resolution that was downsampled in the baseline pipeline.
Item (iv) is a recipe-level choice rather than a regularisation manipulation; we flag it explicitly because it contributes to the larger absolute lift observed on Tiny-ImageNet relative to CIFAR-100.
Layer-local methods retain their per-block step decay under the hardened recipe (their schedule is not naturally swappable to cosine without also re-tuning each block's optimizer); they do receive the augmentation, label-smoothing, weight-decay, and readout-dropout changes shared with BP.
This is one of several per-method asymmetries inherent in the canonical-recipe protocol; we discuss its implications under the gain decomposition (Section~\ref{ssec:resnet_hardening_gain}).

\paragraph{Algorithm coverage and run accounting.}
We train nine native algorithms split into two families.
\textbf{BP-competitive locals}---LCE~\cite{nokland2019local}, SFF~\cite{krutsylo2025scalable}, and DF---use per-block supervised heads (full-softmax CE for LCE and SFF, distance-to-prototype for DF) and run on all four datasets and both recipes.
\textbf{Label-injection FFA family}---Vanilla FFA~\cite{hinton2022forward}, SCFF~\cite{chen2024selfcontrastive}, SymBa, LayerCollab~\cite{lorberbom2024layer}, and Trifecta---inject the label spatially and contrast positive against negative goodness; we report these only on MNIST and CIFAR-10, where they train stably across all four ResNet depths under their published recipes.
$72$ MNIST/CIFAR-10 runs ($9{\times}2{\times}4$), $16$~CIFAR-100 baseline runs ($4{\times}1{\times}4$), $16$~CIFAR-100 hardened runs, $16$~Tiny-ImageNet baseline runs, $16$~Tiny-ImageNet hardened runs; $136$ total, all finished with W\&B status \texttt{ok}.
All runs use $200$ epochs, batch size $128$, and a single seed ($s{=}0$); compute budget did not permit additional seeds.
We treat any per-cell difference smaller than ${\sim}2$~pp as not statistically separable from single-seed noise throughout this section, and we report a final-vs-best-epoch robustness check at the end of the main results (Section~\ref{ssec:resnet_main_results}) confirming that all qualitative conclusions hold under either metric.

\subsubsection{Main results across task difficulty}
\label{ssec:resnet_main_results}

\begin{table}[ht!]
\centering
\caption{Detached-readout accuracy (\%) of BP-competitive methods on ResNet-\{18,24,56,108\} across four task-difficulty regimes. ``base.''/``hard.''\ denote the baseline / hardened training recipe. $\Delta$ is the gap from BP to the best local method on the row. Single seed throughout.}
\label{tab:resnet_readout_main}
\small
\setlength{\tabcolsep}{4.5pt}
\begin{tabular}{l l ccccc}
\toprule
Dataset & Depth & BP & LCE & SFF & DF & $\Delta_{\text{BP}-\text{best}}$ \\
\midrule
\multirow{4}{*}{MNIST}
 & R18  & 99.66 & 99.62 & 99.60 & 99.64 & $+0.02$ \\
 & R24  & 99.71 & 99.65 & 99.69 & 99.68 & $+0.02$ \\
 & R56  & 99.67 & 99.70 & 99.70 & 99.67 & $-0.03$ \\
 & R108 & 99.68 & 99.66 & 99.70 & 99.64 & $-0.02$ \\
\midrule
\multirow{4}{*}{CIFAR-10}
 & R18  & 95.15 & 89.06 & 87.20 & 87.16 & $+6.09$ \\
 & R24  & 95.42 & 89.03 & 87.09 & 87.19 & $+6.39$ \\
 & R56  & 95.68 & 89.59 & 88.96 & 89.53 & $+6.09$ \\
 & R108 & 95.40 & 90.02 & 89.88 & 89.75 & $+5.38$ \\
\midrule
\multirow{4}{*}{CIFAR-100 (base.)}
 & R18  & 77.74 & 64.64 & 62.98 & 62.56 & $+13.10$ \\
 & R24  & 78.02 & 64.38 & 63.12 & 62.70 & $+13.64$ \\
 & R56  & 77.91 & 65.27 & 63.55 & 63.96 & $+12.64$ \\
 & R108 & 78.19 & 67.11 & 63.94 & 65.27 & $+11.08$ \\
\midrule
\multirow{4}{*}{CIFAR-100 (hard.)}
 & R18  & 79.79 & 68.48 & 68.00 & 67.84 & $+11.31$ \\
 & R24  & 80.40 & 68.92 & 68.19 & 69.47 & $+10.93$ \\
 & R56  & 79.74 & 70.89 & 69.84 & 72.00 & $+7.74$ \\
 & R108 & 80.19 & 71.48 & 70.59 & 72.43 & $+7.76$ \\
\midrule
\multirow{4}{*}{Tiny-ImageNet (base.)}
 & R18  & 56.05 & 44.49 & 44.56 & 45.37 & $+10.68$ \\
 & R24  & 56.64 & 45.62 & 44.91 & 45.97 & $+10.67$ \\
 & R56  & 56.38 & 46.51 & 47.44 & 48.18 & $+8.20$ \\
 & R108 & 57.63 & 48.34 & 48.39 & 49.48 & $+8.15$ \\
\midrule
\multirow{4}{*}{Tiny-ImageNet (hard.)}
 & R18  & 64.69 & 51.26 & 51.77 & 51.52 & $+12.92$ \\
 & R24  & 65.10 & 52.02 & 52.77 & 52.64 & $+12.33$ \\
 & R56  & 66.35 & 55.35 & 55.05 & 55.47 & $+10.88$ \\
 & R108 & 67.24 & 55.31 & 55.28 & 55.93 & $+11.31$ \\
\bottomrule
\end{tabular}
\end{table}

\paragraph{Saturation regime: MNIST.}
At MNIST scale, BP, LCE, SFF, and DF all sit in $99.60$--$99.71\%$ across every depth (Table~\ref{tab:resnet_readout_main}); the BP--local gap is within measurement noise ($\lvert\Delta\rvert \leq 0.06$~pp at every depth, with BP \emph{behind} the best local at R56 and R108) and shows no monotone depth trend, simply because BP itself has saturated the task.
This row provides a saturation-regime calibration for the canonical-recipe protocol: the protocol's per-method asymmetries (SGD with $5{\times}10^{-4}$ weight decay for BP versus weight-decay-free per-block Adam for the local methods) do not produce a measurable gap when the task ceiling is reached, ruling out catastrophic recipe bias.
Whether smaller, hard-task-only recipe biases exist at the magnitudes observed below cannot be tested from MNIST alone.

\paragraph{Mild-difficulty regime: CIFAR-10---a clean split between local families.}
CIFAR-10 cleanly separates the two local-learning families.
BP-competitive locals stay close to BP: BP $95.15$--$95.68\%$ versus LCE $89.0$--$90.0\%$, SFF $87.1$--$89.9\%$, DF $87.2$--$89.8\%$ across ResNet-18--108, giving a BP--best-local $\Delta$ that sits in a tight $5.38$--$6.39$~pp band across all four depths (we treat the within-band variation as not separable from single-seed noise).
The label-injection FFA family, however, drops to $29$--$42\%$: SymBa $40.80$--$42.03\%$, Trifecta $38.06$--$41.41\%$, SCFF $31.97$--$33.66\%$, Vanilla FFA $31.14$--$34.94\%$, LayerCollab $29.29$--$32.45\%$.
The BP--FFA-family gap on CIFAR-10 therefore exceeds $50$~pp at every depth and does not narrow with depth.
The two local-learning families therefore behave very differently on CIFAR-10: BP-competitive locals trail BP by a stable ${\sim}6$~pp, while label-injection FFA methods trail BP by more than $50$~pp at every depth. Conclusions about ``MNIST/CIFAR-10 BP--local gaps'' should always specify which family is meant; conflating the two obscures a $\sim\!10\times$ difference in the size of the gap.

\paragraph{Hard regime: CIFAR-100 and Tiny-ImageNet.}
On CIFAR-100 and Tiny-ImageNet, the absolute BP advantage widens substantially.
Under the baseline recipe, the R108 BP--best-local gap is $11.08$~pp on CIFAR-100 and $8.15$~pp on Tiny-ImageNet; under the hardened recipe, it is $7.76$~pp and $11.31$~pp, respectively---both strictly larger than on MNIST/CIFAR-10.
The best local-learning result observed at R108 is DF at $72.43\%$ on CIFAR-100-hard and $55.93\%$ on Tiny-ImageNet-hard, both well below BP ($80.19\%$ and $67.24\%$).
This widening is the deep-ResNet counterpart of the CIFAR-10 CNN-benchmark trend in Section~\ref{app:cifar10_cnn_bench}: as task difficulty raises the effective minimum representation capacity, layer-local objectives---which see only locally-supervised signal---fall further behind end-to-end gradients.
Notice that Tiny-ImageNet behaves differently from CIFAR-100 under hardening: at R108 the gap \emph{grows} from $8.15$~pp (baseline) to $11.31$~pp (hardened), whereas on CIFAR-100 it shrinks from $11.08$ to $7.76$~pp.
We hypothesize that item~(iv) of the hardened recipe contributes to this asymmetry: only on Tiny-ImageNet does the hardened recipe restore the native $64{\times}64$ input resolution, and this resolution lift may amplify BP's end-to-end gradient access to fine-grained spatial detail more than it amplifies the per-block local objectives.
The hypothesis is suggested by the contrast with CIFAR-100 (where input remains $32{\times}32$ in both recipes) but \emph{is not isolated by an ablation}: we did not run Tiny-ImageNet baseline at $64{\times}64$ or hardened at $32{\times}32$, so resolution is not separable from the augmentation/weight-decay/label-smoothing axes of the hardened recipe.

\begin{figure}[ht!]
\centering
\includegraphics[width=0.95\linewidth]{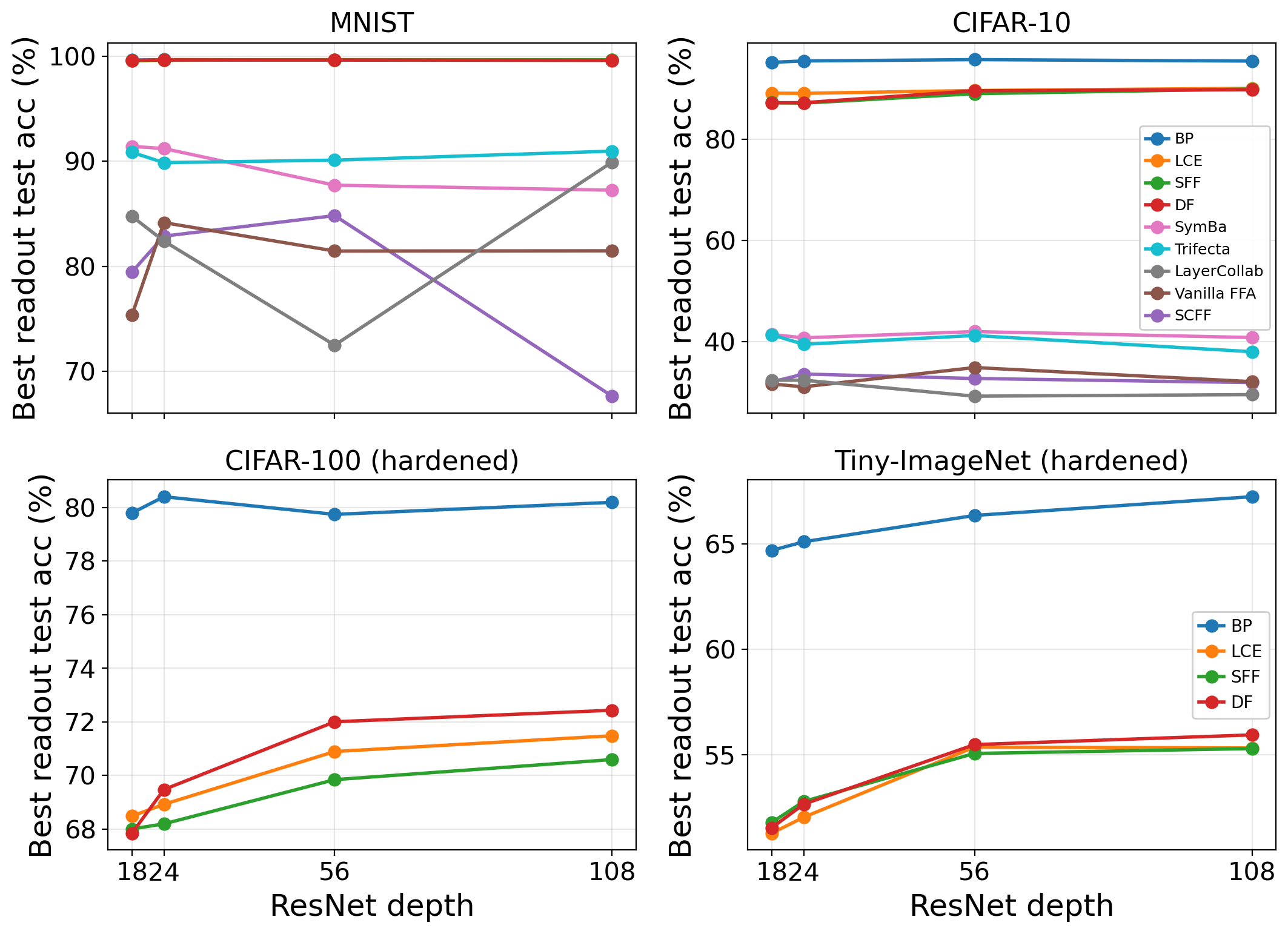}
\caption{Detached-readout best test accuracy versus ResNet depth across four task-difficulty regimes (MNIST and CIFAR-10 use the baseline recipe; CIFAR-100 and Tiny-ImageNet use the hardened recipe).
The BP--local-learning gap is essentially zero on MNIST (all four BP-competitive methods saturate), $5$--$6$~pp on CIFAR-10 for BP-competitive locals (and ${>}50$~pp for the label-injection FFA family; see Table~\ref{tab:resnet_readout_ffa_family}), and grows to $8$--$13$~pp on CIFAR-100/Tiny-ImageNet---where added depth and the hardened recipe partially reshape, but do not close, the gap.}
\label{fig:resnet_readout_difficulty}
\end{figure}

\paragraph{Depth narrows the gap on CIFAR-100; Tiny-ImageNet is largely depth-flat.}
Within each hard-regime block of Table~\ref{tab:resnet_readout_main}, the BP--DF gap shrinks with depth on CIFAR-100 (baseline $15.18 \to 12.92$~pp from R18 to R108; hardened $11.95 \to 7.76$~pp), but stays roughly flat on Tiny-ImageNet (baseline $10.68 \to 8.15$~pp; hardened $12.92 \to 11.31$~pp).
The CIFAR-100 hardened narrowing of ${\sim}4$~pp from R18 to R108 is the largest single depth effect we observe, and it is qualitatively in the direction predicted by the kernel-contraction picture (Theorem~\ref{thm:kernel_contraction}): if the dataset-scaled critical depth $m_\Sigma^* = \Theta(\log N / |\log\rho|)$ grows with task complexity, then for hard datasets the four-rung ResNet ladder $L \in \{18,\dots,108\}$ may sit below $m_\Sigma^*$, and added depth still recovers usable representation.
We emphasise that the kernel-contraction theorem is derived for fully-connected networks; extending it to residual networks (whose iterated kernel map carries the additive identity from skip connections) is beyond the present scope, and we do not estimate $m_\Sigma^*$ for the ResNet kernel.
We therefore read this paragraph as a directional consistency check, not a quantitative theory test.
On MNIST and CIFAR-10 the gap is essentially flat in $L$ because BP has already saturated and there is no headroom to recover.

\paragraph{Final-vs-best-epoch robustness.}
The reported metric \texttt{readout\_test\_acc\_best} is the maximum test accuracy across the $200$ training epochs, a known overestimator under selection on the test set.
We re-compute every BP--best-local gap in Table~\ref{tab:resnet_readout_main} using each run's final-epoch readout test accuracy; the resulting gap shifts are small and almost always in the same direction (final-epoch gap exceeds best-epoch gap by $0.02$ to $0.94$~pp on CIFAR variants, $0.05$ to $1.24$~pp on Tiny-ImageNet baseline, $0.15$ to $0.40$~pp on Tiny-ImageNet hardened, with a single $-0.28$~pp shift at Tiny-ImageNet hardened R18).
Crucially, the \emph{sign} of every gap is preserved at every cell, and the qualitative orderings (BP ahead by $\sim\!6$~pp on CIFAR-10, $\sim\!8$--$15$~pp on CIFAR-100 / Tiny-ImageNet) are identical under either metric.
The conclusions of this section are therefore robust to the best-vs-final-epoch metric choice.

\begin{figure}[ht!]
\centering
\includegraphics[width=0.82\linewidth]{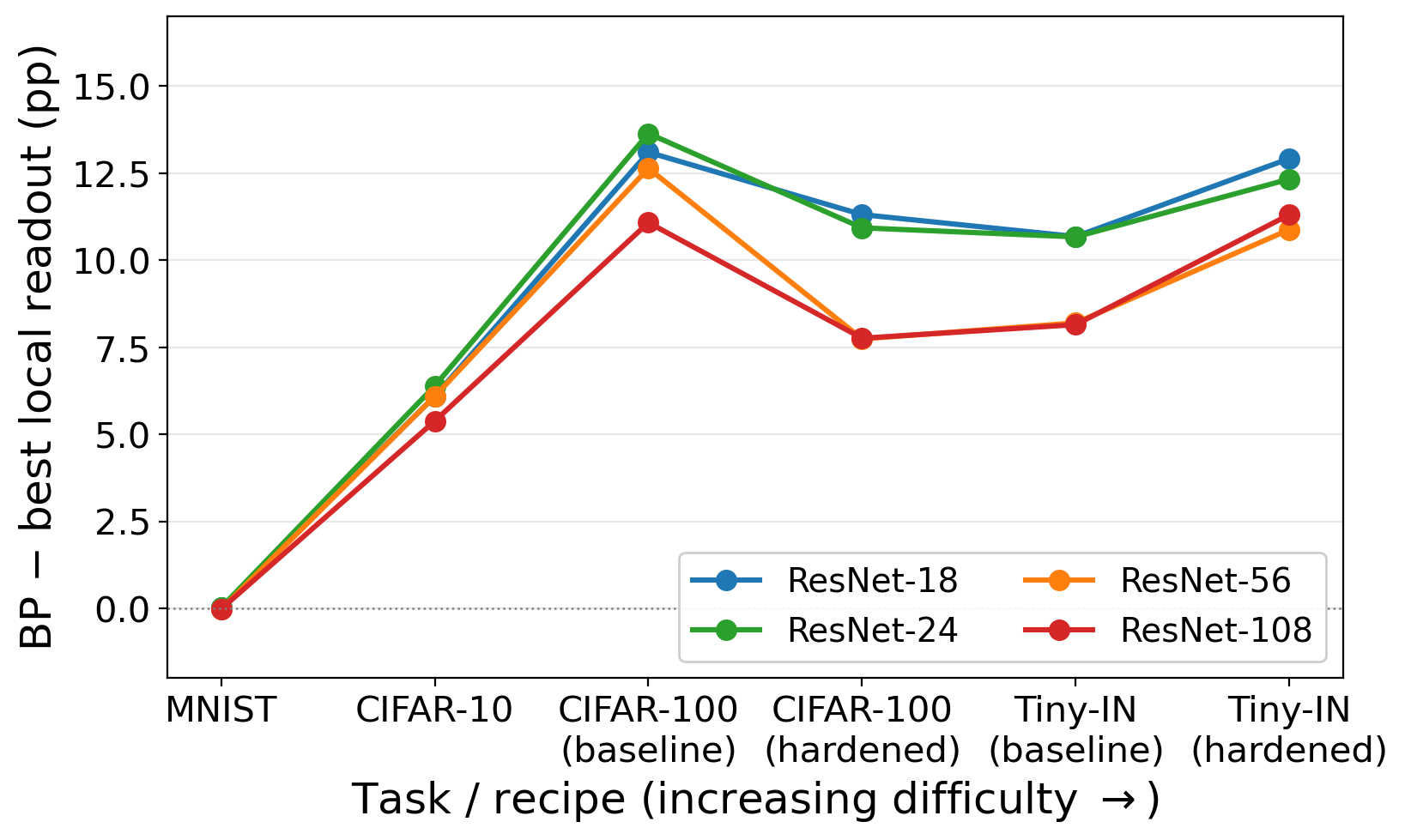}
\caption{Headline view of the data in Table~\ref{tab:resnet_readout_main}: BP minus best BP-competitive-local detached-readout accuracy (pp) plotted across the six (dataset, recipe) regimes ordered by approximate difficulty.
One line per ResNet depth.
The gap is essentially zero on MNIST (BP saturates), $5$--$6$~pp on CIFAR-10, and rises into the $7$--$15$~pp band on CIFAR-100 / Tiny-ImageNet across both recipes.
At every regime, the within-depth spread is $\leq 3$~pp---small relative to the difficulty-driven trend.
Single-seed estimates throughout.}
\label{fig:resnet_gap_vs_difficulty}
\end{figure}

\subsubsection{Native head versus detached readout}
\label{ssec:resnet_native_vs_readout}

A second readable signal in this dataset is how each algorithm's \emph{native} test accuracy compares to its detached readout accuracy on the same trunk.
The native classifier is whatever head ships with the algorithm: a single linear head on top of the last feature for BP, a per-block supervised head averaged across all blocks for SFF, the last-block prediction for LCE and DF, and a goodness-aggregation rule for the label-injection family.
The detached readout, in contrast, is a fixed all-stage concatenation linear probe trained jointly with the trunk.
A systematic difference between the two reveals whether the trunk's representation is more or less class-readable than the algorithm's own classifier consumes.

\begin{table}[ht!]
\centering
\caption{Native test accuracy versus detached-readout accuracy (\%) on CIFAR-100 and Tiny-ImageNet under the hardened recipe. ``$\Delta_{r-n}$''\ is readout minus native; positive means the detached readout extracts more class information from the same trunk than the algorithm's native head does. The pattern is algorithm-specific: BP and DF show $|\Delta_{r-n}|<1$~pp; LCE's readout exceeds its native head by up to $+1.0$~pp; SFF's native head consistently exceeds its readout by up to $-3.20$~pp because SFF's native logits aggregate \emph{all} per-block heads.}
\label{tab:native_vs_readout}
\small
\setlength{\tabcolsep}{4pt}

\footnotesize
\setlength{\tabcolsep}{3.2pt}
\renewcommand{\arraystretch}{1.05}

\begin{tabular}{l l ccc ccc ccc ccc}
\toprule
& & \multicolumn{3}{c}{BP} & \multicolumn{3}{c}{LCE} & \multicolumn{3}{c}{SFF} & \multicolumn{3}{c}{DF} \\
\cmidrule(lr){3-5}\cmidrule(lr){6-8}\cmidrule(lr){9-11}\cmidrule(lr){12-14}
Dataset & Depth & native & read. & $\Delta_{r-n}$ & native & read. & $\Delta_{r-n}$ & native & read. & $\Delta_{r-n}$ & native & read. & $\Delta_{r-n}$ \\
\midrule
\multirow{4}{*}{C-100 (hard.)}
 & R18  & 79.85 & 79.79 & $-0.06$ & 67.98 & 68.48 & $+0.50$ & 69.61 & 68.00 & $-1.61$ & 67.31 & 67.84 & $+0.53$ \\
 & R24  & 80.38 & 80.40 & $+0.02$ & 68.80 & 68.92 & $+0.12$ & 70.07 & 68.19 & $-1.88$ & 69.03 & 69.47 & $+0.44$ \\
 & R56  & 79.77 & 79.74 & $-0.03$ & 70.16 & 70.89 & $+0.73$ & 71.88 & 69.84 & $-2.04$ & 71.25 & 72.00 & $+0.75$ \\
 & R108 & 79.91 & 80.19 & $+0.28$ & 71.02 & 71.48 & $+0.46$ & 72.81 & 70.59 & $-2.22$ & 72.56 & 72.43 & $-0.13$ \\
\midrule
\multirow{4}{*}{T-IN (hard.)}
 & R18  & 65.06 & 64.69 & $-0.37$ & 50.55 & 51.26 & $+0.71$ & 52.67 & 51.77 & $-0.90$ & 51.08 & 51.52 & $+0.44$ \\
 & R24  & 65.46 & 65.10 & $-0.36$ & 51.01 & 52.02 & $+1.01$ & 53.13 & 52.77 & $-0.36$ & 52.48 & 52.64 & $+0.16$ \\
 & R56  & 66.60 & 66.35 & $-0.25$ & 54.40 & 55.35 & $+0.95$ & 56.53 & 55.05 & $-1.48$ & 54.84 & 55.47 & $+0.63$ \\
 & R108 & 67.50 & 67.24 & $-0.26$ & 54.92 & 55.31 & $+0.39$ & 58.48 & 55.28 & $-3.20$ & 55.50 & 55.93 & $+0.43$ \\
\bottomrule
\end{tabular}
\end{table}

Table~\ref{tab:native_vs_readout} reveals an algorithm-specific representation geometry that the readout column alone obscures.
BP places essentially all of its linearly-readable class information in the penultimate feature: $|\Delta_{r-n}| \leq 0.37$~pp at every cell, so the all-stage probe finds nothing beyond what BP's own last-feature classifier already uses.
LCE and DF leave a small but consistent class signal in earlier residual stages that their last-block native heads discard; the all-stage readout recovers it ($+0.4$ to $+1.0$~pp).
SFF distributes class information across all per-block CE heads---its native logit is a depth-wise average---so a single linear projection over concatenated hidden features cannot match the multi-head average ($\Delta_{r-n}$ down to $-3.20$~pp at Tiny-ImageNet R108).
The geometric reading complements the main BP-vs-local thesis: gap differences in Table~\ref{tab:resnet_readout_main} reflect not only how much class information the trunk encodes, but also how that information is spatially distributed across stages, and the all-stage readout protocol is a slightly conservative probe for distributed representations like SFF's while being slightly liberal for last-feature-concentrated representations like BP's.

\subsubsection{Label-injection FFA family: per-algorithm characterisation and failure modes}
\label{ssec:resnet_ffa_family}

\begin{table}[ht!]
\centering
\caption{Label-injection FFA family detached-readout accuracy (\%) on the same ResNet ladder. These methods are excluded from CIFAR-100 / Tiny-ImageNet, where they failed to train stably across depths under their published recipes.}
\label{tab:resnet_readout_ffa_family}
\small
\setlength{\tabcolsep}{5pt}
\begin{tabular}{l l ccccc}
\toprule
Dataset & Depth & SCFF & Vanilla FFA & SymBa & LayerCollab & Trifecta \\
\midrule
\multirow{4}{*}{MNIST}
 & R18  & 79.44 & 75.32 & 91.43 & 84.76 & 90.87 \\
 & R24  & 82.87 & 84.16 & 91.23 & 82.40 & 89.88 \\
 & R56  & 84.83 & 81.45 & 87.73 & \textbf{72.47} & 90.12 \\
 & R108 & \textbf{67.60} & 81.47 & 87.25 & 89.94 & 90.98 \\
\midrule
\multirow{4}{*}{CIFAR-10}
 & R18  & 32.09 & 31.69 & 41.46 & 32.45 & 41.41 \\
 & R24  & 33.66 & 31.14 & 40.80 & 32.45 & 39.52 \\
 & R56  & 32.78 & 34.94 & 42.03 & 29.29 & 41.26 \\
 & R108 & 31.97 & 32.16 & 40.86 & 29.60 & 38.06 \\
\bottomrule
\end{tabular}
\end{table}

The five label-injection methods sit on the same backbone with the same per-block Adam recipe and differ only in their goodness aggregation and negative-sample strategy.
We characterise them in three groups.
\emph{SymBa and Trifecta} are the strongest of the family on MNIST ($87$--$92\%$) and on CIFAR-10 ($38$--$42\%$), and both are roughly depth-flat across the ResNet ladder.
A plausible mechanism is that SymBa's $\alpha$-scaled symmetric margin and Trifecta's odd/even-layer update schedule both keep the per-block goodness scale bounded as depth grows, which the simpler sigmoid-contrastive form does not; we do not measure block-wise goodness magnitudes here so we report this as a hypothesis rather than a finding.
\emph{Vanilla FFA and SCFF} use the standard sigmoid contrastive form with two distinct negative strategies (random wrong label / image-permutation derangement).
On MNIST both sit at $75$--$85\%$ except for the SCFF R108 outlier discussed below; on CIFAR-10 both collapse to $31$--$35\%$ at every depth, well below SymBa/Trifecta.
The CIFAR-10 deficit is consistent with the archived CIFAR-10 goodness ablations from the repr-gap branch---norm-style goodness has higher Lipschitz smoothness than projection-style goodness, which limits the per-block convergence rate---but separating goodness function from negative strategy in the present ResNet runs would require a controlled ablation we did not perform.
\emph{LayerCollab} adds a detached collaboration offset $\gamma_{<t} = \sum_{t'<t} G_{t'}$ to each layer's goodness; on MNIST it is unstable across depths ($72.47\%$ at R56, $89.94\%$ at R108---a $17.47$~pp swing under fixed seed and recipe), and on CIFAR-10 it sits at $29.29$--$32.45\%$, slightly below Vanilla FFA.
\emph{Single-cell anomalies.}
Two single cells deviate sharply from their depth-neighbours: LayerCollab MNIST R56 $=72.47\%$ (versus $84.76$/$82.40$/$89.94$ elsewhere) and SCFF MNIST R108 $=67.60\%$ (versus $79.44$/$82.87$/$84.83$).
Both runs completed normally (W\&B status \texttt{ok}); we report them as observed.
With a single seed we cannot distinguish recipe-specific instability of FFA-family contrastive losses from seed-specific outliers, and a multi-seed re-run would be required to characterise the source.

\subsubsection{Hardening recipe: gain decomposition}
\label{ssec:resnet_hardening_gain}

\begin{table}[ht!]
\centering
\caption{Detached-readout accuracy gain (pp) from baseline to hardened recipe, per (depth, algorithm). Last column averages across the four BP-competitive algorithms.}
\label{tab:resnet_hardening_gains}
\small
\setlength{\tabcolsep}{5pt}
\begin{tabular}{l l ccccc}
\toprule
Dataset & Depth & BP & LCE & SFF & DF & avg \\
\midrule
\multirow{4}{*}{CIFAR-100}
 & R18  & $+2.05$ & $+3.84$ & $+5.02$ & $+5.28$ & $+4.05$ \\
 & R24  & $+2.38$ & $+4.54$ & $+5.07$ & $+6.77$ & $+4.69$ \\
 & R56  & $+1.83$ & $+5.62$ & $+6.29$ & $+8.04$ & $+5.45$ \\
 & R108 & $+2.00$ & $+4.37$ & $+6.65$ & $+7.16$ & $+5.05$ \\
\midrule
\multirow{4}{*}{Tiny-ImageNet}
 & R18  & $+8.64$ & $+6.77$ & $+7.21$ & $+6.15$ & $+7.19$ \\
 & R24  & $+8.46$ & $+6.40$ & $+7.86$ & $+6.67$ & $+7.35$ \\
 & R56  & $+9.97$ & $+8.84$ & $+7.61$ & $+7.29$ & $+8.43$ \\
 & R108 & $+9.61$ & $+6.97$ & $+6.89$ & $+6.45$ & $+7.48$ \\
\bottomrule
\end{tabular}
\end{table}

The robust observation across Table~\ref{tab:resnet_hardening_gains} is that BP's hardening gain shifts position between datasets: smallest on CIFAR-100 ($+1.83$ to $+2.38$~pp---essentially depth-flat) and largest on Tiny-ImageNet ($+8.46$ to $+9.97$~pp).
The three layer-local methods sit at $+3.84$ to $+8.04$~pp on CIFAR-100 and $+6.15$ to $+8.84$~pp on Tiny-ImageNet, all positive but with per-cell variation comparable to single-seed noise.
We do not read the relative ordering of the three local methods within each dataset as significant: per-cell differences of $0.5$--$2$~pp under one seed cannot reliably rank them.
The dataset-level signal we do read is BP's gain ranking: on CIFAR-100 BP gains the least (consistent with BP's baseline being close to its own near-ceiling), while on Tiny-ImageNet BP gains the most.
The most parsimonious explanation is item~(iv) of the hardened recipe---only on Tiny-ImageNet does the hardened recipe restore the native $64{\times}64$ resolution---but as flagged in Section~\ref{ssec:resnet_main_results} we did not run the resolution ablation that would isolate this from augmentation, weight decay, and label smoothing.
A weaker-but-confirmed reading: the average $+4.81$~pp CIFAR-100 hardening gain is consistent with regularisation acting more on the lower-baseline local methods than on the already-near-ceiling BP, whereas the larger $+7.61$~pp average Tiny-ImageNet hardening gain mixes regularisation with the spatial-resolution lift in a way our experiments do not separate.

\subsubsection{Summary and connection to other sections}
\label{ssec:resnet_summary}

The BP--local-learning gap behaves \emph{monotonically in task difficulty} on deep ResNets under the per-method canonical-recipe protocol: near zero on MNIST, $5$--$6$~pp on CIFAR-10 for BP-competitive locals (and $>50$~pp for the label-injection FFA family), and $8$--$13$~pp on CIFAR-100/Tiny-ImageNet at R108.
Two intervention axes can shrink---but do not close---the gap on hard tasks: (i) increasing ResNet depth, where the largest single effect we observe is the ${\sim}4$~pp narrowing on CIFAR-100 hardened from R18 to R108, qualitatively in the direction predicted by the kernel-contraction picture (Theorem~\ref{thm:kernel_contraction}; we treat this as a directional consistency check rather than a quantitative theory test, and Tiny-ImageNet's depth-flat gap shows the effect is not universal); and (ii) the hardened training recipe, whose gain decomposition in Table~\ref{tab:resnet_hardening_gains} reveals dataset-dependent BP-versus-local ordering that we cannot fully separate from the Tiny-ImageNet $32 \to 64$ resolution lift.
The MNIST/CIFAR-10 ``small gap'' of the BP-competitive local family is therefore not evidence that local learning matches BP; it is the visible-light region of a saturation effect that disappears once the task is hard enough to expose the optimization gap.
The pattern dovetails with the LM scaling-law results of Section~\ref{sec:experiments}---local methods win or tie at low-difficulty / shallow-depth corners, and the gap opens as either axis is pushed---and complements the algorithm-axis result of Section~\ref{app:cifar10_cnn_bench}, where the gap to BP at fixed task (CIFAR-10) varies sharply with the choice of local-learning rule.
Two limitations of the present study are inherent to the protocol and deserve emphasis: \emph{(i)~single seed.}
All $136$ runs use seed $0$; we report numbers as observed without per-seed variance.
The MNIST cluster within $0.06$~pp at every depth provides a calibration that the protocol is not high-variance at saturation, but we cannot statistically discriminate the small CIFAR-10 trends with depth ($+6.39 \to +5.38$~pp) without additional seeds.
\emph{(ii)~per-method recipe asymmetries.}
BP uses SGD with weight decay $5{\times}10^{-4}$ (baseline) or $10^{-3}$ (hardened); the layer-local methods use weight-decay-free per-block Adam under both recipes, with manual step decay rather than cosine annealing.
We adopt these asymmetries deliberately as the canonical-per-method choice rather than imposing a uniform optimizer (see Section~\ref{ssec:resnet_setup_protocol}), and the MNIST result confirms they do not produce a systematic BP advantage on saturated tasks.
A future study with matched optimizers would isolate whether the residual hard-task gap reflects pure optimization-paradigm difference (end-to-end gradient access) or recipe asymmetry; under the present protocol, the two are not separately identified.

\subsection{Language-model pre-training}
\label{app:llm_pretraining}

Table~\ref{tab:chinchilla} uses decoder-only Transformers trained on the OpenWebText subset: GPT-2 BPE tokenization with vocabulary size $50{,}257$, $50$M training tokens, and $0.5$M validation tokens.
All runs use context length $256$, dropout $0.1$, batch size $32$, and therefore $32 \times 256 = 8{,}192$ tokens per optimization step.
The five model scales are
\[
(L, d, h) \in \{(2,128,4), (4,256,4), (6,384,6), (8,512,8), (12,768,12)\},
\]
corresponding to the tiny, small, medium, large, and xlarge rows of Table~\ref{tab:chinchilla}.
The ``Params'' and ``Budget'' columns in that table refer to the backbone parameter count and the Chinchilla-style target token budget $20 \times \text{params}$, giving $0.14$B, $0.32$B, $0.60$B, $1.02$B, and $2.48$B tokens respectively.
Because the OWT subset contains only $50$M training tokens, these budgets correspond to repeated passes through the same stream, from about $2.8\times$ passes at tiny to $49.6\times$ passes at xlarge.

The optimization schedule is shared across methods.
We use AdamW with peak learning rate $6 \times 10^{-4}$, cosine decay to $6 \times 10^{-5}$, warmup for $\min(400, \text{max\_iters}/20)$ steps, and gradient clipping at norm $1.0$.
Validation perplexity is evaluated every $5\%$ of the training horizon (at least every $100$ steps) on $10$ random validation batches, and the table reports the best validation perplexity observed during training.
For BP and FFA, perplexity is computed with the shared tied \texttt{lm\_head} after the final layer norm; for LCE, perplexity is computed directly from the last block's untied CE head, matching the objective used during training.

The three methods differ only in how the local training signal is constructed.
BP performs standard end-to-end next-token cross-entropy on the full model.
The LCE column corresponds to the \texttt{ffa\_lce\_untied} mode: each block is trained by a detached full-softmax cross-entropy loss with its own untied vocabulary head, again using per-block AdamW plus a separate embedding optimizer.
The table's parameter and budget columns still count only the shared Transformer backbone, so the LCE runs incur a substantially larger effective parameter count because of the extra per-layer untied heads (for example, the xlarge LCE run has $586.9$M total parameters versus $123.7$M backbone parameters).

%% file: appendix/CIFAR10-CNN_bench.tex
\begin{tabular}{@{}llccc@{}}
\toprule
Algorithm & Category & CNN3 & CNN6 & CNN9 \\
\midrule
  Standard BP & Baseline & 79.26 $\pm$ 0.24 & 85.30 $\pm$ 0.15 & 86.12 $\pm$ 0.28 \\
  \midrule
  Vanilla FFA \cite{hinton2022forward} & FFA & 47.98 $\pm$ 1.71 & 52.31 $\pm$ 1.29 & 54.05 $\pm$ 1.99 \\
  SymBa \cite{lee2023symba} & FFA-contrastive & 60.96 $\pm$ 1.65 & 61.93 $\pm$ 0.90 & 63.69 $\pm$ 0.50 \\
  SCFF \cite{chen2024selfcontrastive} & FFA-contrastive & 50.04 $\pm$ 1.95 & 50.75 $\pm$ 1.99 & 53.47 $\pm$ 1.85 \\
  Trifecta \cite{dooms2023trifecta} & FFA-extended & 57.51 $\pm$ 0.57 & 60.52 $\pm$ 0.34 & 62.42 $\pm$ 1.05 \\
  Scodellaro CNN-FFA \cite{scodellaro2025cnn} & FFA-extended & 38.41 $\pm$ 0.65 & 43.12 $\pm$ 0.97 & 43.58 $\pm$ 2.38 \\
  Distance-Forward \cite{wu2024distance} & FFA-extended & 74.88 $\pm$ 0.32 & 75.40 $\pm$ 0.19 & 75.27 $\pm$ 0.67 \\
  SFF \cite{krutsylo2025scalable} & Local-CE & 80.46 $\pm$ 0.45 & 78.49 $\pm$ 0.27 & 77.20 $\pm$ 0.14 \\
  Greedy Layerwise \cite{belilovsky2019greedy} & Local-CE & 78.45 $\pm$ 0.32 & 77.95 $\pm$ 0.10 & 77.69 $\pm$ 0.56 \\
  N{\o}kland $L_{\mathrm{pred}}$ \cite{nokland2019local} & Local-CE & 82.45 $\pm$ 0.32 & 82.73 $\pm$ 0.08 & 82.74 $\pm$ 0.56 \\
  N{\o}kland $L_{\mathrm{sim}}$ \cite{nokland2019local} & Local-sim & 73.06 $\pm$ 0.35 & 72.61 $\pm$ 0.53 & 71.85 $\pm$ 0.84 \\
  N{\o}kland $L_{\mathrm{pred}+\mathrm{sim}}$ \cite{nokland2019local} & Local-combined & 82.82 $\pm$ 0.10 & 82.51 $\pm$ 0.26 & 82.75 $\pm$ 0.23 \\
  Greedy InfoMax \cite{lowe2019greedy} & InfoNCE-local & N/A & 45.81 $\pm$ 0.64 & 45.65 $\pm$ 0.07 \\
  Dendritic Localized \cite{lv2025dendritic} & Dendritic-local & 82.37 $\pm$ 0.13 & 83.05 $\pm$ 0.35 & 82.56 $\pm$ 0.39 \\
  \midrule
  AugLocal \cite{ma2024auglocal} & Aux-net & 80.68 $\pm$ 0.37 & 84.79 $\pm$ 0.40 & 86.69 $\pm$ 0.23 \\
  Difference Target Prop \cite{lee2015difference} & Target-Prop & 75.54 $\pm$ 0.16 & 66.84 $\pm$ 2.32 & 56.94 $\pm$ 2.66 \\
  \midrule
  Counter-Current \cite{kao2024countercurrent} & Dual-path feedback & 54.09 $\pm$ 0.33 & 55.57 $\pm$ 0.64 & 55.64 $\pm$ 0.78 \\
  Layer Collaboration \cite{lorberbom2024layer} & Multi-layer-coord & 50.41 $\pm$ 2.69 & 58.34 $\pm$ 1.51 & 60.84 $\pm$ 1.83 \\
  \midrule
  Forward Projection \cite{oshea2025forwardprojection} & Random-features & 58.96 $\pm$ 0.77 & 55.33 $\pm$ 0.71 & 54.70 $\pm$ 0.85 \\
  \midrule
  PC ($K=1$) & Energy-based & 66.24 $\pm$ 1.32 & 59.72 $\pm$ 1.35 & 52.35 $\pm$ 1.14 \\
  PC ($K=5$) & Energy-based & 78.82 $\pm$ 0.28 & 81.74 $\pm$ 0.05 & 64.27 $\pm$ 1.90 \\
  PC ($K=20$) & Energy-based & 79.03 $\pm$ 0.28 & 83.67 $\pm$ 0.40 & 83.01 $\pm$ 0.21 \\
\bottomrule
\end{tabular}

%% file: appendix/F_taxonomy.tex
\section{Taxonomy of Local Learning Baselines}
\label{app:baseline_taxonomy}

This appendix provides a systematic comparison of local learning methods relevant to the price of locality framework. 

\paragraph{FFA-NCE (Hinton's original~\cite{hinton2022forward}).}
Layer-local contrastive loss with $K \ll |V|$ random negatives. 

\paragraph{FFA-LCE.}
Each layer $\ell$ maintains an untied classification head $U^{(\ell)} \in \mathbb{R}^{V \times d}$ and minimizes full-softmax cross-entropy.

\paragraph{Greedy Layerwise CE~\cite{belilovsky2019greedy}.}
Layers are trained \emph{sequentially}: layer $j$ is optimized with an auxiliary classifier $C_{\gamma_j}$ (cross-entropy loss), then frozen before training layer $j{+}1$. 

\paragraph{Trifecta / SymBa~\cite{lee2023symba,krutsylo2025scalable}.}
Symmetric contrastive loss $\mathcal{L}_{\text{SymBa}} = \log(1 + e^{-\alpha(G_{\text{pos}} - G_{\text{neg}})})$ with BatchNorm and overlapping local updates (OLU).

\paragraph{AugLocal~\cite{ma2024auglocal}.}
Ma et al.\ equip each layer $\ell$ with a depth-decaying augmented auxiliary network whose depth $k_\ell = k_{\max}(1 - (\ell-1)/(L-1))$ decreases linearly with layer index.
This lets shallow layers receive rich, multi-layer supervised signals while deep layers reduce to single linear heads, balancing local signal quality against auxiliary parameter cost.

\paragraph{$N$-wise interlocking backpropagation~\cite{gomez2022interlocking}.}
Gomez et al.\ partition the network into $A$ modules (one per accelerator) and allow gradients to propagate back across exactly $N{-}1$ module boundaries before being stopped by a detach.
Setting $N{=}1$ recovers pure local learning (identical to the greedy CE baseline with $k{=}1$ auxiliary heads); $N{=}A$ recovers standard end-to-end BP.
For intermediate $N$, the step time per batch is $2N$ (vs.\ $2A$ for end-to-end), so increasing $N$ trades parallelism for accuracy.

\paragraph{Local Error Signals / predsim~\cite{nokland2019local}.}
N{\o}kland \& Eidnes systematically compare three local objectives: (i)~$\mathcal{L}_{\text{pred}}$: per-layer cross-entropy (functionally identical to FFA-LCE's mechanism and to the $k{=}1$ case of~\cite{belilovsky2019greedy}), (ii)~$\mathcal{L}_{\text{sim}}$: similarity matching between activation inner products and label similarity matrices, and (iii)~$\mathcal{L}_{\text{predsim}} = (1{-}\beta)\mathcal{L}_{\text{pred}} + \beta\mathcal{L}_{\text{sim}}$ with $\beta{=}0.99$. 

\paragraph{Local Supervision (LS).}
Each layer has an independent softmax classifier; gradients do not propagate to preceding layers. 

\paragraph{Direct Feedback Alignment (DFA)~\cite{lillicrap2016random}.}
Fixed random matrices $B_\ell$ project the output-layer error directly to each layer: $\Delta W_\ell \propto B_\ell \boldsymbol{\delta}_L$. Unlike purely local methods, DFA has access to the \emph{global} error signal $\boldsymbol{\delta}_L$, but the random projection degrades signal quality with depth. 

\paragraph{Forward Projection (FP).}
Each layer solves a one-shot pseudo-inverse: $W_\ell = Y_\ell H_\ell^+$. No iterative optimization, no inter-layer communication. 

\paragraph{FFA-MSP.}
To diversify per-layer target horizons and resist the rank-$1$ kernel collapse of Theorem~\ref{thm:kernel_contraction}, layer $\ell$ predicts the token $K_\ell$ steps ahead via a per-layer projection $P^{(\ell)}\in\R^{d\times d}$ (initialised to $I$).
Here $\bx=(x_1,\ldots,x_T)$ denotes the token sequence, $\bx_t=(x_1,\ldots,x_t)$ is the causal context up to position $t$, $\be_{v}\in\R^{d}$ is the token embedding of $v\in\ccalV$, and $K_{\max}\in\naturals$ is a fixed horizon cap.
The layer-$\ell$ goodness is
\begin{align}\label{eq:msp_goodness}
    \goodness^{(\ell)}_{\mathrm{MSP}}\!\bigl(\bh^{(\ell)}(\bx_t),\,\bx\bigr)
    = \bigl(P^{(\ell)}\bh^{(\ell)}(\bx_t)\bigr)^{\!\top}\be_{x_{t+K_\ell}},
    \qquad K_\ell = \min(2^{\ell},\,K_{\max}),
\end{align}
substituted into the contrastive loss~\eqref{eq:ffa_loss}. Exponentially growing horizons force layer $\ell$ to encode information at a distinct temporal scale. Despite this structural advantage, MSP suffers a \emph{training--evaluation mismatch}: the evaluation probe targets $x_{t+1}$ while layer $\ell$ was optimized for $x_{t+K_\ell}$.

\paragraph{Predictive Coding (PC)~\cite{rao1999predictive,whittington2017approximation,millidge2022predictive}.}
PC alternates inference of hidden states with local parameter updates. 
Collect all layer weights into $\bW$ and all inferred states into $\bh$. The optimized output variable and sample loss of the generic EBM formulation are identified with the predictive-coding classifier output and cross-entropy loss by $y^{\prime} := c(\bh^{(L)})$. 
Its energy and population objective are
\begin{align}
    \ccalE(\bh,\bW;x,y)
    &:=
    \ell\bigl(c(\bh^{(L)}),y\bigr)
    +\frac{\beta}{2}\sum_{\ell=1}^{L}\frac{1}{d^{(\ell)}}\bigl\|\be^{(\ell)}\bigr\|^2,
\end{align}
where  $d^{(\ell)}$ is the layer dimension and $\beta>0$ weights prediction error, and the prediction error at layer $\ell$ is
\begin{align}
    \be^{(\ell)}
    :=
    \bh^{(\ell)}-f^{(\ell)}\bigl(\bh^{(\ell-1)};W^{(\ell)}\bigr).
\end{align}
With these notations, the predictive-coding population objective is
\begin{align}
    \cL_{\mathrm{PC}}(\bW)
    :=
    \mbE_{x,y}\left[\min_{\bh}\ccalE(\bh,\bW;x,y)\right].
\end{align}
During the training, an inner loop of $K$ inference steps is performed to minimize the energy $\ccalE(\bh,\bW;x,y)$ with respect to the hidden states $\bh$.
Once the hidden states $\be_K^{(\ell)}$ has been inferred, the block update is local:
\begin{align}
    \Delta W^{(\ell)}_{\mathrm{PC}}
    \;\propto\;
    \beta\,\mathrm{D}_{W^{(\ell)}}f^{(\ell)}
    (\widehat{\bh}^{(\ell-1)};W^{(\ell)})^\top\be_K^{(\ell)}.
\end{align}
At the relaxation fixed point, $\beta\be_\infty^{(\ell)}=-\nabla_{\widehat{\bh}^{(\ell)}}\ell_{\mathrm{CE}}$, so this update converges to the corresponding BP parameter update~\cite{whittington2017approximation,millidge2022predictive}. This limit relies on freezing predictions and Jacobians and clamping the negative CE output gradient; a free-output energy ablation does not establish it.

%% file: appendix/G_acknowledgement.tex
\section{Broader Impact}
\label{sec:broader_impact}

This paper provides a theoretical analysis of the Forward-Forward Algorithm (FFA) and explains why local learning rules can underperform backpropagation as depth and scale increase. One potential positive impact of this work is that it can help guide the design of more reliable and better-understood alternatives to backpropagation in next generation AI acceleration hardware. At the same time, our main conclusion is cautionary rather than enabling: the theory identifies structural limitations of current FFA-style methods, including an irreducible optimization floor and representational collapse. In that sense, the paper is more useful for clarifying the limits of local learning than for directly expanding the capabilities of deployed machine learning systems. We therefore do not identify any specific societal risks that arise uniquely from the theoretical results in this work beyond the general risks already associated with machine learning research and deployment.

\paragraph{LLM Use Statement.}
\label{sec:LLM-use}
LLMs were used only as research assistants for implementation-related tasks, including helping draft experimental code, generating preliminary proof drafts, and preparing data-visualization scripts and figures. Any proof drafts produced with LLM assistance were subsequently rewritten by the authors to ensure mathematical correctness and readability. LLMs were not part of the core scientific methodology: the theoretical framing, theorem statements, proofs, experimental design, result verification, and final scientific claims were determined and checked by the authors.